%% file: main.tex
\documentclass{article}
\usepackage[letterpaper,bindingoffset=0.2in,left=1in,right=1in,top=1in,bottom=1in,footskip=.25in]{geometry}

\usepackage{microtype}
\usepackage{graphicx}
\usepackage{subfigure}
\usepackage{booktabs} 

\usepackage{lmodern}
\usepackage{amsfonts}       

\usepackage{amsmath,amsthm}
\DeclareMathOperator*{\argmax}{arg\,max}
\DeclareMathOperator*{\argmin}{arg\,min}
\usepackage{algorithm}
\usepackage{float}
\input{headers}
\restylefloat{table}
\usepackage{caption}
\newcommand{\vct}[1]{\boldsymbol{#1}} 
\newcommand{\mat}[1]{\boldsymbol{#1}} 

\newcommand{\bbR}{\mathbb{R}}
\usepackage{stfloats}
\usepackage{threeparttable}
\usepackage{comment}
\usepackage{xcolor}
\usepackage{wrapfig}
\usepackage{enumitem}

\usepackage{hyperref}

\newcommand{\bbP}{\mathbb{P}}
\newcommand{\bbQ}{\mathbb{Q}}
\renewcommand{\bbR}{\mathbb{R}}
\newcommand{\bbI}{\mathbb{I}}

\newcommand{\calD}{\mathcal{D}}
\newcommand{\calI}{\mathcal{I}}
\newcommand{\calM}{\mathcal{M}}
\newcommand{\calN}{\mathcal{N}}
\newcommand{\calR}{\mathcal{R}}
\newcommand{\calX}{\mathcal{X}}
\newcommand{\calY}{\mathcal{Y}}

\newcommand{\adaped}{\texttt{AdaPeD}}
\newcommand{\dpadaped}{\texttt{DP-AdaPeD}}

\renewcommand{\bg}{\boldsymbol{g}}
\renewcommand{\bh}{\boldsymbol{h}}
\renewcommand{\bw}{\boldsymbol{w}}
\renewcommand{\bx}{\boldsymbol{x}}

\newcommand{\bnu}{\boldsymbol{\nu}}

\newcommand{\Bern}{\mathsf{Bern}}
\newcommand{\Binom}{\mathsf{Binom}}
\newcommand{\Beta}{\mathsf{Beta}}
\newcommand{\priv}{\mathsf{priv}}
\newcommand{\KD}{\mathsf{KD}}
\newcommand{\KL}{\mathsf{KL}}

\newcommand{\hmu}{\widehat{\vct{\mu}}}

\newcommand{\hsigma}{\widehat{\sigma}}
\newcommand{\htheta}{\widehat{\vct{\theta}}}
\newcommand{\sigmatheta}{\sigma_{\theta}}
\newcommand{\hata}{\widehat{a}}
\newcommand{\hatp}{\widehat{p}}

\newtheorem*{claim*}{Claim}
\newtheorem*{corollary*}{Corollary}

\title{A Generative Framework for Personalized Learning and Estimation: Theory, Algorithms, and Privacy}

\author{
	Kaan Ozkara\thanks{The first and second authors made equal contribution.}
	\and
	Antonious M. Girgis\footnotemark[1]
	\and
	Deepesh Data
	\and
	Suhas Diggavi
}
\date{University of California, Los Angeles, USA\\
	kaan@ucla.edu, amgirgis@ucla.edu, deepesh.data@gmail.com, suhas@ee.ucla.edu
}

\begin{document}

	\maketitle

	\input{abstract.tex}
	\section{Introduction} \label{sec:intro}
	\input{introduction.tex}

	
	\input{estimation}

	\input{estimation_gaussian}

	\input{estimation_bernoulli}

	\input{estimation_discrete}
	
	\input{learning}

\input{unification}

	\section{Experiments} \label{sec:experiments}

	\input{experiments.tex}

	\input{proofs_estimation}

	\input{proofs_learning}

	\newpage
	\bibliography{bibliography}
	\bibliographystyle{plain}


\newpage

	\renewcommand{\thesection}{\Alph{section}}
	\appendix
	\setcounter{equation}{0}
	{\allowdisplaybreaks
	\section{Preliminaries on Differential Privacy}\label{app:preliminary}

\input{app_preliminary.tex}

	\section{Personalized Learning -- Discrete Mixture Model}\label{app:disc-mixture}
	\input{app_disc_mixture.tex}

	\label{submission}
	}

\end{document}

%% file: headers.tex
\usepackage[T1]{fontenc}
\usepackage{tikz}
\usetikzlibrary{arrows}
\usetikzlibrary{calc}
\usetikzlibrary{snakes}
\usepackage{cite}
\usepackage{graphicx}
\usepackage{array}
\usepackage{mdwmath}
\usepackage{mdwtab}
\usepackage{eqparbox}
\usepackage{fixltx2e}
\usepackage{url}
\usepackage{pgf}
\usepackage{lipsum}
\usepackage{multirow}
\usepackage{xspace}
\usepackage{xcolor}
\usepackage{algorithmic}
\usepackage{bm}
\usepackage{bbm}
\usepackage{booktabs}
\usepackage{etoolbox}
\usepackage{makecell}
\usepackage{nicefrac}
\usepackage{textcomp}
\usepackage{stmaryrd}
\usepackage{mathrsfs}
\usepackage{paralist}
\usepackage{comment}

\usepackage[font=footnotesize,labelsep=space,labelfont=bf]{caption}

\usepackage{url}

\usepackage[pdftex,bookmarks=true,pdfstartview=FitH,colorlinks,linkcolor=blue,filecolor=blue,citecolor=blue,urlcolor=blue,pagebackref=false]{hyperref}
\makeatletter

\newcommand{\Rmnum}[1]{\expandafter\@slowromancap\romannumeral #1@}
\makeatother

\newtheorem{theorem}{Theorem}
\newtheorem*{theorem*}{Theorem}
\newtheorem{definition}{Definition}
\newtheorem{lemma}{Lemma}
\newtheorem*{lemma*}{Lemma}
\newtheorem{claim}{Claim}
\newtheorem*{cor*}{Corollary}
\newtheorem{remark}{Remark}
\newtheorem{fact}{Fact}

\newcommand{\namedref}[2]{\hyperref[#2]{#1~\ref*{#2}}}

\newcommand{\eps}{\ensuremath{\epsilon}\xspace}

\definecolor{darkred}{rgb}{0.5, 0, 0} 
\definecolor{darkblue}{rgb}{0,0,0.5} 
\hypersetup{
    colorlinks=true,
    linkcolor=darkred,
    citecolor=darkblue,
    urlcolor=darkblue   
}

\newcommand{\bg}{\ensuremath{{\bf g}}\xspace}
\newcommand{\bh}{\ensuremath{{\bf h}}\xspace}

\newcommand{\bx}{\ensuremath{{\bf x}}\xspace}

\newcommand{\bw}{\ensuremath{{\bf w}}\xspace}

\newcommand{\calS}{\ensuremath{\mathcal{S}}\xspace}

\newcommand{\bbE}{\ensuremath{\mathbb{E}}\xspace}

\renewcommand{\paragraph}[1]{\smallskip\noindent{\bf #1}~}

%% file: abstract.tex
A distinguishing characteristic of federated learning is that the (local) client data could have statistical heterogeneity. This heterogeneity has motivated the design of \emph{personalized} learning, where individual (personalized) models are trained, through collaboration. There have been various personalization methods proposed in literature, with seemingly very different forms and methods ranging from use of a single global model for local regularization and model interpolation, to use of multiple global models for personalized clustering, etc. In this work, we begin with a generative framework that could potentially unify several different algorithms as well as suggest new algorithms.  We apply our generative framework to personalized estimation, and connect it to the classical empirical Bayes' methodology. We develop private personalized estimation under this framework. We then use our generative framework for learning, which unifies several known personalized FL algorithms and also suggests new ones; we propose and study a new algorithm \texttt{AdaPeD} based on a Knowledge Distillation, which numerically outperforms several known algorithms. We also develop privacy for personalized learning methods with guarantees for user-level privacy and composition. We numerically evaluate the performance as well as the privacy for both the estimation and learning problems, demonstrating the advantages of our proposed methods.

%% file: introduction.tex
A fundamental question is how one can use collaboration to help
personalized learning/estimation for users who have limited data
that they want to keep private, and that are 
potentially generated according to  heterogeneous (unknown) distributions.
This is  motivated by federated learning (FL),
where users collaboratively train machine learning models
by leveraging local data residing on user (edge) devices but without actually sharing
the local data
\cite{mcmahan2017communicationefficient,kairouz2021advances}. Due to the (statistical)
heterogeneity in local data, it has been realized that a single global
learning model might perform poorly for individual clients, hence
motivating the need for \emph{personalized} learning/estimation with
individual (personalized) schemes, obtained potentially through
collaboration. There have been a plethora of different personalization
methods proposed in
literature \cite{fallah2020personalized,dinh2020personalized,deng2020adaptive,mansour2020approaches,acar2021debiasing,li2021ditto,ozkara2021quped,zhang2021personalized,hu2020personalized},
with seemingly very different forms and methods. Despite these
approaches, it is unclear what is the fundamental statistical
framework that underlies them. The goal of this paper is to develop a
framework that could unify them and also lead to new algorithms and
performance bounds.

This problem is connected to the classical empirical Bayes' method, pioneered by 
Stein \cite{Stein56,james1961estimation}. Stein studied jointly estimating
Gaussian individual parameters, generated by an unknown (parametrized)
Gaussian population distribution. They showed a surprising result that
one can enhance the estimate of individual parameters based on the
observations of a population of Gaussian random variables
with \emph{independently} generated parameters from an unknown
(parametrized) Gaussian population distribution. Effectively, this
methodology advocated
\emph{estimating} the unknown population distribution using the individual independent
samples, and then using it effectively as an empirical prior for
individual estimates.\footnote{This was shown to uniformly improve the
mean-squared error averaged over the population, compared to an
estimate using just the single local sample.} This was studied for
Bernoulli variables with heterogeneously generated individual
parameters by Lord \cite{lord1967} and the optimal error bounds for
maximum likelihood estimates for population distributions were recently
developed in \cite{vinayak2019maximum}.

Despite this strong philosophical connection, personalized estimation
requires to address several new challenges. A main difference is that
classical  empirical Bayes' estimation is not studied for the
distributed case, which brings in information (communication and
privacy) constraints that we here study.\footnote{The homogeneous case for
distributed estimation is well-studied (see \cite{Zhang:EECS-2016-47}
and references).}  Moreover, it is not developed for \emph{distributed
learning}, where clients want to build \emph{local} predictive models
with limited local samples; we develop this framework and algorithms
in Section~\ref{sec:learning}.  However, Stein's idea serves as an
inspiration for our generative framework for personalized learning.

We consider a (statistical) generative model, where there is
  an \emph{unknown} population distribution $\bbP$ from 
  which local parameters $\{\vct{\theta}_i\}$ are generated, which in
  turn generate the local data through the distribution
  $\mathbb{Q}(\vct{\theta}_i)$. We mostly focus on
  a \emph{parametrized} population distribution
  $\mathbb{P}(\Gamma)$ (for unknown parameters $\Gamma$) for simplicity, though this can also be
  applied to non-parametric cases.  For the estimation problem, where
  users want to estimate their local parameter $\vct{\theta}_i$, we
  study the distributed case with information constraints for several
  examples of $\mathbb{P}$ and $\mathbb{Q}$ (see
  Theorems~\ref{thm:gaussian_est-q} and \ref{thm:bern_private_estimate} and
  results in the appendices). We estimate the (parametrized)
  population distribution under these information constraints and use
  this as an empirical prior for local estimation.  The effective amplification
  of local samples through collaboration, in
  Section \ref{sec:estimation}, gives insight about when collaboration
  is most useful.

In the learning framework studied in Section \ref{sec:learning},
the \emph{unknown} population distribution $\mathbb{P}$ generates
local parameters $\{\vct{\theta}_i\}$ which in turn parametrize
(unknown) local distributions $p_{\vct{\theta}_i}$, which is the local
generative model. The local samples are generated from this
distribution $p_{\vct{\theta}_i}$, and the local learner builds a
model using this data. However, inspired by the estimation approach,
one can \emph{estimate} the population parameters using these
(distributed) local data, and use such an \emph{empirical} estimate as
a prior for the local model. This is done through iterative
optimization, alternating between building a population 
model and using it to refine the local model. We develop this
framework in Section \ref{sec:learning}, and show that under different
parametric models of the population distribution, one connects to
several well-known personalized FL algorithms. For example, a
parametrized Gaussian population model, gives a
$\ell_2$-regularization between global and local models advocated
in \cite{dinh2020personalized,ozkara2021quped,hanzely2020federated,hanzely2020lower,li2021ditto};
one can also use other population models (\emph{e.g.,} to get
$\ell_1$-regularization). If the population model is a mixture, it
connects to the algorithms developed
in \cite{marfoq2021federated,zhang2021personalized,mansour2020approaches,ghosh2020efficient,smith2017federated}. However,
not every (parametrized) population model can be written as a mixture
distribution, and therefore our framework gives flexibility. For a
population model which relates to a Knowledge Distillation (KD)
regularization we connect to \cite{ozkara2021quped}. 
As one can observe, there are many other methods one can study using such a
framework.  As an illustration, when one parametrizes this population
model to include uncertainty about scaling of such a distribution, we
obtain a new algorithm, which we term \texttt{AdaPeD} (in Section~\ref{sec:adaped}), and also develop its
privatized version, which we term \dpadaped (in Section~\ref{sec:dp-adaped}), with guarantees on user-level
privacy and composition. We also develop other algorithms and results using this generative learning framework in the Appendices.

\paragraph{Contributions.} Besides the contribution of developing a statistical generative model described above (and in Sections \ref{sec:estimation} and \ref{sec:learning}), we also show the following:
\begin{itemize}[leftmargin=*,topsep=0pt,nosep]
\item We develop results for personalized estimation under information constraints for heterogeneous data under our generative model. This also allows us to identify regimes where collaboration helps with performance.
\item We extend the empirical Bayes' philosophy to personalized learning and connect it to several recently studied personalized FL algorithms. The framework also enables us to develop new algorithms.
\item We develop a new personalized learning algorithm, \texttt{AdaPeD}, which uses a KD regularization and adapts to relevance of local and population data iteratively. We also develop \dpadaped, a privatized version of \adaped, and give theoretical guarantees under user-level privacy and composition.
\item Finally, we give numerical results for both synthetic and real data for both personalized estimation and learning, and show that \adaped\ performs better than several state-of-the-art personalized FL algorithms.

\end{itemize}

\input{RelatedWork}

%% file: RelatedWork.tex
\paragraph{Related Work.} We believe that ours is the first general framework that helps with the design of personalized algorithms for learning and estimation.  Our work can be seen in the intersection of personalized learning, estimation, and privacy.

\paragraph{Personalized FL:}
As mentioned earlier, there has been a significant interest in personalized FL over the past few years. Recent work adopted different approaches for learning personalized models: 
{\sf(i)} \emph{Meta-learning:} first learn a global model and then personalize it locally by updating it using clients' local data \cite{fallah2020personalized,acar2021debiasing,kohdak2019adaptive}; these methods are based on Model Agnostic Meta Learning (MAML) \cite{jiang2019improving}, and could be disadvantaged by not jointly building local and global models as done in several other works. 
{\sf (ii)} \emph{Regularization:} Combine global and local models throughout the training \cite{deng2020adaptive,mansour2020approaches,hanzely2020federated}. In particular \cite{hanzely2020federated,hanzely2020lower,dinh2020personalized} augment the traditional FL  objective via a penalty term that enables collaboration between global and personalized models; such a regularizer fits into our generative framework through different choices of the parametrized population distribution, as discussed in Section \ref{sec:learning}.
{\sf (iii)} \emph{Clustered FL:} considers multiple global models to collaborate among only those clients that share similar personalized models \cite{zhang2021personalized,mansour2020approaches,ghosh2020efficient,smith2017federated}; a generalization of this method used soft clustering and mixing these multiple global models in \cite{marfoq2021federated}; all these methods fit into our generative framework using mixture population distributions as discussed in Section \ref{sec:learning}.
{\sf (iv)} \emph{Knowledge distillation} of global model to personalized local models \cite{lin2020ensemble} and jointly training global and local models using KD \cite{lin2020ensemble,li2019fedmd,shen2020federated,ozkara2021quped}. The distillation methods could also be explained using our framework as discussed in Section \ref{sec:learning} and appendices.
{\sf (v)} \emph{Multi-task Learning (MTL):} This can enable specific relationships between client models \cite{dinh2020personalized,hanzely2020federated,smith2017federated,vanhaesebrouck2017decentralized,zantedeschi2020fully}. 
{\sf (vi)} \emph{Common representations:} There have been several recent works on assuming that users have a shared low-dimensional  subspace and each individual model is based on this (see \cite{prateek2021differentially,du2021fewshot,RaghuRBV20,tian2020rethinking} and references therein). 
As explained in  Section \ref{sec:learning}, many of these approaches can also be cast in our generative framework.

\paragraph{Privacy for Personalized Learning.} There has been a lot of work in privacy for FL when the goal is to learn a \emph{single} global model (see \cite{aistats_GirgisDDKS21} and references therein); though there are fewer papers that address user-level privacy \cite{liu-theertha-user-level-dp20,levy_user-level-dp21,ghazi-user-level-dp-correlated21}. There has been more recent work on applying these ideas to learn personalized models \cite{prateek2021differentially,geyer2017differentially,hu2020personalized,Li2020Differentially}. These are for specific algorithms/models, \emph{e.g.,} \cite{prateek2021differentially} focuses on the common representation model described earlier or on item-level privacy \cite{prateek2021differentially,hu2020personalized,Li2020Differentially}. 

\paragraph{Paper Organization.}
In Sections~\ref{sec:estimation} and~\ref{sec:learning}, we set up our generative framework for the personalized estimation and learning, respectively, and show how our framework explains the underlying statistical model behind several personalized FL algorithms from literature. For estimation, we study the Gaussian and the Bernoulli models with and without information constraints, and for learning, we present our new personalized learning algorithm \adaped\ and also \dpadaped, along with its privacy guarantees. 
Section~\ref{sec:experiments} provides numerical results. In Sections~\ref{sec:proof-est},~\ref{sec:proof-learning} we provide the proofs for our analytical results. Omitted details, such as background on DP, are provided in appendices.

%% file: estimation.tex
\section{Personalized Estimation}\label{sec:estimation}

We consider a client-server architecture, where there are $m$ clients. Let $\bbP(\Gamma)$ denote a global population distribution that is parameterized by an unknown $\Gamma$ and let $\vct{\theta}_1,\ldots,\vct{\theta}_m$ are sampled i.i.d.\ from $\bbP(\Gamma)$ and are unknown to the clients. Client $i$ is given a dataset $X_i:=(X_{i1},\ldots,X_{in})$, where $X_{ij}, j\in[n]$ are sampled i.i.d.\ from some distribution $\bbQ(\vct{\theta}_i)$, parameterized by $\vct{\theta}_i\in\bbR^d$. Note that heterogeneity in clients' datasets is induced through the variance in $\bbP(\Gamma)$, and if the variance of $\bbP(\Gamma)$ is zero, then all clients observe i.i.d.\ datasets sampled from the {\em same} underlying distribution. 

The goal at client $i$ for all $i\in[m]$ is to estimate $\vct{\theta}_i$ through the help of the server. We focus on one-round communication schemes, where client $j$ applies a (potentially randomized) mechanism $q$ on its dataset $X_j$ and sends $q_j:=q(X_j)$ to the server, who aggregates the received messages and broadcasts that to all clients. The aggregated message at the server is denoted by $\mathsf{Agg}(q_1,\ldots,q_m)$. Based on $(X_i,\mathsf{Agg}(q_1,\ldots,q_m))$, client $i$ outputs an estimate $\htheta_i$ of $\vct{\theta}$.

We measure the performance of our estimator through the Bayesian risk for mean squared error (MSE).
 More specifically, given a true prior distribution $\bbP$ and the associated true prior density $\pi$, true local parameter $\vct{\theta}_i\sim \bbP$, and an estimator $\htheta_i$; we are interested in bounding the MSE:
\begin{align}\label{eqn:mse-defn}
	\mathbb{E}_{\vct{\theta}_i\sim \bbP} \mathbb{E}_{\htheta_i,q,X_1,...,X_m} \|\htheta_i - \vct{\theta}_i\|^2 = \int \mathbb{E}_{\htheta_i,q,X_1,...,X_m} \|\htheta_i - \vct{\theta}_i\|^2 \pi(\vct{\theta}_i) d\vct{\theta}_i,
\end{align}
where $\htheta=\htheta(X_i,\mathsf{Agg}(q_1,\ldots,q_m))$.


The above-described generative framework can model many different scenarios, and we will study in detail three settings: Gaussian model, Bernoulli model, and Mixture model, out of which, first two we will in Sections~\ref{sec:est_gaussian},~\ref{sec:est_bernoulli} respectively, and the third will be presented in Section~\ref{sec:est_mixture},~\ref{app:est-mixture}.

%% file: estimation_gaussian.tex
\subsection{Gaussian Model}\label{sec:est_gaussian}

In the Gaussian setting, 
$\bbP(\Gamma)=\calN(\vct{\mu},\sigmatheta^2\bbI_d)$ and $\bbQ(\vct{\theta}_i)=\calN(\vct{\theta}_i,\sigma_x^2\bbI_d)$ for all $i\in[m]$, which implies that $\vct{\theta}_1,\ldots,\vct{\theta}_m\sim\calN(\vct{\mu},\sigmatheta^2\bbI_d)$ i.i.d.\ and $X_{i1},\ldots,X_{in}\sim\calN(\vct{\theta}_i,\sigma_x^2\bbI_d)$ i.i.d.\ for $i\in[m]$. Here, $\sigmatheta\geq0,\sigma_x>0$ are known, and $\vct{\mu},\vct{\theta}_1,\ldots,\vct{\theta}_m$ are unknown.
For the case of a single local sample this is identical to the classical  James-Stein estimator~\cite{james1961estimation}; Theorem \ref{thm:gauss_estimate} does a simple extension for multiple local samples and is actually a stepping stone for the information constrained estimation result of Theorem \ref{thm:gaussian_est-q}.
Omitted proofs and details from this subsection are provided in Section~\ref{app:est-gaussian}. 

\paragraph{Our proposed estimator.}
Since there is no distribution on $\vct{\mu}$, and given $\vct{\mu}$, we know the distribution of $\vct{\theta}_i$'s, and subsequently, of $X_{ij}$'s. So, we consider the maximum likelihood estimator:
\begin{align}
	\htheta_1,\ldots,\htheta_m,\hmu := \argmax_{\vct{\theta}_1,\ldots,\vct{\theta}_m,\vct{\mu}} p_{\{\vct{\theta}_i,X_i\}|\vct{\mu}}\left(\vct{\theta}_1,\ldots,\vct{\theta}_m,X_1,\ldots,X_m|\vct{\mu}\right) \label{eqn:post_estimator}
\end{align}
\begin{theorem}\label{thm:gauss_estimate}
Solving \eqref{eqn:post_estimator} yields the following closed form expressions for $\hmu$ and $\htheta_1,\ldots,\htheta_m$:
\begin{align}\label{eqn:estimator}
\hmu = \frac{1}{m}\sum_{i=1}^m \overline{X}_i \qquad \text{ and } \qquad 		
\htheta_{i}=a\overline{X}_i+(1-a)\hmu, \text{ for } i\in[m], \quad \text{ where } a=\frac{\sigmatheta^{2}}{\sigmatheta^{2}+\nicefrac{\sigma_{x}^{2}}{n}}.
\end{align}
The above estimator achieves the MSE: 
$\mathbb{E}_{\vct{\theta}_i,X_1,\hdots,X_m}\|\htheta_i-\vct{\theta}_i\|^2 \leq \frac{d\sigma_x^2}{n}\Big(\frac{1-a}{m}+a\Big).$
\end{theorem}

The estimators in \eqref{eqn:estimator} suggest the following scheme: Each client $i\in[m]$ sends the average $\overline{X}_i=\frac{1}{n}\sum_{j=1}^nX_{ij}$ of its local dataset and sends that to the server. Upon receiving all $m$ averages, the server further takes the average of the messages to compute $\hmu = \frac{1}{m}\sum_{i=1}^m \overline{X}_i$ and sends that back to all clients. So, the mechanism $q$ and the aggregation function $\mathsf{Agg}$ in \eqref{eqn:mse-defn} are just the average functions.

\begin{remark}[Personalized estimate vs.\ local estimate]
When $\sigmatheta\to0$, then $a\to0$, which implies that $\htheta_i\to\hmu$ and MSE $\to d\sigma_x^2/mn$. Otherwise, when $\sigmatheta^2$ is large in comparison to $\sigma_x^2/n$ or $n\to\infty$, then $a\to1$, which implies that $\htheta_i\to\overline{X}_i$ and MSE $\to d\sigma_x^2/n$. 
These conform to the facts that {\sf (i)} when there is no heterogeneity, then the global average is the best estimator, and {\sf (ii)} when heterogeneity is not small, but we have a lot of local samples, then the local average is the best estimator.
Observe that the multiplicative gap between the MSE of the proposed personalized estimator and the MSE of the local estimator (based on local data only, which gives an MSE of $d\sigma_x^2/n$) is given by $(\frac{1-a}{m}+a)\leq 1$ that proves the superiority of the personalized model over the local model, which is equal to $1/m$ when $\sigmatheta=0$ and equal to $0.01$ when $m=10^{4},n=100$ and $\sigma_x^2=10,\sigmatheta^{2}=10^{-3}$, for example.  
\end{remark}

\begin{remark}[Optimality of our personalized estimator]
In Section~\ref{app:est-gaussian}, we show the minimax lower bound: $\inf_{\widehat{\vct{\theta}}} \sup_{\vct{\theta} \in \Theta} \mathbb{E}_{X\sim\calN(\vct{\theta},\sigma_x^2)}\|\widehat{\vct{\theta}}(X)-\vct{\theta}\|^2 \geq \frac{d\sigma_x^2}{n}\big(\frac{1-a}{m}+a\big)$, which exactly matches the upper bound on the MSE in Theorem~\ref{thm:gauss_estimate}, thus establishes the optimality our personalized estimator in \eqref{eqn:estimator}. 
\end{remark}

\paragraph{Privacy and communication constraints.}
Observe that the scheme presented above does not protect privacy of clients' data and messages from the clients to the server can be made communication-efficient. These could be achieved by employing specific mechanisms $q$ at clients: For privacy, we can take a differentially-private $q$, and for communication-efficiency, we can take $q$ to be a quantizer. 
Inspired by the scheme presented above, here we consider $q$ to be a function $q:\bbR^d\to\calY$, that takes the average of $n$ data points as its input, and the aggregator function $\mathsf{Agg}$ to be the average function. Define $\hmu_q:=\frac{1}{m}\sum_{i=1}^mq(\overline{X}_i)$ and consider the following personalized estimator for the $i$-th client:
\begin{equation}\label{eqn:person_const}
	\htheta_i=a\overline{X}_i + (1-a)\hmu_q, \qquad \text{ for some } a\in[0,1].
\end{equation}

%
\begin{theorem}\label{thm:gaussian_est-q}
Suppose for all $\bx\in\bbR^d$, $q$ satisfies $\bbE[q(\bx)]=\bx$ and $\bbE\|q(\bx)-\bx\|^2\leq d\sigma_q^2$ for some finite $\sigma_q$.
Then the personalized estimator in \eqref{eqn:person_const} has MSE:
\begin{align}\label{eqn:gaussian_generic-q}
\mathbb{E}_{\vct{\theta}_i,q,X_1,\ldots,X_m}\|\htheta_i-\vct{\theta}_i\|^2 \leq \frac{d\sigma_x^2}{n}\Big(\frac{1-a}{m}+a\Big) \qquad \text{ where }\qquad a=\frac{\sigma_{\theta}^{2}+\nicefrac{\sigma_q^2}{m-1}}{\sigma_{\theta}^{2}+\nicefrac{\sigma_q^2}{m-1}+\nicefrac{\sigma_x^{2}}{n}}.
\end{align}
Furthermore, assuming $\vct{\mu}\in[-r,r]$ for some constant $r$ (but $\vct{\mu}$ is unknown), we have:
\begin{enumerate}
\item {\it Communication efficiency:} For any $k\in\mathbb{N}$, there is a $q$ whose output can be represented using $k$-bits (i.e., $q$ is a quantizer) that achieves the MSE in \eqref{eqn:gaussian_generic-q} with probability at least $1-\nicefrac{2}{mn}$ and with $\sigma_q = \frac{b}{(2^k-1)}$, 
where $b=r+\sigmatheta\sqrt{\log(m^2n)}+\frac{\sigma_x}{\sqrt{n}}\sqrt{\log(m^2n)}$. 
\item {\it Privacy:} For any $\epsilon_0\in(0,1),\delta>0$, there is a $q$ that is user-level $(\epsilon_0,\delta)$-locally differentially private, that achieves the MSE in \eqref{eqn:gaussian_generic-q} with probability at least $1-\nicefrac{2}{mn}$ and with $\sigma_q=\frac{b}{\epsilon_0}\sqrt{8\log(2/\delta)}$, 
where $b=r+\sigmatheta\sqrt{\log(m^2n)}+\frac{\sigma_x}{\sqrt{n}}\sqrt{\log(m^2n)}$. 
\end{enumerate}
\end{theorem}

%% file: estimation_bernoulli.tex
\subsection{Bernoulli Model}\label{sec:est_bernoulli}
For the Bernoulli model, $\bbP$ is supported on $[0,1]$, and $p_1,\ldots,p_m$ are sampled i.i.d.\ from $\bbP$, and client $i$ is given $n$ i.i.d.\ samples $X_{i1},\ldots,X_{in}\sim\Bern(p_i)$.
This setting has been studied by \cite{tian2017learning, vinayak2019maximum} for estimating $\bbP$, whereas, our goal is to estimate individual parameter $p_i$ at client $i$ using the information from other clients. In order to derive a closed form MSE result, we assume that $\bbP$ is the Beta distribution.
\footnote{Beta distribution has a  density $\Beta(\alpha,\beta)=\frac{1}{B(\alpha,\beta)}x^{\alpha-1}(1-x)^{\beta-1}$ is defined for $\alpha,\beta>0$ and $x\in[0,1]$, where $B(\alpha,\beta)$ is a normalizing constant. Its mean is $\frac{\alpha}{\alpha+\beta}$ and the variance is $\frac{\alpha\beta}{(\alpha+\beta)^2(\alpha+\beta+1)}$.} Here, $\Gamma=(\alpha,\beta),p_1,\ldots,p_m$ are unknown, and client $i$'s goal is to estimate $p_i$ such that the Bayesian risk $\bbE_{p_i\sim\pi}\bbE_{\hatp_i,X_1,\ldots,X_m}(\hatp_i - p_i)^2$ is minimized, where $\pi$ denotes the density of the Beta distribution. Omitted proofs and details from this subsection are provided in Section~\ref{app:est-bernoulli}.

\paragraph{When $\alpha,\beta$ are known.} Analogous to the Gaussian case, we can show that if $\alpha,\beta$ are known, then the posterior mean estimator has a closed form expression: $\hatp_i=a\overline{X}_i + (1-a)\frac{\alpha}{\alpha+\beta}$ (where $a=\nicefrac{n}{\alpha+\beta+n}$) and achieves the MSE: $\bbE_{p_i\sim\pi}\bbE_{\hatp_i,X_1,\ldots,X_m}(\hatp_i - p_i)^2\leq\frac{\alpha\beta}{n(\alpha+\beta)(\alpha+\beta+1)}\frac{n}{\alpha+\beta+n}$. 
Note that $\overline{X}_i$ is the estimator based only on the local data and $\nicefrac{\alpha}{(\alpha+\beta)}$ is the true global mean. Observe that when $n\to\infty$, then $a\to1$, which implies that $\hatp_i\to\overline{X}_i$. Otherwise, when $\alpha+\beta$ is large (i.e., the variance of the beta distribution is small), then $a\to0$, which implies that $\hatp_i\to\nicefrac{\alpha}{(\alpha+\beta)}$. Both these conclusions conform to the conventional wisdom as mentioned in the Gaussian case.
It can be shown that the local estimate $\overline{X}_i$ achieves the Bayesian risk of $\bbE_{p_i\sim\pi}\bbE_{X_i}[(\overline{X}_i-p_i)^2]=\nicefrac{\bbE_{p_i\sim\pi}(p_i(1-p_i))}{n} = \nicefrac{\alpha \beta}{n(\alpha+\beta)(\alpha+\beta+1)}$, which implies that the personalized estimation with perfect prior always outperforms the local estimate with a multiplicative gain $a=\nicefrac{n}{(n+\alpha+\beta)}\leq1$.

\paragraph{When $\alpha,\beta$ are unknown.} 
In this case, inspired by the above discussion, a natural approach would be to estimate the global mean $\mu=\nicefrac{\alpha}{(\alpha+\beta)}$ and the weight $a=\nicefrac{n}{(\alpha+\beta+n)}$, and use that in the above estimator. 
Note that for $a$, we need to estimate $\alpha+\beta$, which is equal to $\nicefrac{\mu(1-\mu)}{\sigma^2}-1$, where $\sigma^2=\nicefrac{\alpha\beta}{(\alpha+\beta)^2(\alpha+\beta+1)}$ is the variance of the beta distribution. Therefore, it is enough to estimate $\mu$ and $\sigma^2$ for the personalized estimators $\{\hatp_i\}$.
In order to make our calculations of MSE simpler, instead of making one estimate of $\mu,\sigma^2$ for all clients, we let each client make its own estimate of $\mu,\sigma^2$ (without using their own data) as: $\widehat{\mu}_i=\frac{1}{m-1}\sum_{l\neq i}\overline{X}_l$ and $\hsigma_i^2=\frac{1}{m-2}\sum_{l\neq i}(\overline{X}_l-\widehat{\mu}_l)^2$,\footnote{Upon receiving $\{\overline{X}_i\}$ from all clients, the server can compute $\{\widehat{\mu}_i,\hsigma_i^2\}$ and sends $(\widehat{\mu}_i,\hsigma_i^2)$ to the $i$-th client.} and then define the local weight as $\hata_i=\frac{n}{\nicefrac{\widehat{\mu}_i(1-\widehat{\mu}_i)}{\hsigma_i^2}-1+n}$. Using these, client $i\in[m]$ uses the following personalized estimator:
\begin{equation}\label{eqn:bern_estimate_unknown}
	\hatp_i=\hata_i\overline{X}_i + (1-\hata_i)\widehat{\mu}_i.
\end{equation}
\begin{theorem}\label{thm:bern_estimate}
With probability at least $1-\frac{1}{mn}$, the MSE of the personalized estimator in~\eqref{eqn:bern_estimate_unknown} is given by:
$\bbE_{p_i\sim\pi}\bbE_{X_1,\ldots,X_m}(\hatp_i - p_i)^2 \leq \bbE[\hata_i^2]\big(\frac{\alpha\beta}{n(\alpha+\beta)(\alpha+\beta+1)}\big)+\bbE[(1-\hata_i)^2]\big(\frac{\alpha\beta}{(\alpha+\beta)^2(\alpha+\beta+1)}+\frac{3\log(4m^2n)}{m-1}\big)$.

\end{theorem}

\begin{remark}
When $n\rightarrow \infty$, then $\hata_i\to1$, which implies that MSE tends to the MSE of the local estimator $\overline{X}_i$, which means if local samples are abundant, collaboration does not help much. When $\sigma^2=\nicefrac{\alpha\beta}{(\alpha+\beta)^2(\alpha+\beta+1)}\to0$, i.e. there is very small heterogeneity in the system, then $\hata_i\to0$, which implies that MSE tends to the error due to moment estimation (the last term in the MSE in Theorem~\ref{thm:bern_estimate}.
\end{remark}

\paragraph{Privacy constraints.}
Let $\epsilon_0>0$ be the privacy parameter. Define $q^{\priv}:[0,1]\to \bbR$ be a private mechanism defined as follows for any $x\in[0,1]$:
\begin{equation}
    q^{\priv}(x)=\left\{\begin{array}{ll}
         \frac{-1}{e^{\epsilon_0}-1}& \text{ w.p. } \frac{e^{\epsilon_0}}{e^{\epsilon_0}+1}-x\frac{e^{\epsilon_0}-1}{e^{\epsilon_0}+1}, \\
         \frac{e^{\epsilon_0}}{e^{\epsilon_0}-1}& \text{ w.p. } \frac{1}{e^{\epsilon_0}+1}+x\frac{e^{\epsilon_0}-1}{e^{\epsilon_0}+1}.
    \end{array}
    \right.
\end{equation}
The mechanism $q^{\priv}$ is unbiased and satisfies user-level $\epsilon_0$-LDP. 
Thus, the $i$th client sends $q^{\priv}(\overline{X}_i)$ to the server, which computes $\widehat{\mu}_i^{\priv} =\frac{1}{m-1}\sum_{l\neq i} q^{\priv}(\overline{X}_l)$ and the variance $\hsigma_{i}^{2(\priv)}=\frac{1}{m-2}\sum_{l\neq i} (q^{\priv}(\overline{X}_l))-\widehat{\mu}_l^{\priv})^2$ for all $i\in[m]$ and sends $(\widehat{\mu}_i^{\priv},\hsigma_{i}^{2(\priv)})$ to client $i$. Upon receiving this, client $i$ defines $\hata_i^{\priv}=\frac{n}{\nicefrac{\widehat{\mu}_i^{\priv}(1-\widehat{\mu}_i^{\priv})}{\hsigma_i^{2(\priv)}}+n}$ and uses 
 $\hatp_i^{\priv} = \hata_i^{\priv}\overline{X}_i + (1-\hata_i^{\priv})\widehat{\mu}^{\priv}$ as its personalized estimator for $p_i$.

\begin{theorem}\label{thm:bern_private_estimate}
With probability at least $1-\frac{1}{mn}$, the MSE of the personalized estimator $\hatp_i^{\priv}$ defined above is given by:
$\bbE_{p_i\sim\pi}\bbE_{q^{\priv},X_1,\ldots,X_m}(\hatp_i^{\priv} - p_i)^2 \leq \bbE[(\hata_i^{\priv})^2]\big(\frac{\alpha\beta}{n(\alpha+\beta)(\alpha+\beta+1)}\big)+\bbE[(1-\hata_i^{\priv})^2]\big(\frac{\alpha\beta}{(\alpha+\beta)^2(\alpha+\beta+1)}+\frac{(e^{\epsilon_0}+1)^2\log(4m^2n)}{3(e^{\epsilon_0}-1)^2(m-1)}\big)$.
\end{theorem}

%% file: estimation_discrete.tex
\subsection{Mixture Model}\label{sec:est_mixture}
Here we consider a discrete prior distribution, in Section~\ref{sec:learning} we will consider a Gaussian mixture prior as well. Consider a set of $m$ clients, where the $i$-th client has a local dataset $X_{i}=\left(X_{i1},\ldots,X_{in}\right)$ of $n$ samples for $i\in[m]$, where $X_{ij}\in\mathbb{R}^{d}$. The local samples $X_i$ of the $i$-th client are drawn i.i.d. from a Gaussian distribution $\calN(\vct{\theta}_i,\sigma_x^2\bbI_d)$ with unknown mean $\vct{\theta}_i$ and known variance $\sigma_x^2\mathbb{I}_{d}$.

In this section, we assume that the personalized models $\vct{\theta}_1,\ldots,\vct{\theta}_m$ are drawn i.i.d. from a discrete distribution $\bbP=\left[p_1,\ldots,p_k\right]$ for given $k$ candidates $\vct{\mu}_1,\ldots,\vct{\mu}_k\in\mathbb{R}^{d}$. In other works, $\Pr[\vct{\theta}_i=\vct{\mu}_l]=p_l$ for $l\in[k]$ and $i\in[m]$. The goal of each client is to estimate her personalized model $\lbrace\vct{\theta}_i\rbrace$ that minimizes the mean square error defined as follows:

\begin{equation}
	\text{MSE} = \mathbb{E}_{\lbrace\vct{\theta}_i,X_i\rbrace}\|\vct{\theta}_i-\hat{\vct{\theta}}_i\|^2,    
\end{equation}
where the expectation is taken with respect to the personalized models $\vct{\theta}_i$ and the local samples $\lbrace X_{ij}\sim\mathcal{N}(\vct{\theta}_i,\sigma_x^2\mathbb{I}_{d}) \rbrace$. Furthermore, $\hat{\vct{\theta}}_i$ denotes the estimate of the personalized model $\vct{\theta}_i$ for $i\in[m]$. 

First, we start with a simple case when the clients have perfect knowledge of the prior distribution, i.e., the $i$-th client knows the $k$ Gaussian distributions $\mathcal{N}\left(\vct{\mu}_1,\sigma_{\vct{\theta}}^{2}\right),\ldots,\mathcal{N}\left(\vct{\mu}_k,\sigma_{\vct{\theta}}^{2}\right)$ and the prior distribution $\vct{\alpha}=\left[\alpha_1,\ldots,\alpha_k\right]$. This will serve as a stepping stone to handle the more general case when the prior distribution is unknown.

\subsubsection{When the Prior Distribution is Known}

In this case, the $i$-th client does not need the data of the other clients as she has a perfect knowledge about the prior distribution. 
\begin{theorem}\label{thm:discrete_perfect}
	For given a perfect knowledge $\vct{\alpha}=[\alpha_1,\ldots,\alpha_k]$ and $\mathcal{N}\left(\vct{\mu}_1,\sigma_{\vct{\theta}}^{2}\right),\ldots,\mathcal{N}\left(\vct{\mu}_k,\sigma_{\vct{\theta}}^{2}\right)$, the optimal personalized estimator that minimizes the MSE is given by:
	\begin{equation}\label{eqn:discrete_perfect_estimator}
		\hat{\vct{\theta}}_i =\sum_{l=1}^{k} a_{l}^{(i)}\vct{\mu}_l,
	\end{equation}
	where $\alpha_{l}^{(i)}=\frac{p_l\exp{\left(-\frac{\sum_{j=1}^{n}\|X_{ij}-\vct{\mu}_l\|^2}{2\sigma_x^2}\right)}}{\sum_{s=1}^{k}p_s\exp{\left(-\frac{\sum_{j=1}^{n}\|X_{ij}-\vct{\mu}_s\|^2}{2\sigma_x^2}\right)}}$ denotes the weight associated to the prior model $\vct{\mu}_l$ for $l\in[k]$.
\end{theorem}

The optimal personalized estimation in~\eqref{eqn:discrete_perfect_estimator} is a weighted summation over all possible candidates vectors $\vct{\mu}_1,\ldots,\vct{\mu}_k$, where the weight $\alpha_{l}^{(i)}$ increases if the prior $p_l$ increases and/or the local samples $\lbrace X_{ij}\rbrace$ are close to the model $\vct{\mu}_l$ for $l\in[k]$. Observe that the optimal estimator $\hat{\vct{\theta}}_i$ in Theorem~\ref{thm:discrete_perfect} that minimizes the MSE is completely different from the local estimator $\left(\frac{1}{n}\sum_{j=1}^{n}X_{ij}\right)$. Furthermore, it is easy to see that the local estimator has the MSE $\left(\frac{d\sigma_x^2}{n}\right)$ which increases linearly with the data dimension $d$. On the other hand, the MSE of the optimal estimator in Theorem~\ref{thm:discrete_perfect} is a function of the prior distribution $\bbP=[p_1,\ldots,p_k]$, the prior vectors $\vct{\mu}_1,\ldots,\vct{\mu}_k$, and the local variance $\sigma_x^2$. Proof of Theorem~\ref{thm:discrete_perfect} is provided in Section~\ref{app:est-mixture}.

\subsubsection{When the Prior Distribution is Unknown}

Now, we consider a more practical case when the prior distribution $\bbP=[p_1,\ldots,p_k]$ and the candidates $\vct{\mu}_1,\ldots,\vct{\mu}_k$ are unknown to the clients. In this case, the clients collaborate with each other by their local data to estimate the priors $\bbP$ and $\vct{\mu}_1,\ldots,\vct{\mu}_k$, and then, each client uses the estimated priors to design her personalized model as in~\eqref{eqn:discrete_perfect_estimator}. 

We present Algorithm~\ref{algo:cluster} based on alternating minimization. The algorithm starts by initializing the local models $\lbrace\vct{\theta}_{i}^{(0)}:=\frac{1}{n}\sum_{j=1}^{n}X_{ij}\rbrace$. Then, the algorithm works in rounds alternating between estimating the priors $\bbP^{(t+1)}=[p_1^{(t+1)},\ldots,p_k^{(t+1)}]$, $\vct{\mu}_1^{(t+1)},\ldots,\vct{\mu}_k^{(t+1)}$ for given local models $\lbrace \vct{\theta}_i^{(t)}\rbrace$ and estimating the personalized models $\lbrace \vct{\theta}_i^{(t+1)}\rbrace$ for given global priors $\bbP^{(t+1)}$ and $\vct{\mu}_1^{(t+1)},\ldots,\vct{\mu}_k^{(t+1)}$. Observe that for given the prior information $\bbP^{(t)},\lbrace \vct{\mu}_l^{t}\rbrace$, each client updates her personalized model in Step~\ref{step:personal_update} which is the optimal estimator for given priors according to Theorem~\ref{thm:discrete_perfect}. On the other hand, for given personalized models $\lbrace \vct{\theta}_i^{(t)}\rbrace$, we estimate the priors $\bbP^{(t)},\lbrace \vct{\mu}_l^{t}\rbrace$ using clustering algorithm with $k$ sets in Step~\ref{step:clustering}. The algorithm $\mathsf{Cluster}$ takes $m$ vectors $\vct{a}_1,\ldots,\vct{a}_m$ and an integer $k$ as its input, and its goal is to generate a set of $k$ cluster centers $\vct{\mu}_1,\ldots,\vct{\mu}_k$ that minimizes $\sum_{i=1}^{m}\min_{l\in k}\|\vct{a}_i-\vct{\mu}_l\|^2$. Furthermore, these clustering algorithms can also return the prior distribution $\bbP$, by setting $p_l:=\frac{|\mathcal{S}_l|}{m}$, where $\mathcal{S}_l\subset \lbrace \vct{a}_1,\ldots,\vct{a}_m\rbrace$ denotes the set of vectors that are belongs to the $l$-th cluster. There are lots of algorithms that do clustering, but perhaps, Lloyd’s algorithm~\cite{lloyd1982least} and Ahmadian~\cite{ahmadian2019better} are the most common algorithms for $k$-means clustering. Our Algorithm~\ref{algo:cluster} can work with any clustering algorithm. 

\begin{algorithm}
	\caption{Personalized Estimation with Discrete Mixture Prior}
	{\bf Input:} Number of iterations $T$, local datasets $(X_{i1},\ldots,X_{in})$ for $i\in[m]$.\\
	\vspace{-0.3cm}
	\begin{algorithmic}[1] 	\label{algo:cluster}
		\STATE \textbf{Initialize} $\vct{\theta}_i^{0}=\frac{1}{n}\sum_{j=1}^{n}X_{ij}$  for $i\in[m]$.
		\FOR{$t=1$ \textbf{to} $T$}
		\STATE {\bf On Clients:}
		\FOR {$i=1$ \textbf{to} $m$:}
		\STATE Receive $\bbP^{(t)},\vct{\mu}_1^{(t)},\ldots,\vct{\mu}_k^{(t)}$ from the server 
		\STATE \label{step:personal_update} Update the personalized model:
		\[\vct{\theta}_i^{t}\gets \sum_{l=1}^{k}\alpha_{l}^{(i)}\vct{\mu}_l^{(t)}\qquad \text{ and } \qquad \alpha_{l}^{(i)}=\frac{p_l^{(t)}\exp{\left(-\frac{\sum_{j=1}^{n}\|X_{ij}-\vct{\mu}_l^{(t)}\|^2}{2\sigma_x^2}\right)}}{\sum_{s=1}^{k}p_s^{(t)}\exp{\left(-\frac{\sum_{j=1}^{n}\|X_{ij}-\vct{\mu}_s^{(t)}\|^2}{2\sigma_x^2}\right)}}\]
		\STATE Send $\vct{\theta}_i^t$ to the server
		\ENDFOR
		\STATE {\bf At the Server:}
		\STATE Receive $\vct{\theta}_1^{(t)},\ldots,\vct{\theta}_m^{(t)}$ from the clients
		\STATE\label{step:clustering} Update the global parameters:
		$\bbP^{(t)},\vct{\mu}_1^{(t)},\ldots,\vct{\mu}_k^{(t)}\gets \mathsf{Cluster}\left(\vct{\theta}_1^{(t)},\ldots,\vct{\theta}_m^{(t)},k\right)$
		\STATE Broadcast $\bbP^{(t)},\vct{\mu}_1^{(t)},\ldots,\vct{\mu}_k^{(t)}$ to all clients
		\ENDFOR
	\end{algorithmic}
	{\bf Output:} Personalized models $\vct{\theta}_1^{T},\ldots,\vct{\theta}_m^{T}$.
\end{algorithm}

\subsubsection{Privacy/Communication Constraints} 
In the personalized estimation Algorithm~\ref{algo:cluster}, each client shares her personalized estimator $\vct{\theta}_{i}^{(t)}$ to the server at each iteration which is not communication-efficient and violates the privacy. In this section we present ideas on how to design communication-efficient and/or private Algorithms for personalized estimation.

\begin{lemma}\label{lemm:concentration}
	Let $\vct{\mu}_1,\ldots\vct{\mu}_k\in\bbR^d$ be unknown means such that $\|\vct{\mu}_i\|_2\leq r$ for each $i\in[k]$. Let $\vct{\theta}_1,\ldots,\vct{\theta}_m\sim \bbP$, where $\bbP=[p_1,\ldots,p_k]$ and $p_l=\Pr[\vct{\theta}_i=\vct{\mu}_l]$. For $i\in[m]$, let $X_{i1},\ldots,X_{in}\sim\mathcal{N}(\vct{\theta}_i,\sigma_x^2)$, i.i.d. 
	Then, with probability at least $1-\frac{1}{mn}$, the following bound holds for all $i\in[m]$:
	\begin{equation}
		\left\|\frac{1}{n}\sum_{j=1}^{n}X_{ij}\right\|_2\leq 4\sqrt{d\frac{\sigma_x^2}{n}}+2\sqrt{\log(m^2n)\frac{\sigma_x^2}{n}}+r.    
	\end{equation}
	
\end{lemma}

Lemma~\ref{lemm:concentration} shows that the average of the local samples $\lbrace\overline{X}_i\rbrace$ has a bounded $\ell_2$ norm with high probability. Thus, we can design a communication-efficient estimation Algorithm as follows: Each client clips her personal model $\vct{\theta}_i^{(t)}$ within radius $4\sqrt{d\frac{\sigma_x^2}{n}}+2\sqrt{\log(m^2n)\frac{\sigma_x^2}{n}}+r$. Then, each client applies a vector-quantization scheme (e.g.,~\cite{bernstein2018signsgd,alistarh2017qsgd,girgis2021shuffled}) to the clipped vector before sending it to the server.

To design a private estimation algorithm with discrete priors, each client clips her personalized estimator $\vct{\theta}_{i}^{(t)}$ within radius $4\sqrt{d\frac{\sigma_x^2}{n}}+2\sqrt{\log(m^2n)\frac{\sigma_x^2}{n}}+r$. Then, we can use a differentially private algorithm for clustering (see e.g.,~\cite{stemmer2020locally} for clustering under LDP constraints and~\cite{ghazi2020differentially} for clustering under central DP constraints.).  Since, we run $T$ iterations in Algorithm~\ref{algo:cluster}, we can obtain the final privacy analysis $(\epsilon,\delta)$ using the strong composition theorem~\cite{dwork2014algorithmic}.

%% file: learning.tex
\section{Personalized Learning}\label{sec:learning}

Consider a client-server architecture with $m$ clients. There is an unknown global population distribution $\bbP(\Gamma)$\footnote{For simplicity we will consider this unknown population distribution $\bbP$ to be parametrized by unknown (arbitrary) parameters $\Gamma$.}  over $\bbR^d$ from which $m$ i.i.d.\ local parameters $\vct{\theta}_1,\ldots,\vct{\theta}_m\in\bbR^d$ are sampled. Each client $i\in[m]$ is provided with a dataset consisting of $n$ data points $\{(X_{i1},Y_{i1}),\ldots,(X_{in},Y_{in})\}$, where $Y_{ij}$'s are generated from $(X_{ij},\vct{\theta}_i)$ using some distribution $p_{\vct{\theta}_i}(Y_{ij}|X_{ij})$. Let $Y_i:=(Y_{i1},\ldots,Y_{in})$ and $X_i:=(X_{i1},\ldots,X_{in})$ for $i\in[m]$.
The underlying generative model for our setting is given by
\begin{align}\label{eq:learning-generative-model}
p_{\{\vct{\theta}_i,Y_i\}|\{X_i\}}(\vct{\theta}_1,\ldots,\vct{\theta}_m,Y_1,\ldots,Y_m|X_1,\ldots,X_m)=\prod_{i=1}^mp(\vct{\theta}_i)\prod_{i=1}^m\prod_{j=1}^np_{\vct{\theta}_i}(Y_{ij}|X_{ij}).
\end{align}
Note that if we minimize the negative log likelihood of \eqref{eq:learning-generative-model}, we would get the optimal parameters:
\begin{align}\label{eq:learning-log-generative-model}
\htheta_1,\ldots,\htheta_m := \argmin_{\vct{\theta}_1,\ldots,\vct{\theta}_m} \sum_{i=1}^m\sum_{j=1}^n -\log(p_{\vct{\theta}_i}(Y_{ij}|X_{ij})) + \sum_{i=1}^m -\log(p(\vct{\theta}_i)).
\end{align}
Here, $f_i(\vct{\theta}_i):=\sum_{j=1}^n -\log(p_{\vct{\theta}_i}(Y_{ij}|X_{ij}))$ denotes the loss function at the $i$-th client, which only depends on the local data, and $R(\{\vct{\theta}_i\}):=\sum_{i=1}^m -\log(p(\vct{\theta}_i))$ is the regularizer that depends on the (unknown) global population distribution $\bbP$ (parametrized by unknown $\Gamma$). 
Note that when clients have little data and we have large number of clients, i.e., $n\ll m$ -- the setting of federated learning, clients may not be able to learn good personalized models from their local data alone (if they do, it would lead to large loss). In order to learn better personalized models, clients may utilize other clients' data through collaboration, and the above regularizer (and estimates of the unknown prior distribution $\bbP$, through estimating its parameters $\Gamma$) dictates how the collaboration might be utilized. 

We consider some special cases of
\eqref{eq:learning-log-generative-model}.
\begin{enumerate}[leftmargin=*]
\item When $\bbP(\Gamma)\equiv\calN(\vct{\mu},\sigmatheta^2\bbI_d)$ for unknown parameters $\Gamma=\{\vct{\mu},\sigmatheta\}$, then $R(\{\vct{\theta}_i\})=\frac{md}{2}\log(2\pi\sigmatheta^2)+\sum_{i=1}^{m}\frac{\|\vct{\mu}-\vct{\theta}_i\|_2^2}{2\sigmatheta^2}$. Here, unknown $\vct{\mu}$ can be connected to the  global model and $\vct{\theta}_i$'s as local models, and the alternating iterative optimization optimizes over both. 
  This justifies the use of $\ell_2$ regularizer in earlier personalized learning works \cite{dinh2020personalized,ozkara2021quped,hanzely2020federated,hanzely2020lower,li2021ditto}.
\item When $\bbP(\Gamma)\equiv\mathsf{Laplace}(\vct{\mu},b)$, for $\Gamma=\{\vct{\mu},b>0\}$, then $R(\{\vct{\theta}_i\})=m\log(2b)+\sum_{i=1}^m\frac{\|\vct{\theta}_i-\vct{\mu}\|_1}{b}$.
\item When $\bbP(\Gamma)$ is a Gaussian mixture,  $\bbP(\Gamma)\equiv\mathsf{GM}(\{p_l\}_{l=1}^{k},\{\vct{\mu_l}\}_{l=1}^{k},\{\sigma_{\theta,l}^2\}_{l=1}^{k})$, for $\Gamma=\{\{p_l\}_{l=1}^{k}:p_l\geq0 \text{ and } \sum_{l=1}^k p_l=1,\{\vct{\mu}_l\}_{l=1}^{k},\{\sigma_{\theta,l}^2\}_{l=1}^{k}:\sigma_{\theta,l}>0\}$, then $R(\{\vct{\theta}_i\})=\sum_{i=1}^m\log(\sum_{l=1}^k\exp(\frac{\|\vct{\mu}_l-\vct{\theta}_i\|_2^2}{2\sigma_{\theta,l}^2})/((2\pi\sigma_{\theta,l})^{md/2}))$.
\item When $p_{\vct{\theta}_i}(Y_{ij}|X_{ij})$ is according to $\calN(\vct{\theta}_i,\sigma_x^2)$, then $f_i(\vct{\theta}_i)$ is the quadratic loss as in the case of linear regression.
\item When $p_{\vct{\theta}_i}(Y_{ij}|X_{ij})=\sigma(\langle \vct{\theta}_i, X_{ij} \rangle)^{Y_{ij}}(1-\sigma(\langle \vct{\theta}_i, X_{ij} \rangle))^{(1-Y_{ij})}$, where $\sigma(z)=\nicefrac{1}{1+e^{-z}}$ for any $z\in\bbR$, then $f_i(\vct{\theta}_i)$ is the cross-entropy loss (or the logistic loss) as in the case of logistic regression.
\end{enumerate}

\subsection{Linear Regression}\label{sec:lin-reg}

In this section, we present the personalized linear regression problem.  Consider a set of $m$ clients, where the $i$-th client has a local dataset consisting of $n$ samples $(X_{i1},Y_{i1}),\ldots,(X_{in},Y_{in})$, where $X_{ij}\in\mathbb{R}^{d}$ denotes the feature vector and $Y_{ij}\in\mathbb{R}$ denotes the corresponding response. Let $Y_i=(Y_{i1},\ldots,Y_{i1})\in\mathbb{R}^{n}$ and $X_i=(X_{i1},\ldots,X_{in})\in\mathbb{R}^{n\times d}$ denote the response vector and the feature matrix at the $i$-th client, respectively. Following the standard regression, we assume that the response vector $Y_i$ is obtained from a linear model as follows:
\begin{equation}
	Y_i = X_i \vct{\theta}_i+w_i,    
\end{equation}
where $\vct{\theta}_i$ denotes personalized model of the $i$-th client and $w_i\sim\mathcal{N}\left(0,\sigma_x^2\mathbb{I}_{n}\right)$ is a noise vector. The clients' parameters $\theta_1,\ldots,\theta_{m}$ are drawn i.i.d. from a Gaussian distribution $\theta_1,\ldots,\theta_{m}\sim \mathcal{N}(\mu,\sigma_{\theta}^{2}\mathbb{I}_{d})$, i.i.d. 

\subsubsection{When the prior distribution is known}
Our goal is to solve the optimization problem stated in \eqref{eq:learning-log-generative-model} (for the linear regression setup) and learn the optimal personalized parameters $\{\htheta_i\}$. 
\begin{align}\label{eqn:linear-regression-perfect}
	\argmin_{\{\vct{\theta}_i\}}\hspace{-0.15cm}\frac{nm}{2}\log(2\pi\sigma_{x}^2)+\hspace{-0.1cm}\sum_{i=1}^{m}\sum_{j=1}^{n}\frac{(Y_{ij}-\langle\vct{\theta}_{i},X_{ij}\rangle)^2}{2\sigma_{x}^2} +\frac{md}{2}\log(2\pi\sigma_{\vct{\theta}}^2)+\sum_{i=1}^{m}\frac{\|\vct{\mu}-\vct{\theta}_i\|_2^2}{2\sigma_{\vct{\theta}}^2}.
\end{align}

The following theorem characterizes the exact form of the optimal $\{\htheta_i\}$ and computes their minimum mean squared error w.r.t.\ the true parameters $\{\vct{\theta}_i\}$.

\begin{theorem}\label{thm:perfect_mse_lin}
	The optimal personalized parameters at client $i$ with known $\vct{\mu},\sigma_{\theta}^2,\sigma_x^2$ is given by:
	\begin{equation}~\label{eqn:personalized_estimate_lin}
		\htheta_i = \left(\frac{\mathbb{I}}{\sigma_{\theta}^2}+\frac{X_i^{T}X_i}{\sigma_x^2}\right)^{-1}\left(\frac{X_i^{T}Y_i}{\sigma_x^2}+\frac{\vct{\mu}}{\sigma_\theta^2}\right).
	\end{equation}
	The mean squared error (MSE) of the above $\htheta_i$ is given by:
	\begin{equation}
		\mathbb{E}_{\bw_i,\vct{\theta}_i}\left\|\htheta_i-\vct{\theta}_i\right\|^{2}=\mathsf{Tr}\left( \left(\frac{\mathbb{I}}{\sigma_{\theta}^2}+\frac{X_i^{T}X_i}{\sigma_x^2}\right)^{-1}\right),
	\end{equation} 
\end{theorem}
Observe that the local model of the $i$-th client, i.e., estimating $\vct{\theta}_i$ only from the local data $(Y_i,X_i)$, is given by:
\begin{equation}~\label{eqn:local_estimate_lin}
	\htheta_i^{(l)} = \left(X_i^{T}X_i\right)^{-1}X_i^{T}Y_i,
\end{equation} 
where we assume the matrix $X_i^{T}X_i$ has a full rank (otherwise, we take the pseudo inverse). This local estimate achieves the MSE given by:
\begin{equation}
	\mathbb{E}\left\|\htheta_i^{(l)}-\vct{\theta}_i\right\|^{2} =\mathsf{Tr}\left(\left(X_i^{T}X_i\right)^{-1}\right)\sigma_x^{2},
\end{equation}
we can prove it by following similar steps as the proof of Theorem~\ref{thm:perfect_mse_lin}. When $\sigma_{\theta}^2\to\infty$, we can easily see that the local estimate~\eqref{eqn:local_estimate_lin} matches the personalized estimate in~\eqref{eqn:personalized_estimate_lin}. 
\subsubsection{When the prior distribution is unknown}
Now we write down the full objective function for linear regression by taking $\bbP(\Gamma)\equiv\calN(\vct{\mu},\sigmatheta^2\bbI_d)$ and $p_{\vct{\theta}_i}(Y_{ij}|X_{ij})$ according to $\calN(\vct{\theta}_i,\sigma_x^2)$; the case of logistic regression can be similarly handled. Here, if we assume that in addition to $\vct{\mu}$, the other parameters $\sigmatheta,\sigma_x$ are also unknown, then we can also optimize over them. In this case, the overall optimization problem becomes:
\begin{align}\label{eqn:linear-regression}
\argmin_{\{\vct{\theta}_i\},\vct{\mu}, \sigmatheta, \sigma_x}\hspace{-0.15cm}\frac{nm}{2}\log(2\pi\sigma_{x}^2)+\hspace{-0.1cm}\sum_{i=1}^{m}\sum_{j=1}^{n}\frac{(Y_{ij}-\langle\vct{\theta}_{i},X_{ij}\rangle)^2}{2\sigma_{x}^2} +\frac{md}{2}\log(2\pi\sigma_{\vct{\theta}}^2)+\sum_{i=1}^{m}\frac{\|\vct{\mu}-\vct{\theta}_i\|_2^2}{2\sigma_{\vct{\theta}}^2}.
\end{align}
We can optimize the above expression through gradient descent (GD) and the resulting algorithm is Algorithm~\ref{algo:personalized_regression}.
In addition to keeping the personalized models $\{\vct{\theta}_i^t\}$, each client also maintains local copies of $\{\vct{\mu}_i^t,\sigma_{\theta,i}^t,\sigma_{x,i}^t\}$ and updates all these parameters by taking appropriate gradients of the objective in \eqref{eqn:linear-regression} and synchronize them with the server to update the global copy of these parameters $\{\vct{\mu}^t,\sigmatheta^t,\sigma_{x}^t\}$. 

\begin{algorithm}[h]
	\caption{Linear Regression GD}
	{\bf Input:} Number of iterations $T$, local datasets $(Y_i, X_i)$ for $i\in[m]$, learning rate $\eta$.\\
	\vspace{-0.3cm}
	\begin{algorithmic}[1] 	\label{algo:personalized_regression}
		\STATE \textbf{Initialize} $\vct{\theta}_i^{0}$  for $i\in[m]$, $\vct{\mu}^0$, $\sigma_x^{2,0}$, $\sigma_{\theta}^{2,0}$.
		\FOR{$t=1$ \textbf{to} $T$} 
		\STATE \textbf{On Clients:}
		\FOR {$i=1$ \textbf{to} $m$:}
		\STATE Receive and set $\vct{\mu}_{i}^{t-1} = \vct{\mu}^{t-1},\sigma^{2,t-1}_{\theta,i}=\sigma^{2,t-1}_{\theta},\sigma^{2,t-1}_{x,i}=\sigma^{2,t-1}_{x}$
		\STATE Update the personalized model: $\vct{\theta}_i^{t}\gets\vct{\theta}_{i}^{t-1}+\eta\left(\sum_{j=1}^{n}\frac{X_{ij}(Y_{ij}-X_{ij}\vct{\theta}_{i}^{t-1})}{\sigma_{x,i}^{2,t-1}}+\frac{\vct{\mu}_i^{t-1}-\vct{\theta}_i^{t-1}}{\sigma_{\theta,i}^{2,t-1}}\right)$
		\STATE Update local version of mean: $\vct{\mu}_i^{t}\gets\vct{\mu}_i^{t-1}-\eta\left(\frac{\vct{\mu}_i^{t-1}-\vct{\theta}_i^{t-1}}{\sigma_{\theta,i}^{2,t-1}}\right)$
		\STATE Update local variance: $\sigma^{2,t}_{x,i}\gets\sigma^{2,t-1}_{x,i}-\eta\left(\frac{n}{2\sigma_{x,i}^{2,t-1}}-\sum_{j=1}^{n}\frac{(Y_{ij}-X_{ij}\vct{\theta}_{i}^{t-1})^2}{2(\sigma_{x,i}^{2,t-1})^2}\right)$
		\STATE Update global variance: $\sigma^{2,t}_{\theta,i}\gets\sigma^{2,t-1}_{\theta,_i}-\eta\left(\frac{d}{2\sigma_{\theta,i}^{2,t-1}}-\frac{\|\vct{\mu}_i^{t-1}-\vct{\theta}^{t-1}_i\|^2}{2(\sigma_{\theta,i}^{2,t-1})^2}\right)$
		\ENDFOR
		\STATE \textbf{At the Server:}
		\STATE Aggregate mean: $\vct{\mu}^{t}=\frac{1}{m}\sum_{i=1}^{m}\vct{\mu}_i^t$
		\STATE Aggregate global variance: $\sigma^{2,t}_{\theta}= \frac{1}{m}\sum_{i=1}^{m}\sigma^{2,t}_{\theta,i}$
		\STATE Aggregate local variance: $\sigma^{2,t}_{x}= \frac{1}{m}\sum_{i=1}^{m}\sigma^{2,t}_{x,i}$
		\STATE Broadcast $\vct{\mu}^{t},\sigma^{2,t}_{\theta},\sigma^{2,t}_{x}$ 
		\ENDFOR
	\end{algorithmic}
	{\bf Output:} Personalized models $\vct{\theta}_1^{T},\ldots,\vct{\theta}_m^{T}$.
\end{algorithm}

\subsection{Logistic Regression}\label{sec:log-reg}
As described in Section~\ref{sec:learning}, by taking $\bbP(\Gamma)\equiv\calN(\vct{\mu},\sigmatheta^2\bbI_d)$ and $p_{\theta_i}(Y_{ij}|X_{ij})=\sigma(\langle \vct{\theta}_i, X_{ij} \rangle)^{Y_{ij}}(1-\sigma(\langle \vct{\theta}_i, X_{ij} \rangle))^{(1-Y_{ij})}$, where $\sigma(z)=\nicefrac{1}{1+e^{-z}}$ for any $z\in\bbR$, then the overall optimization problem becomes:

\begin{align}\label{eqn:logistic-regression}
	&\argmin_{\{\vct{\theta}_i\},\vct{\mu}, \sigmatheta}\sum_{i=1}^{m}\sum_{j=1}^{n}\left[Y_{ij}\log\left(\frac{1}{1+e^{-\langle\vct{\theta}_i,X_{ij}\rangle}}\right) + (1-Y_{ij})\log\left(\frac{1}{1+e^{\langle\vct{\theta}_i,X_{ij}\rangle}}\right)\right] \notag \\
	&\hspace{5cm} +\frac{md}{2}\log(2\pi\sigma_{\theta}^2)+\sum_{i=1}^{m}\frac{\|\vct{\mu}-\vct{\theta}_i\|_2^2}{2\sigma_{\theta}^2}.
\end{align}

When $\vct{\mu}$ and $\sigma_{\theta}^2$ are unknown, we would like to learn them by gradient descent, as in the linear regression case. The corresponding algorithm is described in Algorithm~\ref{algo:personalized_logsitic_regression}.

\begin{algorithm}[H]
	\caption{Logistic Regression GD}
	{\bf Input:} Number of iterations $T$, local datasets $(Y_i, X_i)$ for $i\in[m]$, learning rate $\eta$.\\
	\vspace{-0.3cm}
	\begin{algorithmic}[1] 	\label{algo:personalized_logsitic_regression}
		\STATE \textbf{Initialize} $\vct{\theta}_i^{0}$  for $i\in[m]$, $\vct{\mu}^0$, $\sigma_{\theta}^{2,0}$.
		\FOR{$t=1$ \textbf{to} $T$} 
		\STATE \textbf{On Clients:}
		\FOR {$i=1$ \textbf{to} $m$:}
		\STATE Receive $(\vct{\mu}^{t},\sigma^{2,t}_{\theta})$ from the server and set $\vct{\mu}_{i}^t := \vct{\mu}^{t},\sigma^{2,t}_{\theta,i}:=\sigma^{2,t}_{\theta}$
		\STATE Update the personalized model: \[\vct{\theta}_i^{t}\gets\vct{\theta}_{i}^{t-1}-\eta\left(\sum_{j=1}^{n}\nabla_{\theta_i^{t-1}}l_{CE}^{(p)}(\vct{\theta}_i^{t-1}, (X_i^j,Y_i^j))+\frac{\vct{\mu}_i^{t-1}-\vct{\theta}_i^{t-1}}{\sigma_{\theta,i}^{2,t-1}}\right),\]
		where $l_{CE}^{(p)}$ denotes the cross-entropy loss.
		\STATE Update local version of mean: $\vct{\mu}_i^{t}\gets\vct{\mu}_i^{t-1}-\eta\left(\frac{\vct{\mu}_i^{t-1}-\vct{\theta}_i^{t-1}}{\sigma_{\theta,_i}^{2,t-1}}\right)$
		\STATE Update global variance: $\sigma^{2,t}_{\theta,_i}\gets\sigma^{2,t-1}_{\theta,_i}-\eta\left(\frac{d}{2\sigma_{\theta,i}^{2,t-1}}-\frac{\|\vct{\mu}_i^{t-1}-\vct{\theta}^{t-1}_i\|^2}{2(\sigma_{\theta,i}^{2,t-1})^2}\right)$
		\STATE Send $(\vct{\mu}_i^t,\sigma_{\theta,i}^{2,t})$ to the server
		\ENDFOR
		\STATE \textbf{At the Server:}
		\STATE Receive $\{(\vct{\mu}_i^t,\sigma_{\theta,i}^{2,t})\}$ from the clients
		\STATE Aggregate mean: $\vct{\mu}^{t}=\frac{1}{m}\sum_{i=1}^{m}\vct{\mu}_i^t$
		\STATE Aggregate global variance: $\sigma^{2,t}_{\theta}= \frac{1}{m}\sum_{i=1}^{m}\sigma^{2,t}_{\theta,i}$
		\STATE Broadcast $(\vct{\mu}^{t},\sigma^{2,t}_{\theta})$ to all clients
		\ENDFOR
	\end{algorithmic}
	{\bf Output:} Personalized models $\vct{\theta}_1^{T},\ldots,\vct{\theta}_m^{T}$.
\end{algorithm}

\subsection{Gaussian Mixture Prior}\label{sec:gauss-mixture}

When we take $\bbP(\Gamma)\equiv\mathsf{GM}(\{p_l\}_{l=1}^{k},\{\vct{\mu_l}\}_{l=1}^{k},\{\sigma_{\theta,l}^2\}_{l=1}^{k})$ and $p_{\theta_i}(Y_{ij}|X_{ij})=\sigma(\langle \vct{\theta}_i, X_{ij} \rangle)^{Y_{ij}}(1-\sigma(\langle \vct{\theta}_i, X_{ij} \rangle))^{(1-Y_{ij})}$, where $\sigma(z)=\nicefrac{1}{1+e^{-z}}$ for any $z\in\bbR$, then the overall optimization problem becomes:

\begin{align}\label{eqn:gmm-log-regression}
	&\argmin_{\{\vct{\theta}_i\},\{\vct{\mu}_l\},\{p_l\}, \{\sigma_{\theta,l}\}}\sum_{i=1}^{m}\sum_{j=1}^{n}\left[Y_{ij}\log\left(\frac{1}{1+e^{-\langle\vct{\theta}_i,X_{ij}\rangle}}\right) + (1-Y_{ij})\log\left(\frac{1}{1+e^{\langle\vct{\theta}_i,X_{ij}\rangle}}\right)\right] \notag \\
	&\hspace{5cm} +\sum_{i=1}^m\log(\sum_{l=1}^k\exp(\frac{\|\vct{\mu}_l-\vct{\theta}_i\|_2^2}{2\sigma_{\theta,l}^2})/((2\pi\sigma_{\theta,l})^{md/2}))
\end{align}

This can easily be extended to a generic neural network loss function with multi-class softmax output layer and cross entropy loss, defining such a local loss function in client $i$ as $f_i(\vct{\theta}_i)$ (omitted dependence on data samples) we obtain:

\begin{align}\label{eqn:gmm-nn-loss}
	\argmin_{\{\vct{\theta}_i\},\{\vct{\mu}_l\},\{p_l\}, \{\sigma_{\theta,l}\}}\sum_{i=1}^{m} f_i(\vct{\theta}_i) 
	 +\sum_{i=1}^m\log(\sum_{l=1}^k\exp(\frac{\|\vct{\mu}_l-\vct{\theta}_i\|_2^2}{2\sigma_{\theta,l}^2})/((2\pi\sigma_{\theta,l})^{md/2}))
\end{align}

To solve this problem we can either use an alternating gradient descent approach as in Algorithm~\ref{algo:personalized_regression},\ref{algo:personalized_logsitic_regression} or we can use a clustering based approach where the server runs a clustering algorithm on received personalized models. Here we describe the second one as it provides an interesting point of view and can be combined with DP clustering algorithms. As a result, we propose Algorithm~\ref{algo:cluster_gmm}.

\begin{algorithm}[h]
	\caption{Personalized Learning with Gaussian Mixture Prior}
	{\bf Input:} Number of iterations $T$, local datasets $(X_i,Y_i)$ for $i\in[m]$, learning rate $\eta$.\\
	\vspace{-0.3cm}
	\begin{algorithmic}[1] 	\label{algo:cluster_gmm}
		\STATE \textbf{Initialize} $\vct{\theta}_i^{0}$  for $i\in[m]$ and $\bbP^{(0)},\vct{\mu}_1^{(0)},\ldots,\vct{\mu}_k^{(0)}$.
		\FOR{$t=1$ \textbf{to} $T$}
		\STATE {\bf On Clients:}
		\FOR {$i=1$ \textbf{to} $m$:}
		\STATE Receive $\bbP^{(t-1)},\vct{\mu}_1^{(t-1)},\ldots,\vct{\mu}_k^{(t-1)},\sigma_{\theta,1}^{(t-1)},\ldots,\sigma_{\theta,k}^{(t-1)}$ from the server 
		\STATE\label{step:personal_update} Update the personalized parameters:
		\begin{align*}
			\vct{\theta}_i^{t} \gets \vct{\theta}_i^{t-1} - \eta\nabla_{\vct{\theta}_i^{t-1}}\Big[ f_i(\vct{\theta}_i^{t-1}) 
			+\log(\sum_{l=1}^k\exp(\frac{\|\vct{\mu}_l^{(t-1)}-\vct{\theta}_i^{t-1}\|_2^2}{2(\sigma_{\theta,l}^{(t)})^2})/((2\pi\sigma_{\theta,l}^{(t)})^{md/2})) \Big]
		\end{align*}
		\STATE Send $\vct{\theta}_i^{(t)}$ to the server
		\ENDFOR
		\STATE {\bf At the Server:}
		\STATE Receive $\vct{\theta}_1^{(t)},\ldots,\vct{\theta}_m^{(t)}$ from the clients
		\STATE\label{step:clustering} Update the global parameters:
		$\bbP^{(t)},\vct{\mu}_1^{(t)},\ldots,\vct{\mu}_k^{(t)},\sigma_{\theta,1}^{(t)},\ldots,\sigma_{\theta,k}^{(t)}\gets \mathsf{Cluster}\left(\vct{\theta}_1^{(t)},\ldots,\vct{\theta}_m^{(t)},k\right),$
		\STATE Broadcast $\bbP^{(t)},\vct{\mu}_1^{(t)},\ldots,\vct{\mu}_k^{(t)}$ to all clients
		\ENDFOR
	\end{algorithmic}
	{\bf Output:} Personalized models $\vct{\theta}_1^{T},\ldots,\vct{\theta}_m^{T}$.
\end{algorithm}

\paragraph{Description of Algorithm~\ref{algo:cluster_gmm}.} Here clients receive the global parameters from the server and do a local iteration on the personalized model (multiple local iterations can be introduced as in FedAvg), later the clients broadcast the personalized models. Receiving the personalized models, server initiates a clustering algorithm that outputs global parameters. In general server is assumed to have vast amount of computational resources which makes running clustering algorithm feasible.

\paragraph{Adding DP.} DP clustering algorithms such as \cite{ghazi2020differentially} can be utilized to create a DP version of Algorithm~\ref{algo:cluster_gmm}.

A discrete mixture model as in Section~\ref{sec:est_mixture} can be proposed as a special case of GM with 0 variance. With this we can recover a similar algorithm as in \cite{marfoq2021federated}. Further details are presented in Appendix~\ref{app:disc-mixture}.

\subsection{\adaped: Adaptive Personalization via Distillation}\label{sec:adaped}
\begin{algorithm}[h]
	\caption{Adaptive Personalization via Distillation (\adaped)}
	{\bf Parameters:} local variances $\{\psi_i^{0}\}$, personalized models $\{\vct{\theta}_i^{0}\}$, local copies of the global model $\{\vct{\mu}_i^0\}$, synchronization gap $\tau$, learning rates $\eta_1,\eta_2,\eta_3$, number of sampled clients $K$. 
	\begin{algorithmic}[1] \label{algo:personalized}
		\FOR{$t=0$ \textbf{to} $T-1$}
		\IF{$\tau$ divides $t$}
		\STATE \textbf{On Server do:}\\
		\STATE Choose a subset $\mathcal{K}^t \subseteq [n]$ of $K$ clients \\
		\STATE Broadcast $\vct{\mu}^t$ and $\psi^{t}$
		\STATE \textbf{On Clients} $i\in\mathcal{K}^t$ (in parallel) \textbf{do}:\\
		\STATE Receive $\vct{\mu}^t$ and $\psi^{t}$; set $\vct{\mu}_i^t = \vct{\mu}^t$, $\psi_i^{t} = \psi^{t}$
		\ENDIF	
		\STATE \textbf{On Clients} $i\in\mathcal{K}^t$ (in parallel) \textbf{do}:
		
		\STATE Compute $\bg_{i}^{t} := \nabla_{\vct{\theta}_{i}^{t}} f_i(\vct{\theta}_{i}^{t}) +  \frac{\nabla_{\vct{\theta}_{i}^{t}}f^{\KD}_i(\vct{\theta}_{i}^{t}, \vct{\mu}_{i}^{t})}{2\psi_i^{t}}  $
		\STATE Update: $\vct{\theta}_{i}^{t+1}=\vct{\theta}_{i}^{t} - \eta_1 \bg_{i}^{t}$\\
		\STATE  Compute $\bh_{i}^{t} := \frac{\nabla_{\vct{\mu}_{i}^{t}}f^{\KD}_i(\vct{\theta}_{i}^{t+1}, \vct{\mu}_{i}^{t})}{2\psi_i^{t}}$ \\
		\STATE Update: $\vct{\mu}_{i}^{t+1} = \vct{\mu}_{i}^{t}-\eta_2\bh_{i}^{t}  $\\
		\STATE  Compute $k_{i}^{t} := \frac{1}{2\psi_i^{t}}-\frac{f^{\KD}_i(\vct{\theta}_{i}^{t+1}, \vct{\mu}_{i}^{t+1})}{2(\psi_i^{t})^2}$ \\
		\STATE Update: $\psi_i^{t+1}=\psi_i^{t}-\eta_3k_i^t$ 
		\IF{$\tau$ divides $t+1$}
		\STATE Clients send $\vct{\mu}_{i}^{t}$ and $\psi_i^{t}$ to \textbf{Server}
		\STATE Server receives $\{\vct{\mu}_{i}^{t}\}_{i\in\mathcal{K}^t}$ and $\{\psi_i^{t}\}_{i\in\mathcal{K}^t}$ \\
		\STATE Server computes $\vct{\mu}^{t+1} = \frac{1}{K} \sum_{i\in\mathcal{K}^t} \vct{\mu}_i^{t}$ and $\psi^{t+1} = \frac{1}{K} \sum_{i\in\mathcal{K}^t} \psi_i^{t}$\\
		\ENDIF
		\ENDFOR
	\end{algorithmic}
	{\bf Output:} Personalized models $(\vct{\theta}_i^{T})_{i=1}^m$
\end{algorithm}

It has been empirically observed that the knowledge distillation (KD) regularizer (between local and global models) results in better performance \cite{ozkara2021quped} than the $\ell_2$ regularizer. 
In fact, using our generative framework, we can define a certain prior distribution that gives the KD regularizer (details are provided in Section~\ref{app:adaped}). The specific form of the regularizer that we use in the loss function at the $i$-th client is the following:

\begin{align}
f_i(\vct{\theta}_i)+\frac{1}{2}\log(2\psi)+\frac{f^{\KD}_i(\vct{\theta}_i,\vct{\mu})}{2\psi} \label{loc_loss}
\end{align}

where $\vct{\mu}$ denotes the global model, $\vct{\theta}_i$ denotes the personalized model at client $i$, 
and $\psi$ can be viewed as controlling heterogeneity. Note that the goal for each client is to minimize its local loss function, so individual components cannot be too large. For the second term, this implies that $\psi$ cannot be unbounded. For the third term, if $f^{\KD}_i(\vct{\theta}_i,\vct{\mu})$ is large, then $\psi$ will also increase (implying that the local parameters are too deviated from the global parameter), hence, it is better to emphasize local training loss to make the first term small. If $f^{\KD}_i(\vct{\theta}_i,\vct{\mu})$ is small, then $\psi$ will also decrease (implying that the local parameters are close to the global parameter), so it is better for clients to collaborate and learn better personalized models.
To optimize \eqref{loc_loss} we propose an alternating minimization approach, which we call \adaped\ and is presented in Algorithm~\ref{algo:personalized}.

\paragraph{Description of \adaped.} Besides the personalized model $\vct{\theta}_i^t$, each client $i$ keeps local copies of the global model $\vct{\mu}_i^t$ and of the variance $\psi_i^{t}$, and at synchronization times, server aggregates them to obtain global versions of these $\vct{\mu}^t,\psi^t$. At each iteration $t$ divisible by synchronization gap $\tau$, the server samples a subset $\mathcal{K}_t \subseteq [m]$ of $K\leq m$ clients and broadcasts $\vct{\mu}^t,\psi^t$ to the selected clients, who set their set $\vct{\mu}_i^t=\vct{\mu}^t$ and $\psi_i^t=\psi^t$. When $t+1$ is not divisible by $\tau$, if $i\in\mathcal{K}_t$, client $i$ updates $\vct{\theta}_i^t,\vct{\mu}_i^t,\psi_i^t$ by taking gradients of the local loss function w.r.t.\ the respective parameters (lines $11,13,15$).
Thus, the local training of $\vct{\theta}_i^t$ also incorporates knowledge from other clients' data through $\vct{\mu}_i^{t}$.
When $t+1$ is divisible by $\tau$, sampled clients upload $\{\vct{\mu}_i^{t},\psi_{i}^{t}:i\in\mathcal{K}_t\}$ to the server which aggregates them (lines $18,19$).
At the end of training, clients have learned their personalized models $\{\vct{\theta}_i^{T}\}_{i=1}^{m}$. 

\subsection{\dpadaped: Differentially Private Adaptive Personalization via Distillation}\label{sec:dp-adaped}
Note that client $i$ communicates $\vct{\mu}_i^t,\psi_i^t$ (which are updated by accessing the dataset for computing the gradients $\bh_i^t,k_i^t$) to the server. So, to privatize $\vct{\mu}_i^t,\psi_i^t$, client $i$ adds appropriate noise to $\bh_i^k,k_i^t$, respectively. In order to obtain \dpadaped, we replace lines 13 and 15 by the update rules:
%
\begin{align*}
\text{Line }13:\qquad & \vct{\mu}_{i}^{t+1} = \vct{\mu}_{i}^{t}-\eta_2 \Big(\frac{\bh_i^t}{\max\lbrace \|\bh_i^t\|/C_1,1\rbrace} +\bnu_1\Big), \\
\text{Line }15:\qquad & \psi_i^{t+1}=\psi_i^{t}-\eta_3 \Big(\frac{k_i^t}{\max\lbrace |k_i^t|/C_2,1\rbrace} +\nu_2\Big),
\end{align*}
where $\bnu_1\sim\calN(0,\sigma_{q_1}^2\bbI_d)$ and $\nu_2\sim\calN(0,\sigma_{q_2}^2)$, for some $\sigma_{q_1},\sigma_{q_2}>0$ that depend on the desired privacy level and $C_1,C_2$, which are some predefined constants.

In the following theorem, we state the R\'enyi Differential Privacy (RDP) guarantees of \dpadaped. 

\begin{theorem}\label{thm:privacy_personal}
After $T$ iterations, \emph{\dpadaped}\ satisfies $(\alpha,\epsilon(\alpha))$-RDP, where $\epsilon(\alpha)=\left(\frac{K}{m}\right)^26\left(\frac{T}{\tau}\right)\alpha \left(\frac{C_1^2}{K\sigma_{q_1}^2}+\frac{C_2^2}{K\sigma_{q_2}^2}\right)$ for $\alpha>1$, where $\frac{K}{m}$ denotes the sampling ratio of the clients at each global iteration. 
\end{theorem}
A proof of Theorem~\ref{thm:privacy_personal} is presented in Section~\ref{app:dp-adaped}. 
Observe that here we bound the RDP, as it gives better privacy composition than using the strong composition~\cite{GirgisDDSK_RDP_Shuffle_CCS21}. We can also convert our results to user-level $(\epsilon,\delta)$-DP by using the standard conversion from RDP to approximate $(\epsilon,\delta)$-DP \cite{canonne2020discrete}.

%% file: unification.tex
\subsection{Connecting to Existing Methods}\label{sec:unification}
In Section \ref{sec:intro}, we alluded to connections with other
personalized FL methods. Here we provide a more detailed discussion.

\paragraph{Regularization:} As noted earlier using
\eqref{eq:learning-log-generative-model} with the Gaussian population
prior connects to the use of $\ell_2$ regularizer in earlier
personalized learning works
\cite{dinh2020personalized,ozkara2021quped,hanzely2020federated,hanzely2020lower,li2021ditto},
which also iterates between local and global model estimates. This form can be explicitly seen in Section~\ref{sec:lin-reg}, where in Algorithm \ref{algo:personalized_regression}, we see that the Gaussian prior along with iterative optimization yields the regularized form seen in these methods. In these cases\footnote{One can generalize these by including $\sigmatheta^2$ in the unknown parameters.}, $\bbP(\Gamma)\equiv\calN(\vct{\mu},\sigmatheta^2\bbI_d)$ for unknown parameters $\Gamma=\{\vct{\mu}\}$. Note that since the parameters of the population distribution are unknown, these need to be estimated during the iterative learning process. In the algorithm, \ref{algo:personalized_regression} it is seen the $\vct{\mu}$ plays the role of the global model (and is truly so for the linear regression problem studied in  Section~\ref{sec:lin-reg}). 

\paragraph{Clustered FL:} If one uses a \emph{discrete} mixture
model for the population distribution then the iterative algorithm
suggested by our framework connects to
\cite{zhang2021personalized,mansour2020approaches,ghosh2020efficient,smith2017federated,marfoq2021federated}. In
particular, consider a population model with parameters in the $m$-dimensional probability simplex
$\{\vct{\alpha}:\vct{\alpha}=[\alpha_1,\ldots,\alpha_k], \alpha_i\geq
0, \forall i, \sum_i\alpha_i=1\}$ which describing a distribution. If
there are $m$ (unknown) discrete distributions
$\{\mathcal{D}_1,\ldots,\mathcal{D}_m\}$, one can consider these as
the unknown description of the population model in addition to
$\vct{\alpha}$. Therefore, each local data are generated either as a
mixture as in \cite{marfoq2021federated} or by choosing one of the
unknown discrete distributions with probability $\vct{\alpha}$
dictating the probability of choosing $\mathcal{D}_i$, when hard
clustering is used (\emph{e.g.,} \cite{mansour2020approaches}). Each
node $j$ chooses a mixture probability $\vct{\alpha}^{(j)}$ uniformly
from the $m$-dimensional probability simplex. In the former case, it
uses this mixture probability to generate a local mixture
distribution. In the latter it chooses $\mathcal{D}_i$ with
probability $\vct{\alpha}^{(j)}_i$.

As mentioned earlier, not all parametrized distributions can be
written as a mixture of \emph{finite} number distributions, which is
the assumption for discrete mixtures. Consider a unimodal Gaussian
population distribution (as also studied in Section
\ref{sec:lin-reg}). Since
$\bbP(\Gamma)\equiv\calN(\vct{\mu},\sigmatheta^2\bbI_d)$, for node
$i$, we sample $\vct{\theta}_i\sim \bbP(\Gamma)$. We see that the
actual data distribution for this node is
$p_{\vct{\theta}_i}(y|\vct{x})=\calN(\vct{\theta_i}^{\top}\vct{x},\sigma^2)$. Clearly
the set of such distributions $\{p_{\vct{\theta}_i}(y|\vct{x})\}$
\emph{cannot} be written as any finite mixture as
$\vct{\theta_i}\in\mathbb{R}^d$ and $p_{\vct{\theta}_i}(y|\vct{x})$ is
a unimodal Gaussian distribution, with same parameter $\vct{\theta_i}$
for all data generated in node $i$. Essentially the generative
framework of finite mixtures (as in \cite{marfoq2021federated}) could
be restrictive as it does not capture such parametric models.

\paragraph{Knowledge distillation:} The population distribution
related to a regularizer based on Kullback-Leibler divergence
(knowledge distillation) has been shown in Section
\ref{app:adaped}. Therefore this can be cast in terms of information
geometry where the probability falls of exponentially with in this
geometry. Hence these connect to methods such as
\cite{lin2020ensemble,li2019fedmd,shen2020federated,ozkara2021quped},
but the exact regularizer used does not take into account the full
parametrization, and one can therefore improve upon these methods.

\paragraph{FL with Multi-task Learning (MTL):} In this framework,
a \emph{fixed} relationship between tasks is usually assumed
\cite{smith2017federated}. Therefore one can model this as a Gaussian
model with \emph{known} parameters relating the individual models. The
individual models are chosen from a joint Gaussian with particular
(known) covariance dictating the different models, and therefore
giving the quadratic regularization used in FL-MTL
\cite{smith2017federated}. In this the parameters of the Gaussian
model are known and fixed.

\paragraph{Common representations:} The works in
\cite{du2021fewshot,prateek2021differentially} use a linear model
where $y\sim\mathcal{N}(\vct{x}^{\top}\vct{\theta}_i,\sigma^2)$ can be
considered a local generative model for node $i$. The common
representation approach assumes that $\vct{\theta}_i=\sum_{j=1}^k
\mat{B}\vct{w}_j^{(i)}$, for some $k\ll d$, where
$\vct{\theta}_i\in\mathbb{R}^d$. Therefore, one can parametrize a
population by this (unknown) common basis 
$\mat{B}$, and under
a mild assumption that the weights are bounded, we can choose a
uniform measure in this bounded cube to choose $\vct{w}^{(i)}$ for
each node $i$. The alternating optimization iteratively discovers the
global common representation and the local weights as done in
\cite{du2021fewshot,prateek2021differentially} (and references
therein).

%% file: experiments.tex
\textbf{Personalized estimation.}
We run two types of experiments, one with synthetic data and one with real world data. Here we report {\sf (i)} Bernoulli setting experiment with synthetic data, and {\sf (ii)} Bernoulli setting experiment with US county level election data.

$\bullet$ \textbf{Synthetic experiments in Bernoulli setting.} For this setting, for $\bbP$  we consider three distributions that \cite{tian2017learning} considered: normal, uniform and `3-spike' which have equal weight at 1/4, 1/2, 3/4. Additionally, we consider a $\Beta$ prior. We compute squared error of personalized estimators and local estimators ($\frac{Z_i}{n}$) w.r.t.\ the true $p_i$ and report the average over all clients. We use $m=10000$ clients and $14$ local samples similar to \cite{tian2017learning}. Personalized estimator provides a decrease in MSE by $37.1\pm3.9\%, 12.0\pm1.6\%, 24.3\pm2.8\%, 34.0\pm3.7\%$, respectively, for each aforementioned population distribution compared to their corresponding local estimators. Furthermore, as theoretically noted, less spread out prior distributions (low data heterogeneity) results in higher MSE gap between personalized and local estimators.

$\bullet$ \textbf{Political tendencies on county level.} One natural application of Bernoulli setting is modelling bipartisan elections \cite{tian2017learning}. In this experiment we did a case study by using US presidential elections on county level between 2000-2020 (there are 3112 counties in our dataset, $m=3112$) . For each county the goal is to determine the political tendency parameter $p_i$.  Given 6 election data we did 1-fold cross validation. Local estimator is taking the average of 5 training samples and personalized estimator is the aforementioned posterior mean. We use 5 elections as training data, to simulate a Bernoulli setting we set the data equal to 1 if Republican party won the election and 0 otherwise; and 1 election as test data to measure the MSE of the estimator. As a result, we observe that personalized estimator provides an MSE (averaged over 6 runs) gain of $10.7\pm1.9\%$  compared to local estimator .    

\textbf{Linear regression.}
For this, we create a setting similar to \cite{jain2021differentially}. We set $m=10,000$, $n=10$; and sample client true models according to a Gaussian centered at some randomly chosen $\mu$ with variance $\sigma_\theta^2$. We randomly generate design matrices $X_i$ and create $Y_i$ at each client by adding a zero mean Gaussian noise with true variance $\sigma_x^2$ to $X_i\theta_i$. We measure the average MSE over all clients with and compare personalized and local methods. When $d=50$, personalized regression has an MSE gain of $8.0\pm0.8\%,14.8\pm1.2\%$, and when $d=100$, $9.2\pm1.1\%, 12.3\pm2.0\%$ compared to local and FedAvg regression, respectively. Moreover, compared to personalized regression where $\mu,\sigma_{\theta},\sigma_x$ are known, alternating algorithm only results in $1\%$ and $4.7\%$ increase in MSE respectively for $d=50$ and $d=100$. The results suggest that personalized regression with unknown generative parameters is able to outperform local and federated approaches.


\textbf{AdaPeD.} First we describe the experiment setting and then the results, for \dpadaped\ as well.

$\bullet$ \textbf{Experiment setting.} We consider image classification tasks on MNIST and FEMNIST~\cite{caldas2018leaf}, and train a CNN, similar to the one considered in \cite{mcmahan2017communicationefficient}, that has 2 convolutional and 3 fully connected layers. We set $m=66$ for FEMNIST and $m=50$ for MNIST. Similar to \cite{ozkara2021quped}, for FEMNIST, we use a subset of 198 writers so that each client has access to data from 3 authors, which results in a natural type of data heterogeneity due to writing styles of authors. Consequently, number of training samples is between 203-336 and test samples is between 25-40. On MNIST we introduce pathological heterogeneity by letting each client sample data from 3 randomly selected classes only. For both experiments we set $\tau=10$ and vary the batch size so that each epoch consists of 60 iterations. On MNIST we run the experiment for 50 epochs, on FEMNIST we run it for 40 and 80 epochs, for $0.33$ and $0.15$ client sampling ratio, respectively. For each method, we tune hyperparameters such as learning rates on a set of candidate values.

\begin{figure}\label{fig:dp-adaped}
	\centering
	\captionof{table}{Test accuracy (in \%) for 5-layer CNN model.}
	\begin{tabular}{lccl} \toprule 
		Method   & MNIST ($\frac{K}{n}=0.1$) & FEMNIST($\frac{K}{n}=0.33$) & FEMNIST($\frac{K}{n}=0.15$) \\ \midrule
		FedAvg  &  $92.93 \pm 0.09 $&  $93.14 \pm 0.15$& $92.18 \pm 0.13
		$ \\ 
		FedAvg+fine tuning \cite{jiang2019improving}  &  $95.42 \pm 0.25 $&  $93.81 \pm 0.30$& $94.12 \pm 0.26
		$ \\ 
		\texttt{AdaPeD} (Ours) & $\mathbf{98.53} \pm 0.11$& $\mathbf{96.88} \pm0.29$ &  $\mathbf{96.55} \pm 0.32$\\
		pFedMe  \cite{dinh2020personalized} & $97.83 \pm 0.04$ & $ 94.03  \pm 0.61$&  $ 94.95  \pm 0.55$  \\
		Per-FedAvg  \cite{fallah2020personalized} &$95.88 \pm 0.27$  & $93.28  \pm 0.41$& $93.51  \pm 0.31$ \\
		QuPeD (FP) \cite{ozkara2021quped}  &$98.14 \pm 0.39$  &$96.11 \pm 0.13$ &  $95.99 \pm 0.08$ \\
		Federated ML  \cite{shen2020federated} & $97.93 \pm 0.33 $ &$94.83 \pm 0.59 $&  $ 95.12  \pm 0.18$
	\end{tabular}
	\label{tab:Table client sampling}
\end{figure}

\begin{figure}
	\centering
	\includegraphics[scale=0.5]{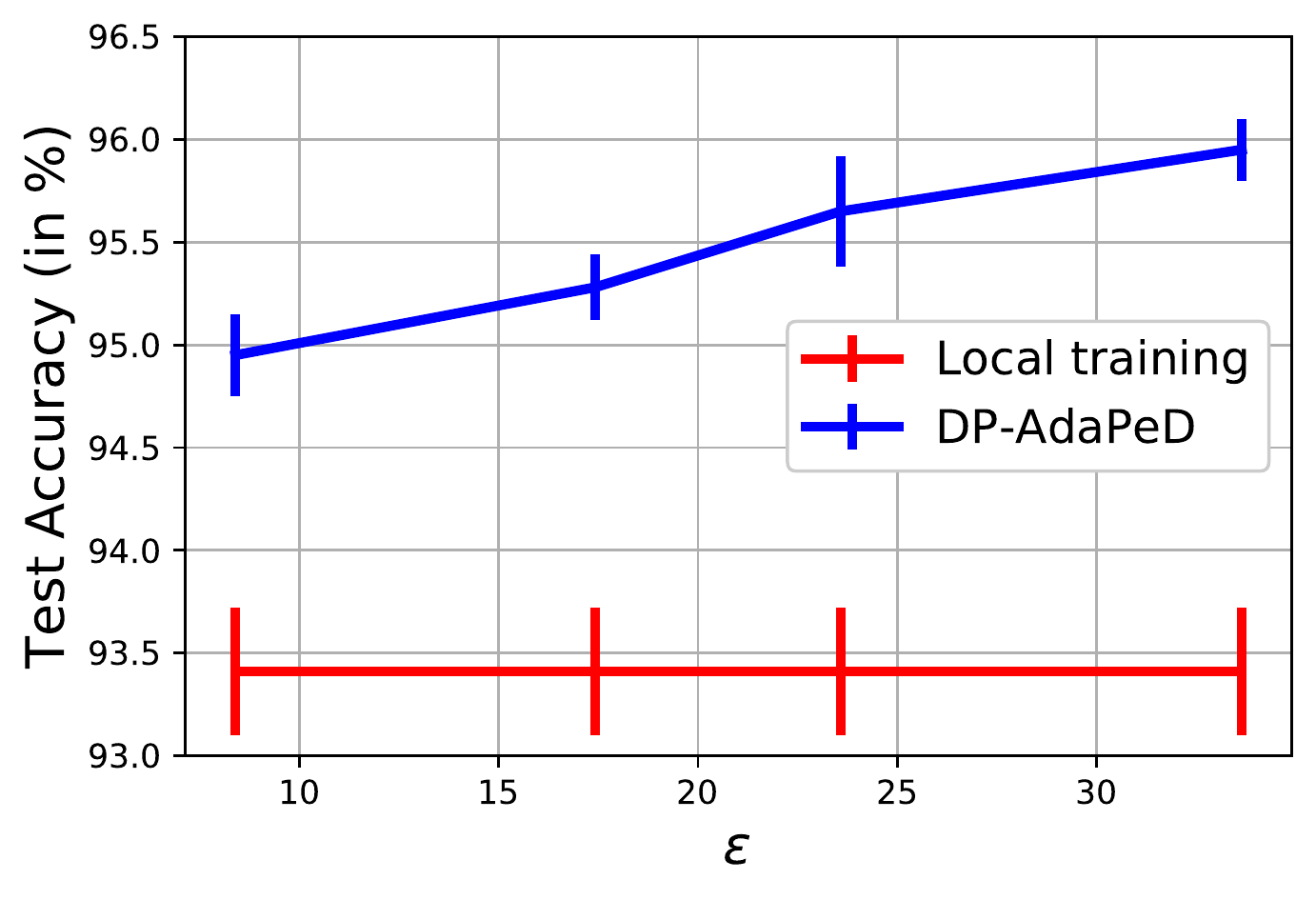}
	\caption{\dpadaped\ Test Accuracy (in \%) vs $\epsilon$ on FEMNIST with client sampling ratio of 0.33. Note that the local training is completely private, however, for convenience, we plot it against $\epsilon$, but it remains constant.}
	\label{fig:dp-adaped}
\end{figure}

$\bullet$ \textbf{Results.} We compare \adaped\ against FedAvg \cite{mcmahan2017communicationefficient}, FedAvg+ \cite{jiang2019improving} and various poersonalized FL algorithms:  pFedMe \cite{dinh2020personalized}, Per-FedAvg \cite{fallah2020personalized}, QuPeD \cite{ozkara2021quped} without model compression, and Federated ML \cite{shen2020federated}. We observe \adaped\ consistently outperforms other methods. It can be seen that methods that use knowledge distillation perform better, on top of this \adaped\ enables us adjust the dependence on collaboration according to the compatibility of global and local decisions/scores. For instance, we set $\sigma_\theta^2$ a certain value initially, and observe as the algorithm runs it decreases which implies clients start to rely on the collaboration more and more. Interestingly, this is not always the case; in particular, for \dpadaped\ we first observe a decrease in $\sigma_\theta^2$ and later the variance estimate increases. This suggests: while there is not much accumulated noise, clients prefer to collaborate; and as the noise accumulation on the global model increases due to DP noise, clients prefer to stop collaboration. This is exactly the type of autonomous behavior that we aimed with adaptive regularization.

$\bullet$ \textbf{DP-AdaPeD.}  In Figure~\ref{fig:dp-adaped}, we observe performance of \dpadaped\ under different $\epsilon$ values. Due to adaptive regularization, clients are able to determine when to collaborate, hence they outperform local training even for low values of $\epsilon$.

\textbf{DP-AdaPeD with local finetuning.} When we employ local finetuning described in Section~\ref{app:adaped}, we can increase AdaPeD's performance under more aggressive privacy constants $\eps$. Another change compared to the result in Section~\ref{sec:experiments} is that now we concatenate the model and $\psi$ and then apply clipping and add noise to the resulting $(d+1)$-dimensional vector. This results in a better performance for relatively low privacy budget. For instance, for FEMNIST with $1/3$ client sampling and the same experimental setting in Section~\ref{sec:experiments} we have the following result in Figure~\ref{fig:dp-adaped-lf}.  

\begin{figure}[H]
	\centering
	\includegraphics[scale=0.5]{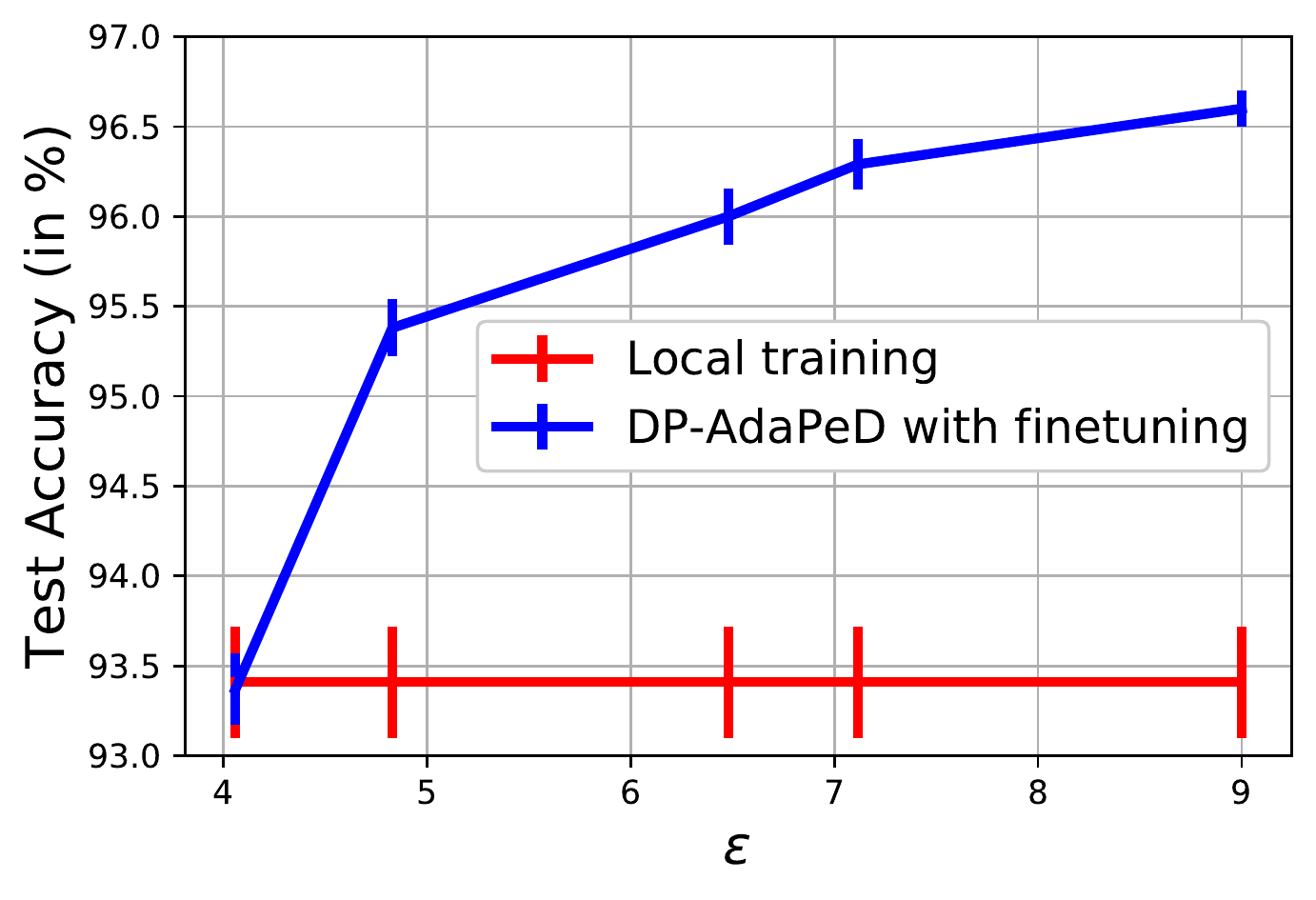}
	\caption{\dpadaped\ with local finetuning Test Accuracy (in \%) vs $\epsilon$ on FEMNIST with client sampling ratio of 0.33. Note that the local training is completely private, however, for convenience, we plot it against $\epsilon$, but it remains constant.}
	\label{fig:dp-adaped-lf}
\end{figure}

\textbf{Comparing to DP version of FedEM~\cite{marfoq2021federated}.} Clustering and multiple global model based methods could have competitive performance in certain scenarios. However, when there are privacy concerns sharing many parameters degrade performance for high privacy, i.e., low $\epsilon$; moreover, since \cite{marfoq2021federated} is based on global models and does not posses a separate local model (which is not just a combination of global models) its performance could be degraded by privatization even more. An additional downside of FedEM \cite{marfoq2021federated} is, when there is client sampling, it requires more iterations for empirical convergence compared to other methods. Hence, to give an advantage to FedEM, we compare AdaPeD and FedEM under full client sampling. On MNIST AdaPeD achieves $99.04\pm0.11\%$ and FedEM achieves $98.90\pm0.20\%$. Therefore, with full client sampling and without privacy the performance is comparable. When we employ Gaussian mechanism for privacy on both methods for parameters that are shared with the server, we obtain the result presented in Figure~\ref{fig:dp-adaped-fedem}.

\begin{figure}[H]
	\centering
	\includegraphics[scale=0.5]{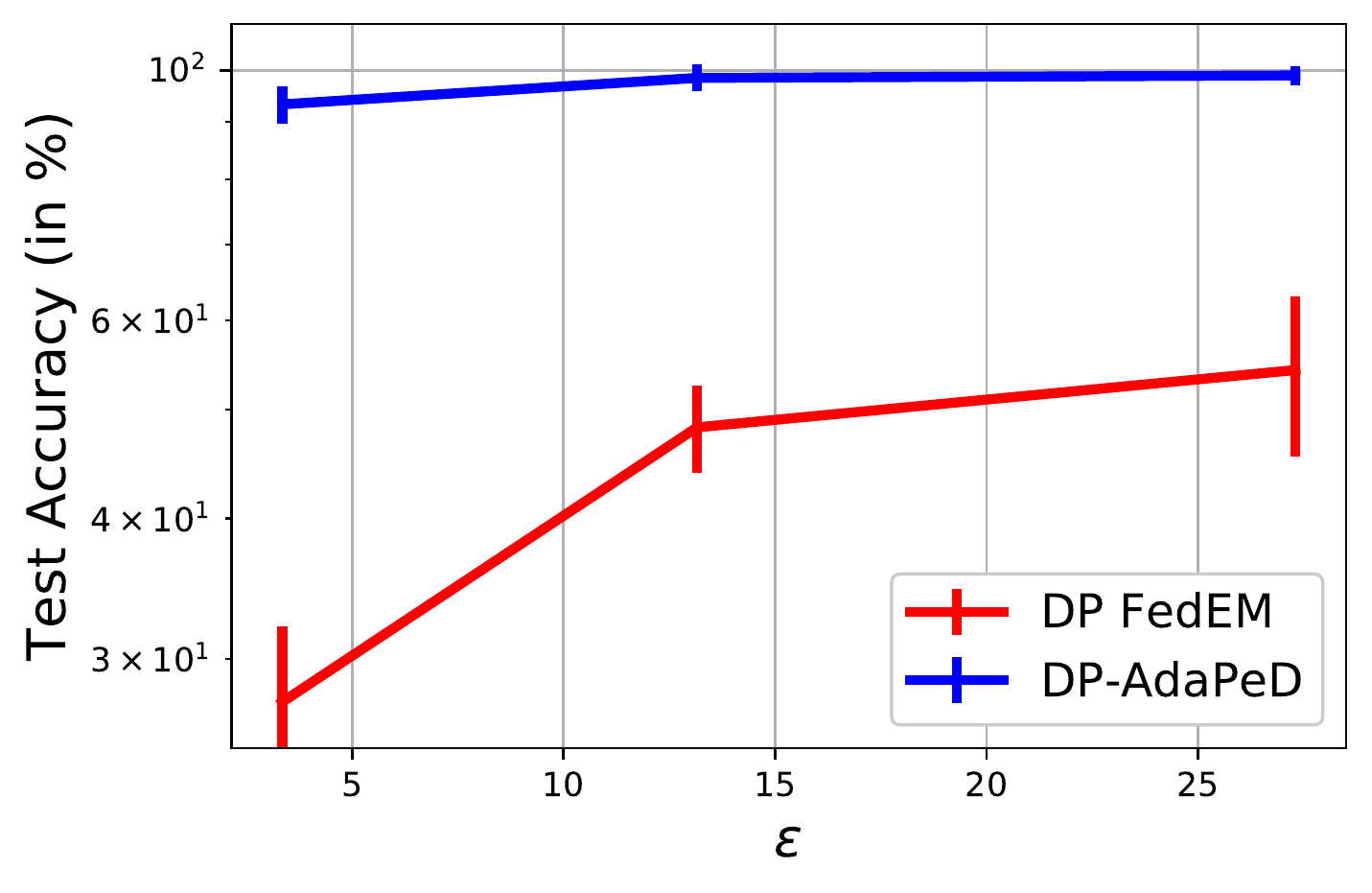}
	\caption{\dpadaped\ and DP FedEM Test Accuracy (in \%) vs $\epsilon$ on MNIST with log scale y axis. Error bar is standard deviation multiplied by 3 for increased visibility.}
	\label{fig:dp-adaped-fedem}
\end{figure}

In DP FedEM -- the private version of FedEM, similar to the concatenation we do in \dpadaped\ for $\psi$ and model parameters, we concatenate two local versions of the global model before adding privacy. As can be seen in Figure~\ref{fig:dp-adaped-fedem}, DP FedEM performs poorly even though FedEM had comparable performance to AdaPed when no privacy was required. This is due to two reasons. One is that more is being shared by each node (multiple models). Secondly, since FedEM just combines global models but does not maintain a separate individual model (which is not a combination of global models) there is further loss in performance due to direct private updates on those models. 

\textbf{DP Personalized Estimation.} To measure the performance tradeoff of the DP mechanism described in Section~\ref{sec:est_gaussian}, we create a synthetic experiment for Gaussian setting. We let $m=10000, n=15$ and $\sigma_{\theta}=0.1,\sigma_{x}=0.5$, and create a dataset at each client as described in Gaussian setting. Applying the DP mechanism we obtain the following result in Figure~\ref{fig:dp-gauss-est}.

\begin{figure}[H]
	\centering
	\includegraphics[scale=0.5]{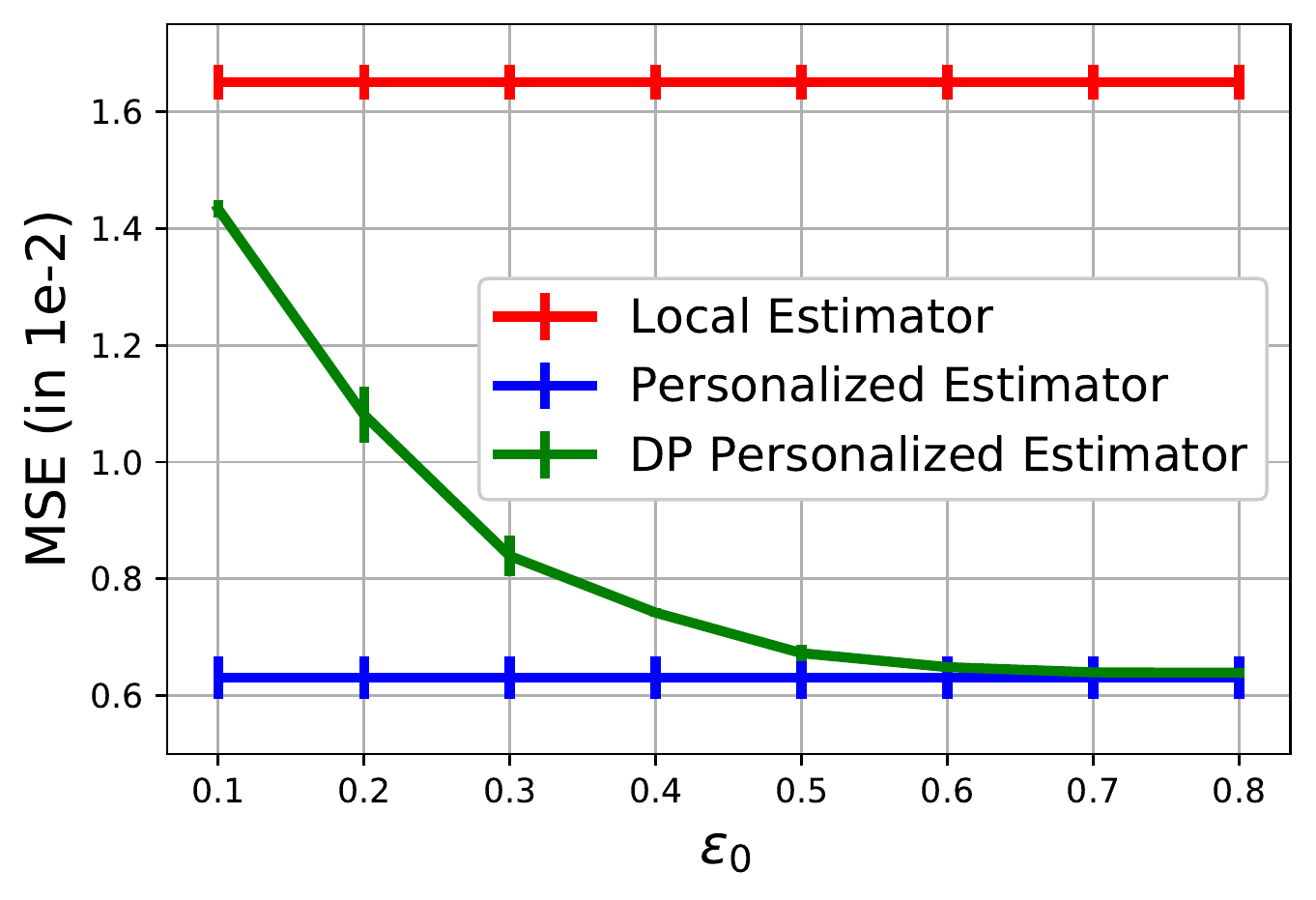}
	\caption{MSE vs $\epsilon_0$. Note that the local training is completely private and personalized training is non-private, however, for convenience, we plot them against $\epsilon$, but they remains constant.}
	\label{fig:dp-gauss-est}
\end{figure}

Here, as expected, we observe that when privacy is low ($\epsilon_0$ is high) the private personalized estimator recovers the regular personalized estimator. When we need higher privacy (lower $\epsilon_0$) the private estimator’s performance starts to become slightly worse than the non-private estimator.

\subsection{Implementation Details }

In this section we give further details on implementation and setting of the experiments that were used in Section~\ref{sec:experiments}. 

\textbf{Linear Regression.} In this experiment we set true values $\sigma_{\theta}^2=0.01, \sigma_{x}^2=0.05$ and we sample each component of $\mu$ from a Gaussian with 0 mean and 0.1 standard deviation and each component of $X$ from a Gaussian with 0 mean and variance 0.05, both i.i.d.

\textbf{Federated Setting.} We implemented Per-FedAvg and pFedMe based on the code from GitHub,\footnote{{\small \url{https://github.com/CharlieDinh/pFedMe}}} and FedEM based on the code from GitHub.\footnote{{\small \url{https://github.com/omarfoq/FedEM}}} Other implementations were not available online, so we implemented ourselves. For each of the methods we tuned learning rate in the set $\{0.2, 0.15, 0.125, 0.1, 0.075, 0.05\}$ and have a decaying learning schedule such that learning rate is multiplied by 0.99 at each epoch. We use weight decay of $1e-4$. For all the methods for both personalized and global models we used a 5-layer CNN, the first two layers consist of convolutional layers of filter size $5\times5$ with 6 and 16 filters respectively. Then we have 3 fully connected layers of dimension $256\times120$, $120\times84$, $84\times10$ and lastly a softmax operation.

\begin{itemize}[leftmargin=*]
	\item AdaPeD\footnote{For federated experiments we have used PyTorch's Data Distributed Parallel package.}: We fine-tuned $\psi$ in between $0.5-5$ with 0.5 increments and set it to 4. We set $\eta_3=3e-2$. We manually prevent $\psi$ becoming smaller than 0.5 so that local loss does not become dominated by the KD loss. We use $\eta_2=0.1$ and $\eta_1=0.1$. \footnote{We use {\small \url{https://github.com/tao-shen/FEMNIST_pytorch}} to import FEMNIST dataset.}
	
	\item Per-FedAvg \cite{fallah2020personalized} and pFedMe \cite{dinh2020personalized}: For Per-FedAvg, we used 0.075 as the learning rate and $\alpha = 0.001$. For pFedMe we used the same learning rate schedule for main learning rate, $K=3$ for the number of local iterations; and we used $\lambda=0.5$, $\eta=0.2$.
	
	\item QuPeD \cite{ozkara2021quped}: We choose $\lambda_p = 0.25$, $\eta_1 = 0.1$ and $\eta_3=0.1$ as stated.
	
	\item Federated Mutual Learning \cite{shen2020federated}: Since authors do not discuss the hyperparameters in the paper, we used $\alpha=\beta=0.25$, global model has the same learning schedule as the personalized models. 
	
	\item FedEM \cite{marfoq2021federated}: Here we use M=2 global models minding the privacy budget. We let  learning rate to 0.2 and use the same schedule and weight decay as other methods.
\end{itemize}

%% file: proofs_estimation.tex
\section{Proofs and Additional Details for Personalized Estimation}\label{sec:proof-est}
\subsection{Personalized Estimation -- Gaussian Model}\label{app:est-gaussian}
\input{app_est_gaussian.tex}
\subsection{Personalized Estimation -- Bernoulli Model}\label{app:est-bernoulli}
\input{app_est_bernoulli.tex}
\subsection{Personalized Estimation -- Mixture Model}\label{app:est-mixture}

\input{app_est_mixture_DD.tex}

%% file: app_est_gaussian.tex
\subsubsection{Proof of Theorem~\ref{thm:gauss_estimate}}\label{app:proof_thm_gauss}
\begin{theorem*}[Restating Theorem~\ref{thm:gauss_estimate}]
	Solving \eqref{eqn:post_estimator} yields the following closed form expressions for $\hmu$ and $\htheta_1,\ldots,\htheta_m$:
	\begin{align}\label{eqn:estimator}
		\hmu = \frac{1}{m}\sum_{i=1}^m \overline{X}_i \qquad \text{ and } \qquad 		
		\htheta_{i}=a\overline{X}_i+(1-a)\hmu, \text{ for } i\in[m], \quad \text{ where } a=\frac{\sigmatheta^{2}}{\sigmatheta^{2}+\nicefrac{\sigma_{x}^{2}}{n}}.
	\end{align}
	The above estimator achieves the MSE: 
	$\mathbb{E}_{\vct{\theta}_i,X_1,\hdots,X_m}\|\htheta_i-\vct{\theta}_i\|^2 \leq \frac{d\sigma_x^2}{n}\Big(\frac{1-a}{m}+a\Big).$
\end{theorem*}

\begin{proof}

We will derive the optimal estimator and prove the MSE for one dimensional case, i.e., for $d=1$; the final result can be obtained by applying these to each of the $d$ coordinates separately.

The posterior estimators of the local means $\theta_1,\ldots,\theta_m$ and the global mean $\mu$ are obtained by solving the following optimization problem:
\begin{align}~\label{eqn:post_estimator}
		\hat{\theta}_1,\ldots,\hat{\theta}_m,\hat{\mu}&=\argmax_{\theta_1,\ldots,\theta_m,\mu} p_{\mathbf{X}|\theta}\left(X_1,\ldots,X_m|\theta_1,\ldots,\theta_m\right)p_{\theta|\mu}(\theta_1,\ldots,\theta_m|\mu)\\
		&=\argmin_{\theta_1,\ldots,\theta_m,\mu}-\log\left(p_{\mathbf{X}|\theta}\left(X_1,\ldots,X_m|\theta_1,\ldots,\theta_m\right)\right)-\log\left(p_{\theta|\mu}(\theta_1,\ldots,\theta_m|\mu)\right)\\
		&=\argmin_{\theta_1,\ldots,\theta_m,\mu} C + \sum_{i=1}^{m}\sum_{j=1}^{n}\frac{\left(X_{i}^{j}-\theta_i\right)^2}{\sigma_x^2}+\sum_{i=1}^{m}\frac{\left(\theta_i-\mu\right)^2}{\sigma_{\theta}^2},
\end{align}
where the second equality is obtained from the fact that the $\log$ function is a monotonic function, and $C$ is a constant independent of the variables $\theta=\left(\theta_1,\ldots,\theta_{m}\right)$. Observe that the objective function $F(\theta,\mu)=\sum_{i=1}^{m}\sum_{j=1}^{n}\frac{\left(X_{i}^{j}-\theta_i\right)^2}{\sigma_x^2}+\sum_{i=1}^{m}\frac{\left(\theta_i-\mu\right)^2}{\sigma_{\theta}^2}$ is jointly convex in $(\theta,\mu)$. Thus, the optimal is obtained by setting the derivative to zero as it is an unbounded optimization problem. 

\begin{align}
	\frac{\partial F}{\partial \theta_i}\bigg|_{\mu=\hat{\mu},\theta_i=\hat{\theta}_i} &=\frac{\sum_{j=1}^{n}2(\hat{\theta}_i-X_{i}^{j})}{\sigma_x^2}+\frac{2(\hat{\theta}_i-\hat{\mu})}{\sigma_{\theta}^2}=0,\qquad \forall i\in[m]\\
	\frac{\partial F}{\partial\mu}\bigg|_{\mu=\hat{\mu},\theta_i=\hat{\theta}_i} &= \frac{\sum_{i=1}^{m}2(\hat{\mu}-\hat{\theta}_i)}{\sigma_{\theta}^2}=0.
\end{align}    
By solving these $m+1$ equations in $m+1$ unknowns, we get:
\begin{equation}~\label{eqn:app_estimate}
	\hat{\theta}_{i}=\alpha\left(\frac{1}{n}\sum_{j=1}^{n}X_{i}^{j}\right)+(1-\alpha)\left(\frac{1}{mn}\sum_{i=1}^{m}\sum_{j=1}^{n}X_{i}^{j}\right),
\end{equation}
where $\alpha=\frac{\sigma_{\theta}^{2}}{\sigma_{\theta}^{2}+\frac{\sigma_{x}^{2}}{n}}$. 
By letting $\overline{X}_i=\frac{1}{n}\sum_{j=1}^{n}X_{i}^{j}$ for all $i\in[m]$, we can write $\hat{\theta}_{i} = \alpha \overline{X}_i + (1-\alpha)\frac{1}{m}\sum_{i=1}^m\overline{X}_i$.

Observe that $\mathbb{E}\left[\hat{\theta}_i|\theta\right]=\alpha\theta_i+\frac{1-\alpha}{m}\sum_{l=1}^{m}\theta_l$, where $\theta=(\theta_1,\ldots,\theta_m)$. Thus, the estimator~\eqref{eqn:app_estimate} is an unbiased estimate of $\lbrace\theta_i\rbrace$. Substituting the $\hat{\theta}_i$ in the MSE, we get that
\begin{align}
	\mathbb{E}_{X_1,\ldots,X_m}\left[\left(\hat{\theta}_i-\theta_i\right)^2\right]&=\mathbb{E}_{\theta}\left[\mathbb{E}_{X_1,\ldots,X_m}\left[\left(\hat{\theta}_i-\theta_i\right)^2|\theta\right]\right] \notag \\
	&=\mathbb{E}_{\theta}\left[\mathbb{E}_{X_1,\ldots,X_m}\left[\left(\hat{\theta}_i-\mathbb{E}\left[\hat{\theta}_i|\theta\right]+\mathbb{E}\left[\hat{\theta}_i|\theta\right]-\theta_i\right)^2|\theta\right]\right] \notag \\
	&=\mathbb{E}_{\theta}\left[\mathbb{E}_{X_1,\ldots,X_m}\left[\left(\hat{\theta}_i-\mathbb{E}\left[\hat{\theta}_i|\theta\right]\right)^2|\theta\right]\right]+\mathbb{E}_{\theta}\left[\mathbb{E}_{X_1,\ldots,X_m}\left[\left(\mathbb{E}\left[\hat{\theta}_i|\theta\right]-\theta_i\right)^2|\theta\right]\right] \label{proof_gauss_interim2}
\end{align}
\begin{claim}\label{claim:proof_gauss_claim}
	\begin{align*}
		\mathbb{E}_{\theta}\left[\mathbb{E}_{X_1,\ldots,X_m}\left[\left(\hat{\theta}_i-\mathbb{E}\left[\hat{\theta}_i|\theta\right]\right)^2|\theta\right]\right] &= \alpha^2\frac{\sigma_x^2}{n}+(1-\alpha)^2\frac{\sigma_x^2}{mn}+2\alpha(1-\alpha)\frac{\sigma_x^2}{mn} \\
		\mathbb{E}_{\theta}\left[\mathbb{E}_{X_1,\ldots,X_m}\left[\left(\mathbb{E}\left[\hat{\theta}_i|\theta\right]-\theta_i\right)^2|\theta\right]\right] &= (1-\alpha)^2\mathbb{E}_{\theta}\left[\left(\frac{1}{m}\sum_{k=1}^{m}\theta_k-\theta_i\right)^2\right] \leq  (1-\alpha)^2\frac{\sigma_{\theta}^2(m-1)}{m}
	\end{align*}
\end{claim}

Substituting the result of Claim~\ref{claim:proof_gauss_claim} into \eqref{proof_gauss_interim2}, we get
\begin{align}
	\mathbb{E}_{X_1,\ldots,X_m}\left[\left(\hat{\theta}_i-\theta_i\right)^2\right]&\leq\alpha^2\frac{\sigma_x^2}{n}+(1-\alpha)^2\frac{\sigma_x^2}{mn}+2\alpha(1-\alpha)\frac{\sigma_x^2}{mn}+(1-\alpha)^2\frac{\sigma_{\theta}^2(m-1)}{m} \label{proof_gauss_interim3} \\
	&\stackrel{\text{(a)}}{=}\frac{\sigma_x^2}{n}\left(\alpha^2+\frac{(1-\alpha)^2+2\alpha(1-\alpha)}{m}+\alpha(1-\alpha)\frac{m-1}{m}\right) \notag \\
	&=\frac{\sigma_x^2}{n}\left(\alpha+\frac{1-\alpha}{m}\right), \notag
\end{align}    
where in (a) we used $\alpha=\frac{\sigma_{\theta}^2}{\sigma_{\theta}^2+\frac{\sigma_x^2}{n}}$ for the last term to write $(1-\alpha)^2\frac{\sigma_{\theta}^2(m-1)}{m} = \frac{\sigma_x^2}{n}\alpha(1-\alpha)\frac{m-1}{m}$.

Observe that the estimator in~\eqref{eqn:app_estimate} is a weighted summation between two estimators: the local estimator $\overline{X}_i=\frac{1}{n}\sum_{j=1}^{n}X_{i}^{j}$, and the global estimator $\hat{\mu}=\frac{1}{m}\sum_{i=1}^{m}\overline{X}_i$. Thus, the MSE in (a) consists of four terms: 1) The variance of the local estimator ($\frac{\sigma_x^2}{n}$). 2) The variance of the global estimator ($\frac{\sigma_x^2}{nm}$). 3) The correlation between the local estimator and the global estimator ($\frac{\sigma_x^2}{nm}$). 4) The bias term $\mathbb{E}_{\theta}\left[\mathbb{E}_{X_1,\ldots,X_m}\left[\left(\mathbb{E}\left[\hat{\theta}_i|\theta\right]-\theta_i\right)^2|\theta\right]\right]$. 
This completes the proof of Theorem~\ref{thm:gauss_estimate}. 
\end{proof}

\begin{proof}[Proof of Claim~\ref{claim:proof_gauss_claim}]
	For the first equation:
	\begin{align*}
		&\mathbb{E}_{\theta}\left[\mathbb{E}_{X_1,\ldots,X_m}\left[\left(\hat{\theta}_i-\mathbb{E}\left[\hat{\theta}_i|\theta\right]\right)^2|\theta\right]\right] = \mathbb{E}_{\theta}\left[\mathbb{E}_{X_1,\ldots,X_m}\left[\left(\alpha(\overline{X}_i-\theta_i) + (1-\alpha)\frac{1}{m}\sum_{k=1}^m(\overline{X}_k-\theta_k)\right)^2\ |\ \theta\right]\right] \\
		&\qquad= \alpha^2\mathbb{E}\left[\mathbb{E}\left[(\overline{X}_i-\theta_i)^2\ |\ \theta\right]\right] + (1-\alpha)^2\mathbb{E}\left[\mathbb{E}\left[\left(\frac{1}{m}\sum_{k=1}^m(\overline{X}_k-\theta_k)\right)^2\ |\ \theta\right]\right] \\
		&\hspace{3cm}+ 2\alpha(1-\alpha)\mathbb{E}\left[\mathbb{E}\left[\frac{1}{m}\sum_{k=1}^m(\overline{X}_i-\theta_i)(\overline{X}_k-\theta_k)\ |\ \theta\right]\right] \\
		&\qquad= \alpha^2\frac{\sigma_x^2}{n} + (1-\alpha)^2\frac{\sigma_x^2}{mn} + 2\alpha(1-\alpha)\frac{\sigma_x^2}{mn}
	\end{align*}
	For the second equation, first note that $\mathbb{E}\left[\hat{\theta}_i|\theta\right]-\theta_i = \alpha\theta_i+\frac{1-\alpha}{m}\sum_{k=1}^{m}\theta_k - \theta_i = (1-\alpha)\left(\frac{1}{m}\sum_{k=1}^{m}\theta_k - \theta_i\right)$:
	\begin{align*}
		&\mathbb{E}_{\theta}\left[\mathbb{E}_{X_1,\ldots,X_m}\left[\left(\mathbb{E}\left[\hat{\theta}_i|\theta\right]-\theta_i\right)^2|\theta\right]\right] = (1-\alpha)^2\mathbb{E}\left[\left(\frac{1}{m}\sum_{k=1}^{m}\theta_k - \theta_i\right)^2\right] \\
		&\qquad= \frac{(1-\alpha)^2}{m^2}\mathbb{E}\left[\left(\sum_{k\neq i}(\theta_k - \theta_i)\right)^2\right] \\
		&\qquad= \frac{(1-\alpha)^2}{m^2}\left[\sum_{k\neq i}\mathbb{E}(\theta_k - \theta_i)^2 + \sum_{k\neq i, l\neq i, k\neq l}\mathbb{E}(\theta_k - \theta_i)(\theta_l - \theta_i)\right] \\
		&\qquad\leq \frac{(1-\alpha)^2}{m^2}\left[\sum_{k\neq i}[\mathbb{E}(\theta_k - \mu)^2 + \mathbb{E}(\theta_i - \mu)^2] + \sum_{k\neq i, l\neq i, k\neq l}\mathbb{E}(\theta_k - \theta_i)(\theta_l - \theta_i)\right] \\
		&\qquad= \frac{(1-\alpha)^2}{m^2}\left[2(m-1)\sigma_{\theta}^2 + \sum_{k\neq i, l\neq i, k\neq l}\mathbb{E}(\theta_k - \theta_i)(\theta_l - \theta_i)\right] \\
		&\qquad= \frac{(1-\alpha)^2}{m^2}\left[2(m-1)\sigma_{\theta}^2 + \sum_{k\neq i, l\neq i, k\neq l}\mathbb{E}(\mu - \theta_i)^2\right] \tag{Since $\mathbb{E}[\theta_k]=\mu$ for all $k\in[m]$} \\
		&\qquad= \frac{(1-\alpha)^2}{m^2}\left[2(m-1)\sigma_{\theta}^2 + (m-1)(m-2)\sigma_{\theta}^2\right] \\
		&\qquad= (1-\alpha)^2\frac{\sigma_{\theta}^2(m-1)}{m}
	\end{align*}
	This concludes the proof of Claim~\ref{claim:proof_gauss_claim}.
\end{proof}

\subsubsection{Proof of Theorem~\ref{thm:gaussian_est-q}}

\begin{theorem*}[Restating Theorem~\ref{thm:gaussian_est-q}]
	Suppose for all $\bx\in\bbR^d$, $q$ satisfies $\bbE[q(\bx)]=\bx$ and $\bbE\|q(\bx)-\bx\|^2\leq d\sigma_q^2$ for some finite $\sigma_q$.
	Then the personalized estimator in \eqref{eqn:person_const} has MSE:
	\begin{align}\label{eqn:gaussian_generic-q}
		\mathbb{E}_{\vct{\theta}_i,q,X_1,\ldots,X_m}\|\htheta_i-\vct{\theta}_i\|^2 \leq \frac{d\sigma_x^2}{n}\Big(\frac{1-a}{m}+a\Big) \qquad \text{ where }\qquad a=\frac{\sigma_{\theta}^{2}+\nicefrac{\sigma_q^2}{m-1}}{\sigma_{\theta}^{2}+\nicefrac{\sigma_q^2}{m-1}+\nicefrac{\sigma_x^{2}}{n}}.
	\end{align}
	Furthermore, assuming $\vct{\mu}\in[-r,r]$ for some constant $r$ (but $\vct{\mu}$ is unknown), we have:
	\begin{enumerate}
		\item {\it Communication efficiency:} For any $k\in\mathbb{N}$, there is a $q$ whose output can be represented using $k$-bits (i.e., $q$ is a quantizer) that achieves the MSE in \eqref{eqn:gaussian_generic-q} with probability at least $1-\nicefrac{2}{mn}$ and with $\sigma_q = \frac{b}{(2^k-1)}$, 
		where $b=r+\sigmatheta\sqrt{\log(m^2n)}+\frac{\sigma_x}{\sqrt{n}}\sqrt{\log(m^2n)}$. 
		\item {\it Privacy:} For any $\epsilon_0\in(0,1),\delta>0$, there is a $q$ that is user-level $(\epsilon_0,\delta)$-locally differentially private, that achieves the MSE in \eqref{eqn:gaussian_generic-q} with probability at least $1-\nicefrac{2}{mn}$ and with $\sigma_q=\frac{b}{\epsilon_0}\sqrt{8\log(2/\delta)}$, 
		where $b=r+\sigmatheta\sqrt{\log(m^2n)}+\frac{\sigma_x}{\sqrt{n}}\sqrt{\log(m^2n)}$. 
	\end{enumerate}
\end{theorem*}
\begin{proof}[Proof of Theorem~\ref{thm:gaussian_est-q}, Equation~\eqref{eqn:gaussian_generic-q}]

Similar to the proof of Theorem~\ref{thm:gauss_estimate}, here also we will derive the optimal estimator and prove the MSE for the one dimensional case, and the final result can be obtained by applying these to each of the $d$ coordinates separately.

Let $\theta=(\theta_1,\ldots,\theta_m)$ denote the personalized models vector. For given a constraint function $q$, we set the personalized model as follows:
\begin{equation}
	\hat{\theta}_i=\alpha \left(\frac{1}{n}\sum_{j=1}^{n}X_{i}^{j}\right)+(1-\alpha)\left(\frac{1}{m}\sum_{i=1}^{m}q(\overline{X}_i)\right)\qquad \forall i\in[m],
\end{equation}
where $\overline{X}_i=\frac{1}{n}\sum_{j=1}^{n}X_{i}^{j}$. From the second condition on the function $q$, we get that 
\begin{equation}~\label{eqn:expect_personal}
	\mathbb{E}\left[\hat{\theta}_i|\theta\right]=\alpha \theta_i+\frac{1-\alpha}{m}\sum_{l=1}^{m}\theta_l,
\end{equation}
Thus, by following similar steps as the proof of Theorem~\ref{thm:gauss_estimate}, we get that:
\begin{align}
	\mathbb{E}\left[\left(\hat{\theta}_i-\theta_i\right)^2\right]&=\mathbb{E}\left[\mathbb{E}\left[\left(\hat{\theta}_i-\theta_i\right)^2|\theta\right]\right] \notag \\
	&=\mathbb{E}\left[\mathbb{E}\left[\left(\hat{\theta}_i-\mathbb{E}\left[\hat{\theta}_i|\theta\right]+\mathbb{E}\left[\hat{\theta}_i|\theta\right]-\theta_i\right)^2|\theta\right]\right] \notag \\
	&=\mathbb{E}\left[\mathbb{E}\left[\left(\hat{\theta}_i-\mathbb{E}\left[\hat{\theta}_i|\theta\right]\right)^2|\theta\right]\right]+\mathbb{E}\left[\mathbb{E}\left[\left(\mathbb{E}\left[\hat{\theta}_i|\theta\right]-\theta_i\right)^2|\theta\right]\right] \notag \\
	&\stackrel{(a)}{=}\alpha^2\frac{\sigma_x^2}{n}+(1-\alpha)^2\mathbb{E}\left[\left(\frac{1}{m}\sum_{l=1}^{m}q(\overline{X}_l)-\theta_l\right)^2|\theta\right] \notag \\
	&\quad+2\alpha(1-\alpha)\mathbb{E}\left[\left(\overline{X}_i-\theta_i\right)\left(\frac{1}{m}\sum_{l=1}^{m}q(\overline{X}_l)-\theta_l\right)|\theta\right]+(1-\alpha)^2\mathbb{E}\left[\left(\frac{1}{m}\sum_{k=1}^{m}\theta_k-\theta_i\right)^2\right] \notag \\
	&\stackrel{(b)}{=}\alpha^2\frac{\sigma_x^2}{n}+\frac{(1-\alpha)^2\left(\frac{\sigma_x^2}{n}+\sigma_q^2\right)}{m}+\frac{2\alpha(1-\alpha)\sigma_x^2}{mn}+(1-\alpha)^2\mathbb{E}\left[\left(\frac{1}{m}\sum_{k=1}^{m}\theta_k-\theta_i\right)^2\right] \notag \\
	&\leq\alpha^2\frac{\sigma_x^2}{n}+\frac{(1-\alpha)^2\left(\frac{\sigma_x^2}{n}+\sigma_q^2\right)}{m}+2\alpha(1-\alpha)\frac{\sigma_x^2}{mn}+(1-\alpha)^2\frac{\sigma_{\theta}^2(m-1)}{m} \notag \\
	&\stackrel{(c)}{=}\frac{\sigma_x^2}{n}\left(\alpha^2+\frac{(1-\alpha)^2+2\alpha(1-\alpha)}{m}+\alpha(1-\alpha)\frac{m-1}{m}\right) \notag \\
	&=\frac{\sigma_x^2}{n}\left(\alpha+\frac{1-\alpha}{m}\right), \label{eq:gaussian_est-q-mse-bound}
\end{align}    
where step (a) follows by substituting the expectation of the personalized model from~\eqref{eqn:expect_personal}. Step (b) follows from the first and third conditions of the function $q$. Step (c) follows by choosing $\alpha=\frac{\sigma_{\theta}^2+\frac{\sigma_q^2}{m-1}}{\sigma_{\theta}^2+\frac{\sigma_q^2}{m-1}+\frac{\sigma_x^2}{n}}$. This derives the result stated in \eqref{eqn:gaussian_generic-q} in Theorem~\ref{thm:gaussian_est-q}.  
\end{proof}

\begin{proof}[Proof of Theorem~\ref{thm:gaussian_est-q}, Part 1]
The proof consists of two steps. First, we use the concentration property of the Gaussian distribution to show that the local sample means $\lbrace\overline{X}_i\rbrace$ are bounded within a small range with high probability. Second, we apply an unbiased stochastic quantizer on the projected sample mean. 

The local samples $X_{i}^{1},\ldots,X_{i}^{n}$ are drawn i.i.d. from a Gaussian distribution with mean $\theta_i$ and variance $\sigma_x^2$, and hence, we have that $\overline{X}_i\sim \mathcal{N}(\theta_i,\frac{\sigma_x^2}{n})$. Thus, from the concentration property of the Gaussian distribution, we get that $\Pr[|\overline{X}_i-\theta_i|>c_1]\leq \exp\left(-\frac{nc_1^2}{\sigma_x^2}\right)$ for all $i\in[m]$. Similarly, the models $\theta_1,\ldots,\theta_m$ are drawn i.i.d. from a Gaussian distribution with mean $\mu\in[-r,r]$ and variance $\sigma_{\theta}^2$, hence,, we get $\Pr[|\theta_i-\mu|>c_2]\leq \exp\left(-\frac{c_2^2}{\sigma_\theta^2}\right)$ for all $i\in[m]$. Let $\mathcal{E}=\left\{\overline{X}_i\in[-a,a]:\forall i\in [m]\right\}$, where $a=r+c_1+c_2$. Thus, from the union bound, we get that $\Pr[\mathcal{E}]>1-m(e^{-\frac{nc_1^2}{\sigma_x^2}}+e^{-\frac{c_2^2}{\sigma_\theta^2}})$. By setting $c_1=\sqrt{\frac{\sigma_x^2}{n}\log(m^2n)}$ and $c_2=\sqrt{\sigma_{\theta}^2\log(m^2n)}$, we get that $a=r+\frac{\sigma_x}{\sqrt{n}}\sqrt{\log(m^2n)}+\sigma_{\theta}\sqrt{\log(m^2n)}$, and $\Pr[\mathcal{E}]=1-\frac{2}{mn}$.

Let $q_k:[-a,a]\to\mathcal{Y}_k$ be a quantization function with $k$-bits, where $\mathcal{Y}_k$ is a discrete set of cardinality $|\mathbb{Y}_k|=2^{k}$. For given $x\in[-a,a]$, the output of the function $q_k$ is given by:
\begin{equation}
	q_k(x) = \frac{2a}{2^{k}-1}\left(\lfloor \tilde{x}\rfloor + \mathsf{Bern}\left(\tilde{x}-\lfloor \tilde{x}\rfloor\right)\right)-a,    
\end{equation}
where $\mathsf{Bern}(p)$ is a Bernoulli random variable with bias $p$, and $\tilde{x} = \frac{2^{k}-1}{2a}\left(x+a\right)\in[0,2^{k}-1]$. Observe that the output of the function $q_k$ requires only $k$-bits for transmission. Furthermore, the function $q_k$ satisfies the following conditions:
\begin{align}
	\mathbb{E}\left[q_k(x)\right] &=x,\\
	\sigma_{q_k}^2 &=\mathbb{E}\left[(q_k(x)-x)^2\right]\leq \frac{a^2}{(2^{k}-1)^2}.
\end{align}
Let each client applies the function $q_k$ on the projected local mean $\tilde{X}_i=\mathsf{Proj}_{[-a,a]}\left[\overline{X}_i\right]$ and sends the output to the server for all $i\in[m]$. Conditioned on the event $\mathcal{E}$, i.e., $\overline{X}_i\in[-a,a] \quad \forall i\in[m]$, and using \eqref{eq:gaussian_est-q-mse-bound}, we get that
\begin{equation}
	\begin{aligned}
		MSE&=\mathbb{E}_{\theta,\mathbf{X}}\left[\left(\hat{\theta}_i-\theta_i\right)^2\right]\leq \frac{\sigma_x^2}{n}\left(\frac{1-\alpha}{m}+\alpha\right),\\
	\end{aligned}
\end{equation}
where $\alpha=\frac{\sigma_{\theta}^{2}+\frac{a^2}{(2^{k}-1)^2(m-1)}}{\sigma_{\theta}^{2}+\frac{a^2}{(2^{k}-1)^2(m-1)}+\frac{\sigma_{x}^{2}}{n}}$ and $a=r+\frac{\sigma_x}{\sqrt{n}}\sqrt{\log(m^2n)}+\sigma_{\theta}\sqrt{\log(m^2n)}$. 
Note that the event $\mathcal{E}$ happens with probability at least $1-\frac{2}{mn}$. 
\end{proof}

\begin{proof} [Proof of Theorem~\ref{thm:gaussian_est-q}, Part 2]
We define the (random) mechanism $q_p:[-a,a]\to\mathcal{R}$ that takes an input $x\in[-a,a]$ and generates a user-level $(\epsilon_0,\delta)$-LDP output $y\in\mathbb{R}$, where $y=q_p(x)$ is given by:
\begin{equation}
	q_p(x) = x +\nu,     
\end{equation}
where $\nu\sim\mathcal{N}(0,\sigma_{\epsilon_0}^2)$ is a Gaussian noise. By setting $\sigma_{\epsilon_0}^2=\frac{8a^2\log(2/\delta)}{\epsilon_0^2}$, we get that the output of the function $q_p(x)$ is $(\epsilon_0,\delta)$-LDP from~\cite{dwork2014algorithmic}. Furthermore, the function $q_p$ satisfies the following conditions:
\begin{align}
	\mathbb{E}\left[q_p(x)\right] &=x,\\
	\sigma_{q_p}^2 &=\mathbb{E}\left[(q_p(x)-x)^2\right]\leq \frac{8a^2\log(2/\delta)}{\epsilon_0^2}.
\end{align}
Similar to the proof of Theorem~\ref{thm:gaussian_est-q}, Part 1, 
let each client applies the function $q_p$ on the projected local mean $\tilde{X}_i=\mathsf{Proj}_{[-a,a]}\left[\overline{X}_i\right]$ and sends the output to the server for all $i\in[m]$. Conditioned on the event $\mathcal{E}$, i.e., $\overline{X}_i\in[-a,a] \quad \forall i\in[m]$, and using \eqref{eq:gaussian_est-q-mse-bound}, we get that
\begin{equation}
	\begin{aligned}
		\text{MSE} &=\mathbb{E}_{\theta,\mathbf{X}}\left[\left(\hat{\theta}_i-\theta_i\right)^2\right] \leq \frac{\sigma_x^2}{n}\left(\frac{1-\alpha}{m}+\alpha\right),\\
	\end{aligned}
\end{equation}
where $\alpha=\frac{\sigma_{\theta}^{2}+\frac{8a^2\log(2/\delta)}{\epsilon_0^2(m-1)}}{\sigma_{\theta}^{2}+\frac{8a^2\log(2/\delta)}{\epsilon_0^2(m-1)}+\frac{\sigma_{x}^{2}}{n}}$ and $a=r+\frac{\sigma_x}{\sqrt{n}}\sqrt{\log(m^2n)}+\sigma_{\theta}\sqrt{\log(m^2n)}$. 
Note that the event $\mathcal{E}$ happens with probability at least $1-\frac{2}{mn}$. 
\end{proof}
\subsubsection{Lower Bound}
Here we discuss the lower bound using Fisher information technique similar to \cite{barnes2020lower}. In particular we use a Bayesian version of Cramer-Rao lower bound and van Trees inequality \cite{vanTrees95}. Let us denote $f(X|\theta)$ as the data generating conditional density function and $\pi(\theta)$ as the prior distribution that generates $\theta$. Let us denote $\mathbb{E}_{\theta}$ as the expectation with respect to the randomness of $\theta$ and $\mathbb{E}$ as the expectation with respect to randomness of $X$ and $\theta$. First we define two types of Fisher information:

\begin{align*}
	I_X(\theta) &= \mathbb{E}_{\theta} \nabla_{\theta} \log(f(X|\theta)) \nabla_{\theta} \log(f(X|\theta))^T\\
	I(\pi) &= \mathbb{E} \nabla_{\theta} \log (\pi(\theta))\nabla_{\theta} \log (\pi(\theta))^T
\end{align*}
namely Fisher information of estimating $\theta$ from samples $X$ and Fisher information of prior $\pi$. Here the logarithm is elementwise.
For van Trees inequality we need the following regularity conditions:
\begin{itemize}
	\item $f(X|\cdot)$ and $\pi(\cdot)$ are absolutely continuous and $\pi(\cdot)$ vanishes at the end points of $\Theta$.
	\item $\mathbb{E}_{\theta} \nabla_{\theta} \log(f(X|\theta)) = 0$
	\item We also assume both density functions are continuously differentiable.
\end{itemize}

These assumptions are satisfied for the Gaussian setting for any finite mean $\mu$, they are satisfied for Bernoulli setting as long as parameters $\alpha$ and $\beta$ are larger than 1. Assuming local samples $X$ are generated i.i.d with $f(x|\theta)$, the van Trees inequality for one dimension is as follows:

\begin{align*}
	\mathbb{E}(\widehat{\theta}(X)-\theta)^2 \geq \frac{1}{n \mathbb{E}I_x(\theta)+I(\pi)}
\end{align*}

where $I_X(\theta)=\mathbb{E}_{\theta} \log(f(X|\theta))'^2$ and $I(\pi)= \mathbb{E} \log (\pi(\theta))'^2$.
Assuming $\theta \in \mathbb{R}^d$ and each dimension is independent from each other, by \cite{vanTrees95} we have:

\begin{align}\label{van_Trees}
	\mathbb{E}\|\widehat{\theta}(X)-\theta\|^2 \geq \frac{d^2}{n \mathbb{E}\text{Tr}(I_x(\theta))+\text{Tr}(I(\pi))}
\end{align}

Note, the lower bound on the average risk directly translates as a lower bound on $\sup_{\theta\in\Theta} \mathbb{E}_X\|\widehat{\theta}(X)-\theta\|^2$. Before our proof we have a useful fact:

\begin{fact}\label{gauss_mean_fact}
	Given some random variable $X\sim N(Y,\sigma_y^2)$ where $Y\sim N(Z,\sigma_z^2)$ we have $X\sim N(z,\sigma_z^2+\sigma_y^2)$. 
\end{fact}
\begin{proof}
	We will give the proof in one dimension, however, it can easily be extended to multidimensional case where each dimension is independent. For all $t \in \mathbb{R}$ we have,
	\begin{align*}
		\mathbb{E}_X[\exp(itX)] = \mathbb{E}_Y\mathbb{E}_X[\exp(itX)|Y]&=\mathbb{E}_Y[\exp(itY-\frac{\sigma_x^2t^2}{2})]\\
		& = \exp(-\frac{\sigma_x^2t^2}{2})\mathbb{E}_Y[\exp(itY)]\\
		&=\exp(-\frac{\sigma_x^2t^2}{2})\exp(itz-\frac{\sigma_y^2t^2}{2})\\
		&= \exp(itz-\frac{(\sigma_x^2+\sigma_y^2)t^2}{2})
	\end{align*}
where the last line is the characteristic function of a Gaussian with mean $z$ and variance $\sigma_x^2+\sigma_y^2$.
\end{proof}

\textbf{Gaussian case with perfect knowledge of prior.}\label{thm:lower_bound_perfect}
In this setting we know that  $\theta_i \sim N(\mu\mathbf{1},\sigma_{\theta}^2\text{I}_{d})$, hence, $I(\pi) = \frac{1}{\sigma_{\theta}^2}\text{I}_{d}$, similarly  $I_X(\theta) = \frac{1}{\sigma_{x}^2}\text{I}_{d}$. Then,

\begin{align}
	\sup_{\theta_i}\mathbb{E}\|\widehat{\theta_i}(X)-\theta_i\|^2 \geq \frac{d^2}{n \mathbb{E}\frac{d}{\sigma_x^2}+\frac{d}{\sigma_{\theta}^2}} = \frac{d\sigma_{\theta}^2\sigma_{x}^2}{n\sigma_{\theta}^2+\sigma_{x}^2}
\end{align}

\textbf{Gaussian case with estimated population mean.} In this setting instead of a true prior we have a prior whose mean is the average of all data spread across clients, i.e., we assume $\theta_i\sim N(\widehat{\mu},\sigma_{\theta}^2\text{I}_{d})$ where $\widehat{\mu} = \frac{1}{mn} \sum_{i,j}^{m,n}X_i^j$. We additionally know that there is a Markov relation such that $X_i^j|\theta_j \sim N(\theta_j,\sigma_x^2\text{I}_{d})$ and $\theta_j \sim N(\mu,\sigma_{\theta}^2\text{I}_{d})$. While the true prior is parameterized with mean $\mu$, $\theta_i$ in this form is not parameterized by $\mu$ but by $\widehat{\mu}$ which itself has randomness due $X_i^j$. However, using Fact~\ref{gauss_mean_fact} twice we can write $\theta_i\sim N(\mu,(\sigma_{\theta}^2+\frac{\sigma_{\theta}^2}{m}+\frac{\sigma_{x}^2}{mn})\text{I}_{d})$. Then using the van Trees inequality similar to the lower bound in perfect case we can obtain:

\begin{align}
	\sup_{\theta_i \in \Theta}    \mathbb{E}_X\|\widehat{\theta}_i(X)-\theta_i\|^2 \geq d \frac{\sigma_{\theta}^2\sigma_{x}^2+\frac{\sigma_x^4}{mn}}{n\sigma_{\theta}^2+\sigma_x^2}
\end{align}

%% file: app_est_bernoulli.tex
\subsubsection{When $\alpha,\beta$ are Known}
For a client $i$, let $\pi(p_i)$ be distributed as Beta$(\alpha,\beta)$. In this setting, we model that each client generates local samples according to Bern$(p_i)$. Consequently, each client has a Binomial distribution regarding the sum of local data samples. Estimating Bernoulli parameter $p_i$ is related to Binomial distribution Bin$(n,p_i)$ (the sum of data samples) $Z_i$ since it is the sufficient statistic of Bernoulli distribution. The distribution for Binomial variable $Z_i$ given $p_i$ is $P(Z_i=z_i|p_i) = \binom{n}{z_i}p_i^{z_i}(1-p_i)^{n-z_i}$. It is a known fact that for any prior, the Bayesian MSE risk minimizer is the posterior mean $\mathbb{E}\left[p_i|Z_i=z_i\right]$. 

When $p_i \sim \text{Beta}(\alpha,\beta)$, we have posterior 
\begin{align*}
    f(p_i|Z_i=z_i) &= \frac{P(z_i|p_i)}{P(z_i)}\pi(p_i) \\
    &= \frac{\binom{n}{z_i}p_i^{z_i}(1-p_i)^{n-z_i}}{P(z_i)}\frac{p_i^{\alpha-1}(1-p_i)^{\beta-1}}{B(\alpha,\beta)} \\
    &= \frac{\binom{n}{z_i}}{P(z_i)}\frac{B(\alpha+z_i,\beta+n-z_i)}{B(\alpha,\beta)}\frac{p_i^{\alpha+z_i-1}(1-p_i)^{\beta+n-z_i-1}}{B(\alpha+z_i,\beta+n-z_i)},
\end{align*}
where $B(\alpha,\beta)=\frac{\Gamma(\alpha)\Gamma(\beta)}{\Gamma(\alpha+\beta)}$, and 
\begin{align*}
P(z_i) &= \int P(z_i|p_i)\pi(p_i)dp_i \\
&= \int \binom{n}{z_i}p_i^{z_i}(1-p_i)^{n-z_i}\frac{p_i^{\alpha-1}(1-p_i)^{\beta-1}}{B(\alpha,\beta)}dp_i \\
&= \binom{n}{z_i}\frac{B(z_i+\alpha,n-z_i+\beta)}{B(\alpha,\beta)} \underbrace{\int \frac{p_i^{\alpha+z_i-1}(1-p_i)^{\beta+n-z_i-1}}{B(\alpha+z_i,\beta+n-z_i)}dp_i}_{\text{integral of a Beta distribution}} \\
&= \binom{n}{z_i}\frac{B(z_i+\alpha,n-z_i+\beta)}{B(\alpha,\beta)}
\end{align*}
Thus, we get that the posterior distribution $f(p_i|Z_i=z_i)=\frac{p_i^{\alpha+z_i-1}(1-p_i)^{\beta+n-z_i-1}}{B(\alpha+z_i,\beta+n-z_i)}$ is a beta distribution Beta$(z_i+\alpha,n-z_i+\beta)$. As a result, the posterior mean is given by:

\begin{align*}
    \widehat{p}_i = \frac{\alpha+Z_i}{\alpha+\beta+n} =a\left(\frac{Z_i}{n}\right)+(1-a)\left(\frac{\alpha}{\alpha+\beta}\right),
\end{align*}
where $a=\frac{n}{\alpha+\beta+n}$. Observe that $\mathbb{E}_{p_i\sim \Beta(\alpha,\beta)}[p_i]=\frac{\alpha}{\alpha+\beta}$, i.e., the estimator is a weighted summation between the local estimator $\frac{z_i}{n}$ and the global estimator $\mu=\frac{\alpha}{\alpha+\beta}$. \\
 We have  $R_{p_i}(\widehat{p}_i) = \mathbb{E}_{\pi}\mathbb{E}(\widehat{p}_i-p_i)^2$. The MSE of the posterior mean is given by:
\begin{align*}
\text{MSE} &= \mathbb{E}[(\hat{p}_i-p_i)^2]\\
& = \mathbb{E}\left[\left(a\left(\frac{z_i}{n}-p_i\right)+(1-a)(\mu-p_i)\right)^2\right]\\
&= a^2 \mathbb{E}\left[\left(\frac{z_i}{n}-p_i\right)^2\right]+(1-a)^2 \mathbb{E}\left[\left(\mu-p_i\right)^2\right]\\
&= a^2\mathbb{E}_{p_i\sim \pi(p_i)}\left[\frac{p_i(1-p_i)}{n}\right]+(1-a)^2\frac{\alpha\beta}{(\alpha+\beta)^2(\alpha+\beta+1)}\\
&= a^2\frac{\alpha\beta}{n(\alpha+\beta)(\alpha+\beta+1)}+(1-a)^2\frac{\alpha\beta}{(\alpha+\beta)^2(\alpha+\beta+1)}\\
&= \frac{\alpha\beta}{n(\alpha+\beta)(\alpha+\beta+1)}\left(\frac{n}{\alpha+\beta+n}\right).
\end{align*}
The last equality is obtained by setting $a=\frac{n}{\alpha+\beta+n}$. 

\subsubsection{When $\alpha,\beta$ are Unknown: Proof of Theorem~\ref{thm:bern_estimate}}
\begin{theorem*}[Restating Theorem~\ref{thm:bern_estimate}]
	With probability at least $1-\frac{1}{mn}$, the MSE of the personalized estimator in~\eqref{eqn:bern_estimate_unknown} is given by:
	$\bbE_{p_i\sim\pi}\bbE_{X_1,\ldots,X_m}(\hatp_i - p_i)^2 \leq \bbE[\hata_i^2]\big(\frac{\alpha\beta}{n(\alpha+\beta)(\alpha+\beta+1)}\big)+\bbE[(1-\hata_i)^2]\big(\frac{\alpha\beta}{(\alpha+\beta)^2(\alpha+\beta+1)}+\frac{3\log(4m^2n)}{m-1}\big)$.
\end{theorem*}

\begin{proof}
The personalized model of the $i$th client with unknown parameters $\alpha,\beta$ is given by:
\begin{equation}~\label{eqn:appendix_bern_estimate_unknown}
\hat{p}_i=\overline{a}\left(\frac{z_i}{n}\right)+(1-\overline{a})\left(\hat{\mu}_i\right),    
\end{equation}
where $\overline{a}=\frac{n}{\frac{\hat{\mu}_i(1-\hat{\mu}_i)}{\hat{\sigma}^2_i}+n}$, the empirical mean $\hat{\mu}_i = \sum_{l\neq i}\frac{z_l}{n}$, and the empirical variance $\hat{\sigma}^2_i=\sum_{l\neq i}(\frac{z_l}{n}-\hat{\mu}_i)^2$.
From~\cite[Lemma~$1$]{tian2017learning}, with probability $1-\frac{1}{m^2n}$, we get that 
\begin{align*}
    |\mu-\hat{\mu}_i|&\leq  \sqrt{\frac{3\log(4m^2n)}{m-1}}\\
    |\sigma^2-\hat{\sigma}_i^2|&\leq \sqrt{\frac{3\log(4m^2n)}{m-1}},
\end{align*}
where $\mu=\frac{\alpha}{\alpha+\beta}$, $\sigma^2=\frac{\alpha\beta}{(\alpha+\beta)^2(\alpha+\beta+1)}$ are the true mean and variance of the beta distribution, respectively. Let $c=\sqrt{\frac{3\log(4m^2n)}{m-1}}$. Conditioned on the event $\mathcal{E}=\lbrace |\mu-\hat{\mu}_i|\leq c,|\sigma^2-\hat{\sigma}_i^2|\leq c:\forall i\in[m]\rbrace$ that happens with probability at least $1-\frac{1}{mn}$, we get that:
\begin{align*}
    \mathbb{E}\left[\left(\hat{p}_i-p_i\right)^2|Z_{-i}\right]&=\overline{a}^2\mathbb{E}\left[\left(\frac{Z_i}{n}-p_i\right)^2\right]+(1-\overline{a})^2\mathbb{E}\left[\left(\hat{\mu}_i-p_i\right)^2|Z_{-i}\right]\\
   &=\overline{a}^2\left(\frac{\alpha\beta}{n(\alpha+\beta)(\alpha+\beta+1)}\right)+(1-\overline{a})^2\left( \mathbb{E}\left[\left(\mu-p_i\right)^2\right]+\left(\mu-\hat{\mu}_i\right)^2\right)\\
   &=\overline{a}^2\left(\frac{\alpha\beta}{n(\alpha+\beta)(\alpha+\beta+1)}\right)+(1-\overline{a})^2\left(\frac{\alpha\beta}{(\alpha+\beta)^2(\alpha+\beta+1)}+\left(\mu-\hat{\mu}_i\right)^2\right)\\
   &\leq \overline{a}^2\left(\frac{\alpha\beta}{n(\alpha+\beta)(\alpha+\beta+1)}\right)+(1-\overline{a})^2\left(\frac{\alpha\beta}{(\alpha+\beta)^2(\alpha+\beta+1)}+c^2\right),
\end{align*}
where the expectation is with respect to $z_i\sim \Binom(p_i,n)$ and $p_i\sim \Beta(\alpha,\beta)$ and $Z_{-i}=\lbrace z_1,\ldots,z_{i-1},z_{i+1},\ldots,z_{m}\rbrace$ denotes the entire dataset except the $i$th client data ($z_i$). By taking the expectation with respect to the datasets $Z_{-i}$, we get that the MSE is bounded by:
\begin{equation*}
\text{MSE} \leq \mathbb{E}\left[\overline{a}^2\right]\left(\frac{\alpha\beta}{n(\alpha+\beta)(\alpha+\beta+1)}\right)+\mathbb{E}\left[(1-\overline{a})^2\right]\left(\frac{\alpha\beta}{(\alpha+\beta)^2(\alpha+\beta+1)}+\frac{3\log(4m^2n)}{m-1}\right),
\end{equation*}
with probability at least $1-\frac{1}{mn}$. This completes the proof of Theorem~\ref{thm:bern_estimate}.
\end{proof}
\subsubsection{With Privacy Constraints: Proof of Theorem~\ref{thm:bern_private_estimate}}
\begin{theorem*}[Restating Theorem~\ref{thm:bern_private_estimate}]
With probability at least $1-\frac{1}{mn}$, the MSE of the personalized estimator $\hatp_i^{\priv}$ defined before Theorem~\ref{thm:bern_private_estimate} is given by:
\begin{align*}
	\bbE_{p_i\sim\pi}\bbE_{q^{\priv},X_1,\ldots,X_m}(\hatp_i^{\priv} - p_i)^2 &\leq  \bbE[(\hata_i^{\priv})^2]\big(\frac{\alpha\beta}{n(\alpha+\beta)(\alpha+\beta+1)}\big) \\ & \quad +\bbE[(1-\hata_i^{\priv})^2]\big(\frac{\alpha\beta}{(\alpha+\beta)^2(\alpha+\beta+1)}+\frac{(e^{\epsilon_0}+1)^2\log(4m^2n)}{3(e^{\epsilon_0}-1)^2(m-1)}\big).
\end{align*}

\end{theorem*}
\begin{proof}
First, we prove some properties of the private mechanism $q_p$. Observe that for any two inputs $x,x'\in[0,1]$, we have that:
\begin{equation}~\label{eqn:private_mechanism_bern}
\frac{\Pr[q_p(x)=y]}{\Pr[q_p(x')=y]}=\frac{\frac{e^{\epsilon_0}}{e^{\epsilon_0}+1}-x\frac{e^{\epsilon_0}-1}{e^{\epsilon_0}+1}}{\frac{e^{\epsilon_0}}{e^{\epsilon_0}+1}-x'\frac{e^{\epsilon_0}-1}{e^{\epsilon_0}+1}}\leq e^{\epsilon_0},    
\end{equation}
for $y=\frac{-1}{e^{\epsilon_0}-1}$. Similarly, we can prove~\eqref{eqn:private_mechanism_bern} for the output $y=\frac{e^{\epsilon_0}}{e^{\epsilon_0}-1}$. Thus, the mechanism $q_p$ is user-level $\epsilon_0$-LDP. 
Furthermore, for given $x\in[0,1]$, we have that
\begin{equation}
\mathbb{E}\left[q_p(x)\right]=x.
\end{equation}
Thus, the output of the mechanism $q_p$ is an unbiased estimate of the input $x$. 
From the Hoeffding's inequality for bounded random variables, we get that:
\begin{equation}
\begin{aligned}
\Pr[|\hat{\mu}_i^{(p)}-\mu|>t]&\leq 2\exp\left(\frac{-3(e^{\epsilon_0}-1)^2(m-1)t^2}{(e^{\epsilon_0}+1)^2}\right)\\
\Pr[|\hat{\sigma}_i^{2(p)}-\sigma_2|>t]&\leq 2\exp\left(\frac{-3(e^{\epsilon_0}-1)^2(m-1)t^2}{(e^{\epsilon_0}+1)^2}\right)\\
\end{aligned}    
\end{equation}
Thus, we have that the event $\mathcal{E}=\lbrace |\hat{\mu}_i^{(p)}-\mu|\leq c_p,|\hat{\sigma}_i^{2(p)}-\sigma^2|\leq c_p:\forall i\in[m]\rbrace$ happens with probability at least $1-\frac{1}{mn}$, where $c_p=\sqrt{\frac{(e^{\epsilon_0}+1)^2\log(4m^2n)}{3(e^{\epsilon_0}-1)^2(m-1)}}$. By following the same steps as the non-private estimator, we get the fact that the MSE of the private model is bounded by:
\begin{align}
\text{MSE} &\leq \mathbb{E}\left[\overline{a}^2\right]\left(\frac{\alpha\beta}{n(\alpha+\beta)(\alpha+\beta+1)}\right) \notag \\
&\hspace{2.5cm}+\mathbb{E}\left[(1-\overline{a})^2\right]\left(\frac{\alpha\beta}{(\alpha+\beta)^2(\alpha+\beta+1)}+\frac{(e^{\epsilon_0}+1)^2\log(4m^2n)}{3(e^{\epsilon_0}-1)^2(m-1)}\right),
\end{align}
where $\overline{a}^{(p)}=\frac{n}{\frac{\hat{\mu}^{(p)}_i(1-\hat{\mu}^{(p)}_i)}{\hat{\sigma}^{2(p)}_i}+n}$ and the expectation is with respect to the clients data $\lbrace z_1,\ldots,z_{i-1},z_{i+1},\ldots,z_{m}\rbrace$and the randomness of the private mechanism $q_p$. This completes the proof of Theorem~\ref{thm:bern_private_estimate}.
\end{proof}

%% file: app_est_mixture_DD.tex
\subsubsection{When the Prior Distribution is Known: Proof of Theorem~\ref{thm:discrete_perfect}}

In this case, the $i$-th client does not need the data of the other clients as she has a perfect knowledge about the prior distribution. 
\begin{theorem*}[Restating Theorem~\ref{thm:discrete_perfect} ]
For given a perfect knowledge $\vct{\alpha}=[\alpha_1,\ldots,\alpha_k]$ and $\mathcal{N}\left(\vct{\mu}_1,\sigma_{\vct{\theta}}^{2}\right),\ldots,\mathcal{N}\left(\vct{\mu}_k,\sigma_{\vct{\theta}}^{2}\right)$, the optimal personalized estimator that minimizes the MSE is given by:
\begin{equation}
\hat{\vct{\theta}}_i =\sum_{l=1}^{k} a_{l}^{(i)}\vct{\mu}_l,
\end{equation}
where $\alpha_{l}^{(i)}=\frac{p_l\exp{\left(-\frac{\sum_{j=1}^{n}\|X_{ij}-\vct{\mu}_l\|^2}{2\sigma_x^2}\right)}}{\sum_{s=1}^{k}p_s\exp{\left(-\frac{\sum_{j=1}^{n}\|X_{ij}-\vct{\mu}_s\|^2}{2\sigma_x^2}\right)}}$ denotes the weight associated to the prior model $\vct{\mu}_l$ for $l\in[k]$.
\end{theorem*}

\begin{proof}
 Let $\vct{\theta}_i\sim \bbP$, where $\bbP=[p_1,\ldots,p_k]$ and $p_l=\Pr[\vct{\theta}_i=\vct{\mu}_l]$ for $l\in[k]$. The goal is to design an estimator $\hat{\vct{\theta}}_i$ that minimizes the MSE given by:
\begin{equation}
\text{MSE} = \mathbb{E}_{\vct{\theta}_i\sim \bbP}\mathbb{E}_{\lbrace X_{ij}\sim\mathcal{N}(\vct{\theta}_i,\sigma_x^2)\rbrace}\left[\|\hat{\vct{\theta}}_i-\vct{\theta}_i\|^2\right].
\end{equation}
Let $X_i=(X_{i1},\ldots,X_{in})$. By following the standard proof of the minimum MSE, we get that:
\begin{equation}
\begin{aligned}
\mathbb{E}_{\vct{\theta}_i}\mathbb{E}_{X_i}\left[\|\hat{\vct{\theta}}_i-\vct{\theta}_i\|^2\right]&= \mathbb{E}_{X_i}\mathbb{E}_{\vct{\theta}_i|X_i}\left[\left.\|\hat{\vct{\theta}}_i-\mathbb{E}[\vct{\theta}_i|X_i]+\mathbb{E}[\vct{\theta}_i|X_i]-\vct{\theta}_i\|^2\right|X_i\right]\\
& = \mathbb{E}_{X_i}\mathbb{E}_{\vct{\theta}_i|X_i}\left[\left.\|\mathbb{E}[\vct{\theta}_i|X_i]-\vct{\theta}_i\|^2\right|X_i\right]+ \mathbb{E}_{X_i}\mathbb{E}_{\vct{\theta}_i|X_i}\left[\left.\|\mathbb{E}[\vct{\theta}_i|X_i]-\hat{\vct{\theta}}_i\|^2\right|X_i\right] \\
&\geq \mathbb{E}_{X_i}\mathbb{E}_{\vct{\theta}_i|X_i}\left[\left.\|\mathbb{E}[\vct{\theta}_i|X_i]-\vct{\theta}_i\|^2\right|X_i\right],
\end{aligned}
\end{equation}
where the last inequality is achieved with equality when $\hat{\vct{\theta}}_i=\mathbb{E}[\vct{\theta}_i|X_i]$. The distribution on $\vct{\theta}_i$ given the local dataset $X_i$ is given by:
\begin{equation}
\begin{aligned}
\Pr[\vct{\theta}_i=\vct{\mu}_l|X_i] &= \frac{f(X_i|\vct{\theta}_i=\vct{\mu}_l)\Pr[\vct{\theta}_i=\vct{\mu}_l]}{f(X_i)}\\  
&= \frac{f(X_i|\vct{\theta}_i=\vct{\mu}_l)\Pr[\vct{\theta}_i=\vct{\mu}_l]}{\sum_{s=1}^{k}f(X_i|\vct{\theta}_i=\vct{\mu}_s)\Pr[\vct{\theta}_i=\vct{\mu}_s]}\\
&= \frac{p_l \exp\left(-\frac{\sum_{j=1}^{n}\|X_{ij}-\vct{\mu}_l\|^2}{2\sigma_x^2}\right)}{\sum_{s=1}^{k}p_s\exp\left(-\frac{\sum_{j=1}^{n}\|X_{ij}-\vct{\mu}_s\|^2}{2\sigma_x^2}\right)} =\alpha_l^{(i)}
\end{aligned}
\end{equation}
As a result, the optimal estimator is given by:
\begin{equation}
\hat{\vct{\theta}}_i= \mathbb{E}[\vct{\theta}_i|X_i]=\sum_{l=1}^{k}\alpha_{l}^{(i)} \vct{\mu}_l. 
\end{equation}
This completes the proof of Theorem~\ref{thm:discrete_perfect}.
\end{proof}

\subsubsection{Privacy/Communication Constraints: Proof of Lemma~\ref{lemm:concentration}} 

\begin{lemma*}[Restating Lemma~\ref{lemm:concentration}]
Let $\vct{\mu}_1,\ldots\vct{\mu}_k\in\bbR^d$ be unknown means such that $\|\vct{\mu}_i\|_2\leq r$ for each $i\in[k]$. Let $\vct{\theta}_1,\ldots,\vct{\theta}_m\sim \bbP$, where $\bbP=[p_1,\ldots,p_k]$ and $p_l=\Pr[\vct{\theta}_i=\vct{\mu}_l]$. For $i\in[m]$, let $X_{i1},\ldots,X_{in}\sim\mathcal{N}(\vct{\theta}_i,\sigma_x^2)$, i.i.d. 
Then, with probability at least $1-\frac{1}{mn}$, the following bound holds for all $i\in[m]$:
\begin{equation*}
\left\|\frac{1}{n}\sum_{j=1}^{n}X_{ij}\right\|_2\leq 4\sqrt{d\frac{\sigma_x^2}{n}}+2\sqrt{\log(m^2n)\frac{\sigma_x^2}{n}}+r.    
\end{equation*}

\end{lemma*}
\begin{proof}
Observe that the vector $(\overline{X}_i-\theta_i)=\frac{1}{n}\sum_{i=1}^{n}X_{ij}-\theta_i$ is a sub-Gaussian random vector with proxy $\frac{\sigma_x^2}{n}$. As a result, we have that:
\begin{equation}
\|\overline{X}_i-\theta_i\|_2\leq 4\sqrt{d\frac{\sigma_x^2}{n}}+2\sqrt{\log(1/\eta)\frac{\sigma_x^2}{n}},
\end{equation} 
with probability at least $1-\eta$ from~\cite{wainwright2019high}. Since $\mu_1,\ldots,\mu_k\in\bbR^d$ are such that $\|\vct{\mu}_i\|_2\leq r$ for each $i\in[k]$, we have:
\begin{equation}
\|\overline{X}_i\|_2\leq 4\sqrt{d\frac{\sigma_x^2}{n}}+2\sqrt{\log(1/\eta)\frac{\sigma_x^2}{n}}+r,
\end{equation}
with probability $1-\eta$ from the triangular inequality. Thus, by choosing $\eta = \frac{1}{m^2n}$ and using the union bound, this completes the proof of Lemma~\ref{lemm:concentration}.
\end{proof}

%% file: proofs_learning.tex
\section{Proofs and Additional Details for Personalized Learning}\label{sec:proof-learning}
\subsection{Personalized Learning Proof of Theorem~\ref{thm:perfect_mse_lin}}\label{app:linear-regression}

\input{app_learn_linear.tex}

\subsection{Personalized Learning -- AdaPeD}\label{app:adaped}
\input{app_learn_adaped.tex}
\subsection{Personalized Learning -- DP-AdaPeD}\label{app:dp-adaped}
\input{app_learn_dpadaped.tex}

%% file: app_learn_linear.tex
\begin{theorem*}[Restating Theorem~\ref{thm:perfect_mse_lin}]
The optimal personalized parameters at client $i$ with known $\vct{\mu},\sigma_{\theta}^2,\sigma_x^2$ is given by:
\begin{equation}
\htheta_i = \left(\frac{\mathbb{I}}{\sigma_{\theta}^2}+\frac{X_i^{T}X_i}{\sigma_x^2}\right)^{-1}\left(\frac{X_i^{T}Y_i}{\sigma_x^2}+\frac{\vct{\mu}}{\sigma_\theta^2}\right).
\end{equation}
The mean squared error (MSE) of the above $\htheta_i$ is given by:
\begin{equation}
\mathbb{E}_{\bw_i,\vct{\theta}_i}\left\|\htheta_i-\vct{\theta}_i\right\|^{2}=\mathsf{Tr}\left( \left(\frac{\mathbb{I}}{\sigma_{\theta}^2}+\frac{X_i^{T}X_i}{\sigma_x^2}\right)^{-1}\right),
\end{equation} 
\end{theorem*}
\begin{proof}
The personalized model with perfect prior is obtained by solving the optimization problem stated in \eqref{eq:learning-log-generative-model}, which is given below for convenience:
\begin{align*}
\htheta_i &=\argmin_{\vct{\theta}_i} \sum_{j=1}^n -\log(p_{\vct{\theta}_i}(Y_{ij}|X_{ij})) -\log(p(\vct{\theta}_i)).
 \\
&=\argmin_{\vct{\theta}_i}\sum_{j=1}^{n}\frac{\left(Y_{ij}-X_{ij}\vct{\theta}_i\right)^2}{2\sigma_x^2}+\frac{\|\vct{\theta}_i-\vct{\mu}\|^2}{2\sigma_{\theta}^2}. \\
&=\argmin_{\vct{\theta}_i}\frac{\|Y_{i}-X_{i}\vct{\theta}_i\|^2}{2\sigma_x^2}+\frac{\|\vct{\theta}_i-\vct{\mu}\|^2}{2\sigma_{\theta}^2}.
\end{align*}
By taking the derivative with respect to $\vct{\theta}_i$, we get
\begin{equation}
\frac{\partial}{\partial \vct{\theta}_i} = \frac{X_{i}^{T}(X_{i}\vct{\theta}_i-Y_{i})}{\sigma_x^2}+\frac{\vct{\theta}_i-\vct{\mu}}{\sigma_{\theta}^2}.
\end{equation}
Equating the above partial derivative to zero, we get that the optimal personalized parameters $\htheta_i$ is given by:
\begin{equation}
\htheta_i=\left(\frac{\mathbb{I}}{\sigma_\theta^2}+\frac{X_i^{T}X_i}{\sigma_x^{2}}\right)^{-1}\left(\frac{X_{i}^{T}Y_i}{\sigma_x^2}+\frac{\vct{\mu}}{\sigma_{\theta}^2}\right).
\end{equation}
Taking the expectation w.r.t.\ $w_i$, we get:
\begin{equation}
\bbE_{\bw_i}[\htheta_i]= \left(\frac{\mathbb{I}}{\sigma_\theta^2}+\frac{X_i^{T}X_i}{\sigma_x^{2}}\right)^{-1}\left(\frac{X_i^{T}X_i\vct{\theta}_i}{\sigma_x^2}+\frac{\vct{\mu}}{\sigma_{\theta}^2}\right),
\end{equation}

Thus, we can bound the MSE as following:
\begin{align*}
\bbE_{\bw_i,\vct{\theta}_i}&\left\|\htheta_i-\vct{\theta}_i\right\|^{2} = \mathbb{E}_{\bw_i,\vct{\theta}_i}\left\|\htheta_i-\mathbb{E}_{\bw_i}[\htheta_i]+\mathbb{E}_{\bw_i}[\htheta_i]-\vct{\theta}_i\right\|^{2} \\
&= \mathbb{E}_{\bw_i,\vct{\theta}_i}\left\|\htheta_i-\mathbb{E}_{\bw_i}[\htheta_i]\right\|^{2} + \mathbb{E}_{\bw_i,\vct{\theta}_i}\left\|\mathbb{E}_{\bw_i}[\htheta_i]-\vct{\theta}_i\right\|^{2} + 2\mathbb{E}_{\bw_i,\vct{\theta}_i} \left\langle \htheta_i-\mathbb{E}_{\bw_i}[\htheta_i], \mathbb{E}_{\bw_i}[\htheta_i]-\vct{\theta}_i \right\rangle \\
&= \mathbb{E}_{\bw_i,\vct{\theta}_i}\left\|\htheta_i-\mathbb{E}_{\bw_i}[\htheta_i]\right\|^{2} + \mathbb{E}_{\bw_i,\vct{\theta}_i}\left\|\mathbb{E}_{\bw_i}[\htheta_i]-\vct{\theta}_i\right\|^{2}
\end{align*}
In the last equality, we used $\mathbb{E}_{\bw_i,\vct{\theta}_i} \left\langle \htheta_i-\mathbb{E}_{\bw_i}[\htheta_i], \mathbb{E}_{\bw_i}[\htheta_i]-\vct{\theta}_i \right\rangle = \mathbb{E}_{\vct{\theta}_i} \left\langle \mathbb{E}_{\bw_i}[\htheta_i]-\mathbb{E}_{\bw_i}[\htheta_i], \mathbb{E}_{\bw_i}[\htheta_i]-\vct{\theta}_i \right\rangle=0$, where the first equality holds because $\mathbb{E}_{\bw_i}[\htheta_i]-\vct{\theta}_i$ is independent of $\bw_i$.

Letting $\mat{M}=\frac{\mathbb{I}}{\sigma_{\theta}^2}+\frac{X_i^{T}X_i}{\sigma_x^2}$, and $\mathsf{Tr}$ denoting the trace operation, we get
\begin{align*}
\bbE_{\bw_i,\vct{\theta}_i}\left\|\htheta_i-\vct{\theta}_i\right\|^{2} &= \mathsf{Tr}\left(\mat{M}^{-1}\mathbb{E}_{\bw_i}\left[\left(\frac{X_i^{T}\bw_i}{\sigma_x^2}\right)\left(\frac{X_i^{T}\bw_i}{\sigma_x^2}\right)^{T}\right]\mat{M}^{-1}\right) \\
&\hspace{3cm} +\mathsf{Tr}\left(\mat{M}^{-1}\mathbb{E}_{\vct{\theta}_i}\left[\left(\frac{\vct{\theta}_i-\vct{\mu}}{\sigma_\theta^2}\right)\left(\frac{\vct{\theta}_i-\vct{\mu}}{\sigma_\theta^2}\right)^{T}\right]\mat{M}^{-1}\right)\\
&=\mathsf{Tr}\left(\mat{M}^{-1}\frac{X_i^{T}X_i}{\sigma_x^2}\mat{M}^{-1}\right)+\mathsf{Tr}\left(\mat{M}^{-1}\frac{\mathbb{I}}{\sigma_\theta^2}\mat{M}^{-1}\right)\\
&=\mathsf{Tr}\left(\mat{M}^{-1}\right).
\end{align*}
This completes the proof of Theorem~\ref{thm:perfect_mse_lin}.
\end{proof}

%% file: app_learn_adaped.tex
\subsubsection{Knowledge Distillation Population Distribution}
In this section we discuss what type of a population distribution can give rise to algorithms/problems that include a knowledge distillation (KD) (or KL divergence) penalty term between local and global models. From Section~\ref{sec:learning}, Equation \eqref{eq:learning-log-generative-model}, consider $p_{\vct{\theta}_i}(y|x)$ as a randomized mapping from input space $\mathcal{X}$ to output class $\mathcal{Y}$, parameterized by $\vct{\theta}_i$. For simplicity, consider the case where $|\mathcal{X}|$ is finite, e.g. for MNIST it could be all possible $28\times28$ black and white images.
Every $p_{\vct{\theta}_i}(y|x)$ corresponds to a probability matrix (parameterized by $\vct{\theta}_i$) of size $|\mathcal{Y}|\times|\mathcal{X}|$, where the $(y,x)$'th represents the probability of the class $y$ (row) given the data sample $x$ (column). 
Therefore, each column is a probability vector. Since we want to sample the probability matrix, it suffices to restrict our attention to any set of $|\mathcal{Y}|-1$ rows, as the remaining row can be determined by these $|\mathcal{Y}|-1$ rows.



Similarly, for a global parameter $\vct{\mu}$, let $p_{\vct{\mu}}(y|x)$ define a randomized mapping from $\mathcal{X}$ to $\mathcal{Y}$, parameterized by the global parameter $\vct{\mu}$. Note that for a fixed global parameter $\vct{\mu}$, the randomized map $p_{\vct{\mu}}(y|x)$ is fixed, whereas, our goal is to sample $p_{\vct{\theta}_i}(y|x)$ for $i=1,\ldots,m$, one for each client. For simplicity of notation, define $p_{\vct{\theta}_i}:=p_{\vct{\theta}_i}(y|x)$ and $p_{\vct{\mu}}:=p_{\vct{\mu}}(y|x)$ to be the corresponding probability matrices, and let the distribution for sampling $p_{\vct{\theta}_i}(y|x)$ be denoted by $p_{p_{\vct{\mu}}}(p_{\vct{\theta}_i})$. Note that different mappings $p_{\vct{\theta}_i}(y|x)$ correspond to different $\vct{\theta}_i$'s, so we define $p(\vct{\theta}_i)$ (in Equation \eqref{eq:learning-log-generative-model}) as $p_{p_{\vct{\mu}}}(p_{\vct{\theta}_i})$, which is the density of sampling the probability matrix $p_{\vct{\theta}_i}(y|x)$.

For the KD population distribution, we define this density $p_{p_{\vct{\mu}}}(p_{\vct{\theta}_i})$ as:
\begin{align}\label{general_population_KD}
	p_{p_{\vct{\mu}}}(p_{\vct{\theta}_i}) = c(\psi)e^{-\psi D_{\KL}(p_{\vct{\mu}}(y|x)\|p_{\vct{\theta}_i}(y|x))}
\end{align}
where $\psi$ is an `inverse variance' type of parameter, $c(\psi)$ is a normalizing function that depends on $(\psi, p_{\vct{\mu}})$, and $D_{\KL}(p_{\vct{\mu}}(y|x)\|p_{\vct{\theta}_i}(y|x))=\sum_{x\in\mathcal{X}}p(x)\sum_{y\in\mathcal{Y}}p_{\vct{\mu}}(y|x)\log\left(\frac{p_{\vct{\mu}}(y|x)}{p_{\vct{\theta}_i}(y|x)}\right)$ is the conditional KL divergence, where $p(x)$ denotes the probability of sampling a data sample $x\in\mathcal{X}$. 
Now all we need is to find $c(\psi)$ given a fixed $\vct{\mu}$ (and therefore fixed $p_{\vct{\mu}}(y|x)$). Here we consider $D_{\KL}(p_{\vct{\mu}}\|p_{\vct{\theta}_i})$, but our analysis can be extended to $D_{\KL}(p_{\vct{\theta}_i}\|p_{\vct{\mu}})$ or $\|p_{\vct{\theta}_i}-p_{\vct{\mu}}\|_2$ as well.

For simplicity and to make the calculations easier, we consider a binary classification task with $\mathcal{Y}=\{0,1\}$ and define $p_{\vct{\mu}}(x):=p_{\vct{\mu}}(y=1|X=x)$ and $q_i(x):=p_{\vct{\theta}_i}(y=1|X=x)$. We have:
\begin{align*}
	D_{\KL}(p_{\vct{\mu}}(y|x)\|p_{\vct{\theta}_i}(y|x)) =\sum_x p(x) \Big( p_{\vct{\mu}}(x)(\log p_{\vct{\mu}}(x)-\log q_i(x))+(1-p_{\vct{\mu}}(x))(\log (1-p_{\vct{\mu}}(x))-\log (1-q_i(x)))\Big).
\end{align*}
Hence, after some algebra we have,
\begin{align*}
p_{p_{\vct{\mu}}}(p_{\vct{\theta}_i}) = c(\psi)e^{\psi\sum_xp(x)H(p_{\vct{\mu}}(x))}e^{\psi \sum_x p(x) (p_{\vct{\mu}}(x)\log(q_i(x))+(1-p_{\vct{\mu}}(x))\log(1-q_i(x))))}
\end{align*}
Then,
\begin{align*}
	c(\psi)\prod_{x}\Big[\int_{0}^{1}e^{\psi p(x)H(p_{\vct{\mu}}(x))}e^{\psi p(x) (p_{\vct{\mu}}(x)\log(q_i(x))+(1-p_{\vct{\mu}}(x))\log(1-q_i(x))))}dq_i(x)\Big]=1.
\end{align*}
Note that 
\begin{align*}
	\int_{0}^{1}e^{\psi p(x) (p_{\vct{\mu}}(x)\log(q_i(x))+(1-p_{\vct{\mu}}(x))\log(1-q_i(x))))}dq_i(x)=B\left(1+\frac{p_{\vct{\mu}}(x)}{\psi p(x)},1+\frac{1-p_{\vct{\mu}}(x)}{\psi p(x)}\right) 
\end{align*}
Accordingly, after some algebra, we can obtain $c(\psi)=\frac{e^{-\psi \sum_xp(x)H(p_{\vct{\mu}}(x))}}{\prod_{x}B\left(1+\frac{p_{\vct{\mu}}(x)}{\psi p(x)},1+\frac{1-p_{\vct{\mu}}(x)}{\psi p(x)}\right)}$, where $H$ is binary Shannon entropy.
Substituting this in \eqref{general_population_KD}, we get
\[p_{p_{\vct{\mu}}}(p_{\vct{\theta}_i}) = \frac{e^{-\psi \sum_xp(x)H(p_{\vct{\mu}}(x))}}{\prod_{x}B(1+\frac{p_{\vct{\mu}}(x)}{\psi p(x)},1+\frac{1-p_{\vct{\mu}}(x)}{\psi p(x)})} e^{-\psi D_{\KL}(p_{\vct{\mu}}(y|x)\|p_{\vct{\theta}_i}(y|x))}\]
which is the population distribution that can result in a KD type regularizer. Note that when we take the negative logarithm of the population distribution we obtain KL divergence loss and an additional term that depends on $\psi$ and $p_{\vct{\mu}}$. This is the form seen in Section~\ref{sec:adaped} Equation \eqref{loc_loss} for AdaPeD algorithm. For numerical purpose, we take the additional term $-\log\left(\frac{e^{-\psi \sum_xp(x)H(p_{\vct{\mu}}(x))}}{\prod_{x}B(1+\frac{p_{\vct{\mu}}(x)}{\psi p(x)},1+\frac{1-p_{\vct{\mu}}(x)}{\psi p(x)})}\right)$ to be simple $\frac{1}{2}\log(2\psi)$. As mentioned in Section~\ref{sec:adaped}, this serves the purpose of regularizing $\psi$. This is in contrast to the objective considered in \cite{ozkara2021quped}, which only has the KL divergence loss as the regularizer, without the additional term.

\subsubsection{AdaPeD with Local Fine Tuning}
When there is a flexibility in computational resources for doing local iterations, unsampled clients can do local training on their personalized models to speed-up convergence at no cost to privacy. This can be used in cross-silo settings, such as cross-institutional training for hospitals, where privacy is crucial and there are available computing resources most of the time. We propose the algorithm for AdaPeD with local fine-tuning:

\begin{algorithm}[h]
	\caption{Adaptive Personalization via Distillation (\adaped) with local fine-tuning}
	{\bf Parameters:} local variances $\{\psi_i^{0}\}$, personalized models $\{\vct{\theta}_i^{0}\}$, local copies of the global model $\{\vct{\mu}_i^0\}$, synchronization gap $\tau$, learning rates $\eta_1,\eta_2,\eta_3$, number of sampled clients $K$. 
	\begin{algorithmic}[1] \label{algo:personalized}
		\FOR{$t=0$ \textbf{to} $T-1$}
		\IF{$\tau$ divides $t$}
		\STATE \textbf{On Server do:}\\
		\STATE Choose a subset $\mathcal{K}^t \subseteq [n]$ of $K$ clients \\
		\STATE Broadcast $\vct{\mu}^t$ and $\psi^{t}$
		\STATE \textbf{On Clients} $i\in\mathcal{K}^t$ (in parallel) \textbf{do}:\\
		\STATE Receive $\vct{\mu}^t$ and $\psi^{t}$; set $\vct{\mu}_i^t = \vct{\mu}^t$, $\psi_i^{t} = \psi^{t}$
		\ENDIF	
		\STATE \textbf{On Clients} $i\notin\mathcal{K}^t$ (in parallel) \textbf{do}:
		\STATE Compute $\bg_{i}^{t} := \nabla_{\vct{\theta}_{i}^{t}} f_i(\vct{\theta}_{i}^{t}) +  \frac{\nabla_{\vct{\theta}_{i}^{t}}f^{\KD}_i(\vct{\theta}_{i}^{t}, \vct{\mu}_{i}^{t'_i})}{2\psi_i^{t'_i}}  $ where $t'_i$ is the last time index where client $i$ received global parameters from the server
		\STATE Update: $\vct{\theta}_{i}^{t+1}=\vct{\theta}_{i}^{t} - \eta_1 \bg_{i}^{t}$\\
		\STATE \textbf{On Clients} $i\in\mathcal{K}^t$ (in parallel) \textbf{do}:		
		\STATE Compute $\bg_{i}^{t} := \nabla_{\vct{\theta}_{i}^{t}} f_i(\vct{\theta}_{i}^{t}) +  \frac{\nabla_{\vct{\theta}_{i}^{t}}f^{\KD}_i(\vct{\theta}_{i}^{t}, \vct{\mu}_{i}^{t})}{2\psi_i^{t}}  $
		\STATE Update: $\vct{\theta}_{i}^{t+1}=\vct{\theta}_{i}^{t} - \eta_1 \bg_{i}^{t}$\\
		\STATE  Compute $\bh_{i}^{t} := \frac{\nabla_{\vct{\mu}_{i}^{t}}f^{\KD}_i(\vct{\theta}_{i}^{t+1}, \vct{\mu}_{i}^{t})}{2\psi_i^{t}}$ \\
		\STATE Update: $\vct{\mu}_{i}^{t+1} = \vct{\mu}_{i}^{t}-\eta_2\bh_{i}^{t}  $\\
		\STATE  Compute $k_{i}^{t} := \frac{1}{2\psi_i^{t}}-\frac{f^{\KD}_i(\vct{\theta}_{i}^{t+1}, \vct{\mu}_{i}^{t+1})}{2(\psi_i^{t})^2}$ \\
		\STATE Update: $\psi_i^{t+1}=\psi_i^{t}-\eta_3k_i^t$ 
		\IF{$\tau$ divides $t+1$}
		\STATE Clients send $\vct{\mu}_{i}^{t}$ and $\psi_i^{t}$ to \textbf{Server}
		\STATE Server receives $\{\vct{\mu}_{i}^{t}\}_{i\in\mathcal{K}^t}$ and $\{\psi_i^{t}\}_{i\in\mathcal{K}^t}$ \\
		\STATE Server computes $\vct{\mu}^{t+1} = \frac{1}{K} \sum_{i\in\mathcal{K}^t} \vct{\mu}_i^{t}$ and $\psi^{t+1} = \frac{1}{K} \sum_{i\in\mathcal{K}^t} \psi_i^{t}$\\
		\ENDIF
		\ENDFOR
	\end{algorithmic}
	{\bf Output:} Personalized models $(\vct{\theta}_i^{T})_{i=1}^m$
\end{algorithm}
Of course, when a client is not sampled for a long period of rounds this approach can become similar to a local training; hence, it might be reasonable to put an upper limit on the successive number of local iterations for each client.

%% file: app_learn_dpadaped.tex
\textbf{Proof of Theorem~\ref{thm:privacy_personal}}

\begin{theorem*}[Restating Theorem~\ref{thm:privacy_personal}]
After $T$ iterations, \emph{\dpadaped}\ satisfies $(\alpha,\epsilon(\alpha))$-RDP for $\alpha>1$, where $\epsilon(\alpha)=\left(\frac{K}{m}\right)^26\left(\frac{T}{\tau}\right)\alpha \left(\frac{C_1^2}{K\sigma_{q_1}^2}+\frac{C_2^2}{K\sigma_{q_2}^2}\right)$, where $\frac{K}{m}$ denotes the sampling ratio of the clients at each global iteration. 
\end{theorem*}
\begin{proof}
In this section, we provide the privacy analysis of \dpadaped.
We first analyze the RDP of a single global round $t\in[T]$ and then, we obtain the results from the composition of the RDP over total $T$ global rounds. Recall that privacy leakage can happen through communicating $\{\vct{\mu}_i\}$ and $\{\psi_i^t\}$ and we privatize both of these. In the following, we do the privacy analysis of privatizing $\{\vct{\mu}_i\}$ and a similar analysis could be done for $\{\psi_i^t\}$ as well.

At each synchronization round $t\in[T]$, the server updates the global model $\vct{\mu}^{t+1}$ as follows:
\begin{equation}
    \vct{\mu}^{t+1} = \frac{1}{K}\sum_{i\in\mathcal{K}{t}}\vct{\mu}_i^{t},
\end{equation}
where $\vct{\mu}_i^{t}$ is the update of the global model at the $i$-th client that is obtained by running $\tau$ local iterations at the $i$-th client. At each of the local iterations, the client clips the gradient $\bh_{i}^{t}$ with threshold $C_1$ and adds a zero-mean Gaussian noise vector with variance $\sigma_{q_1}^2\bbI_d$. When neglecting the noise added at the local iterations, the norm-$2$ sensitivity of updating the global model $\vct{\mu}_i^{t+1}$ at the synchronization round $t$ is bounded by:
\begin{equation}
\Delta \vct{\mu} = \max_{\mathcal{K}^{t},\mathcal{K}^{'t}}\|\vct{\mu}^{t+1}-\vct{\mu}'^{t+1}\|_2^2\leq \frac{\tau C_1^2}{K^2},     
\end{equation}
where $\mathcal{K}^{t},\mathcal{K}^{'t}\subset [m]$ are neighboring sets that differ in only one client. Additionally, $\vct{\mu}^{t+1}=\frac{1}{K}\sum_{i\in\mathcal{K}{t}}\vct{\mu}_i^{t}$ and $\vct{\mu}'^{t+1}=\frac{1}{K}\sum_{i\in\mathcal{K}{'t}}\vct{\mu}_i^{t}$. Since we add i.i.d. Gaussian noises with variance $\sigma_{q_1}^{2}$ at each local iteration at each client, and then, we take the average of theses vectors over $K$ clients, it is equivalent to adding a single Gaussian vector to the aggregated vectors with variance $\frac{\tau\sigma_{q_1}^{2}}{K}$. Thus, from the RDP of the sub-sampled Gaussian mechanism in~\cite[Table 1]{RDP_Gaussian_subsampled19}, \cite{truncatedCDP_bun-etal18}, we get that the global model $\vct{\mu}^{t+1}$ of a single global iteration of \dpadaped\  is  $(\alpha,\epsilon_t^{(1)}(\alpha))$-RDP, where $\epsilon_t(\alpha)$ is bounded by:
\begin{equation}
\epsilon_t^{(1)}(\alpha) =\left(\frac{K}{m}\right)^{2}\frac{6\alpha C_1^2 }{K\sigma_{q_1}^{2}}.    
\end{equation}
Similarly, we can show that the global parameter $\psi^{t+1}$ at any synchronization round of \dpadaped\  is  $(\alpha,\epsilon_t^{(2)}(\alpha))$-RDP, where $\epsilon_t(\alpha)$ is bounded by:
\begin{equation}
\epsilon_t^{(2)}(\alpha) =\left(\frac{K}{m}\right)^{2}\frac{6\alpha C_2^2 }{K\sigma_{q_2}^{2}}.    
\end{equation}
Using adaptive RDP composition~\cite[Proposition 1]{RDP_Mironov17}, we get that each synchronization round of \dpadaped\ is  $(\alpha,\epsilon_t^{(1)}(\alpha)+\epsilon_t^{(2)}(\alpha))$-RDP.
Thus, by running \dpadaped\ over $T/\tau$ synchronization rounds and from the composition of the RDP, we get that \dpadaped\ is $(\alpha,\epsilon(\alpha))$-RDP, where $\epsilon(\alpha)= \left(\frac{T}{\tau}\right)(\epsilon_t^{(1)}(\alpha)+\epsilon_t^{(2)}(\alpha))$. This completes the proof of Theorem~\ref{thm:privacy_personal}.
\end{proof}

%% file: app_preliminary.tex
We give standard privacy definitions that we use in Section~\ref{sec:priv-defn}, some existing results on RDP to DP conversion and RDP composition in Section~\ref{sec:RDP-DP}, and user-level differential privacy in Section~\ref{sec:user-level-dp}.

\subsection{Privacy Definitions}\label{sec:priv-defn}
In this subsection, we define different privacy notions that we will use in this paper: local differential privacy (LDP), central different privacy (DP), and Renyi differential privacy (RDP), and their user-level counterparts.
\begin{definition}[Local Differential Privacy - LDP~\cite{kasiviswanathan2011can}]~\label{defn:LDPdef}
	For $\epsilon_0\geq0$, a randomized mechanism $\calR:\calX\to\calY$ is said to be $\eps_0$-local differentially private (in short, $\eps_{0}$-LDP), if for every pair of inputs $d,d'\in\calX$, we have 
	\begin{equation}~\label{ldp-def}
		\Pr[\calR(d)\in \calS] \leq e^{\eps_0}\Pr[\calR(d')\in \calS], \qquad \forall \calS\subset\calY.
	\end{equation}
\end{definition}
Let $\calD=\lbrace x_1,\ldots,x_n\rbrace$ denote a dataset comprising $n$ points from $\calX$. We say that two datasets $\calD=\lbrace x_1,\ldots,x_n\rbrace$ and $\calD^{\prime}=\lbrace x_1^{\prime},\ldots,x_n^{\prime}\rbrace$ are neighboring (and denoted by $\calD\sim\calD'$) if they differ in one data point, i.e., there exists an $i\in[n]$ such that $x_i\neq x'_i$ and for every $j\in[n],j\neq i$, we have $x_j=x'_j$.
\begin{definition}[Central Differential Privacy - DP \cite{Calibrating_DP06,dwork2014algorithmic}]\label{defn:central-DP}
	For $\epsilon,\delta\geq0$, a randomized mechanism $\calM:\calX^n\to\calY$ is said to be $(\epsilon,\delta)$-differentially private (in short, $(\epsilon,\delta)$-DP), if for all neighboring datasets $\calD\sim\calD^{\prime}\in\calX^{n}$ and every subset $\calS\subseteq \calY$, we have
	\begin{equation}~\label{dp_def}
		\Pr\left[\calM(\calD)\in\calS\right]\leq e^{\eps_0}\Pr\left[\calM(\calD^{\prime})\in\calS\right]+\delta.
	\end{equation}
	If $\delta = 0$, then the privacy is referred to as pure DP.
\end{definition}

\begin{definition}[$(\lambda,\epsilon(\lambda))$-RDP (Renyi Differential Privacy)~\cite{RDP_Mironov17}]\label{defn:RDP}
	A randomized mechanism $\calM:\calX^n\to\calY$ is said to have $\epsilon(\lambda)$-Renyi differential privacy of order $\lambda\in(1,\infty)$ (in short, $(\lambda,\epsilon(\lambda))$-RDP), if for any neighboring datasets $\calD\sim\calD'\in\calX^n$, the Renyi divergence between $\calM(\calD)$ and $\calM(\calD')$ is upper-bounded by $\eps(\lambda)$, i.e.,
	\begin{align*}
		D_{\lambda}(\calM(\calD)||\calM(\calD'))&=\frac{1}{\lambda-1}\log\left(\mathbb{E}_{\theta\sim\calM(\calD')}\left[\left(\frac{\calM(\calD)(\theta)}{\calM(\calD')(\theta)}\right)^{\lambda}\right]\right)\\
		&\leq \epsilon(\lambda),
	\end{align*}
	where $\calM(\calD)(\theta)$ denotes the probability that $\calM$ on input $\calD$ generates the output $\theta$. For convenience, instead of $\eps(\lambda)$ being an upper bound, we define it as $\eps(\lambda)=\sup_{\calD\sim\calD'}D_{\lambda}(\calM(\calD)||\calM(\calD'))$.
\end{definition}

\subsection{RDP to DP Conversion and RDP Composition}\label{sec:RDP-DP}
As mentioned after Theorem~\ref{thm:privacy_personal}, we can convert the RDP guarantees of \dpadaped\ to its DP guarantees using existing conversion results from literature. To the best of our knowledge, the following gives the best conversion.
\begin{lemma}[From RDP to DP~\cite{canonne2020discrete,Borja_HypTest-RDP20}]\label{lem:RDP_DP} 
	Suppose for any $\lambda>1$, a mechanism $\calM$ is $\left(\lambda,\epsilon\left(\lambda\right)\right)$-RDP. Then, the mechanism $\calM$ is $\left(\epsilon,\delta\right)$-DP, where $\epsilon,\delta$ are define below: 
	\begin{equation*}
		\begin{aligned}
			&\text{For a given }\delta\in(0,1) :\\
			& \epsilon = \min_{\lambda} \epsilon\left(\lambda\right)+\frac{\log\left(1/\delta\right)+\left(\lambda-1\right)\log\left(1-1/\lambda\right)-\log\left(\lambda\right)}{\lambda-1}\\
			&\text{For a given }\epsilon>0 : \\
			& \delta = \min_{\lambda}\frac{\exp\left(\left(\lambda-1\right)\left(\epsilon\left(\lambda\right)-\epsilon\right)\right)}{\lambda-1}\left(1-\frac{1}{\lambda}\right)^{\lambda}. 
		\end{aligned}
	\end{equation*}
\end{lemma}
The main strength of RDP in comparison to other privacy notions comes from composition. The following result states that if we adaptively compose two RDP mechanisms with the same order, their privacy parameters add up in the resulting mechanism.
\begin{lemma}[Adaptive composition of RDP~{\cite[Proposition~1]{RDP_Mironov17}}]\label{lemm:compostion_rdp} 
	For any $\lambda>1$, let $\calM_1:\calX\to \calY_1$ be a $(\lambda,\epsilon_1(\lambda))$-RDP mechanism and $\calM_2:\calY_1\times \calX\to \calY$ be a $(\lambda,\epsilon_2(\lambda))$-RDP mechanism. Then, the mechanism defined by $(\calM_1,\calM_2)$ satisfies $(\lambda,\epsilon_1(\lambda)+\epsilon_2(\lambda))$-RDP. 
\end{lemma}
%

\subsection{User-level Differential Privacy~\cite{levy_user-level-dp21}}\label{sec:user-level-dp}
Consider a set of $m$ users, each having a local dataset of $n$ samples. Let $\calD_i=\lbrace x_{i1},\ldots,x_{in}\rbrace$ denote the local dataset at the $i$-th user for $i\in[m]$, where $x_{ij}\in\calX$ and $\calX\subset \mathbb{R}^d$. We define $\calD=(\calD_1,\ldots,\calD_{m})\in(\calX^n)^m$ as the entire dataset. 


We have already defined DP, LDP, and RDP in Section~\ref{sec:priv-defn} w.r.t.\ the item-level privacy. Here, we extend those definition w.r.t.\ the user-level privacy. In order to do that, we need a generic neighborhood relation between datasets: We say that two datasets $\calD$, $\calD'$ are neighboring with respect to distance metric $\mathsf{dis}$ if we have $\mathsf{dis}(\calD,\calD')\leq 1$. 

.
\textbf{Item-level DP/RDP vs.\ User-level DP/RDP.}
By choosing $\mathsf{dis}(\calD,\calD')=\sum_{i=1}^{m}\sum_{j=1}^{n}\mathbbm{1}\lbrace x_{ij}\neq x_{ij}' \rbrace$, we recover the standard definition of the DP/RDP from Definitions~\ref{defn:central-DP},~\ref{defn:RDP}, which we call {\em item-level} DP/RDP. In the item-level DP/RDP, two datasets $\calD$, $\calD'$ are neighboring if they differ in a single item. On the other hand, by choosing $\mathsf{dis}(\calD,\calD')=\sum_{i=1}^{m}\mathbbm{1}\lbrace \calD_i\neq \calD'_i \rbrace$, we call it \textit{user-level} DP/RDP, where two datasets $\calD,\calD'\in (\calX^{n})^m$ are neighboring when they differ in a local dataset of any single user. Observe that when each user has a single item ($n=1$), then both item-level and user-level privacy are equivalent.   

\textbf{User-level Local Differential Privacy (LDP).}
When we have a single user (i.e., $m=1$ and $\mathbb{\calD}=\calX^{n}$), by choosing $\mathsf{dis}\left(\calD,\calD'\right)=\mathbbm{1}\lbrace \calD\neq \calD' \rbrace$ for $\calD,\calD'\in\calX^n$, we call it \textit{user-level LDP}. In this case each user privatize her own local dataset using a private mechanism.

We can define user-level LDP/DP/RDP analogously to their item-level counterparts using the neighborhood relation $\mathsf{dis}$ defined above.

%% file: app_disc_mixture.tex
In this section, we present the linear regression problem as a generalization to the estimation problem with discrete priors. This model falls into the framework studied in \cite{marfoq2021federated} and can be seen as a special case of Gaussian mixture model with 0 variance. The goal is to illustrate how our framework can capture the model in \cite{marfoq2021federated}.

Consider a set of $m$ clients, where the $i$-th client has a local dataset $(X_{i1},Y_{i1}),\ldots,(X_{in},Y_{in})$ of $m$ samples, where $X_{ij}\in\mathbb{R}^{d}$ denotes the feature vector and $Y_{ij}\in\mathbb{R}$ denotes the corresponding response. Let $Y_i=(Y_{i1},\ldots,Y_{i1})\in\mathbb{R}^{n}$ and $X_i=(X_{i1},\ldots,X_{in})\in\mathbb{R}^{n\times d}$ denote the response vector and the feature matrix at the $i$-th client, respectively. Following the standard regression, we assume that the response vector $Y_i$ is obtained from a linear model as follows:
\begin{equation}
	Y_i = X_i \vct{\theta}_i+\bw_i,    
\end{equation}
where $\vct{\theta}_i$ denotes personalized model of the $i$-th client and $\bw_i\sim\mathcal{N}\left(0,\sigma_x^2\mathbb{I}_{n}\right)$ is a noise vector. The clients models are drawn i.i.d. from a discrete distribution $\vct{\theta}_1,\ldots,\vct{\theta}_{m}\sim \bbP$, where $\bbP=[p_1,\ldots,p_k]$ such that $p_l=\Pr[\vct{\theta}_i=\vct{\mu}_l]$ for $i\in[m]$ and $l\in[k]$.

Our goal is to solve the optimization problem stated in \eqref{eq:learning-log-generative-model} (for the linear regression with the above discrete prior) and learn the optimal personalized parameters $\{\htheta_i\}$.

We assume that the discrete distribution $\bbP$ and the prior candidates $\lbrace \vct{\mu}_l \rbrace_{l=1}^{k}$ are unknown to the clients. Inspired from Algorithm~\ref{algo:cluster} for estimation with discrete priors, we obtain Algorithm~\ref{algo:cluster_linear_regression} for learning with discrete prior. 
Note that this is \emph{not} a new algorithm, and is essentially the algorithm proposed in \cite{marfoq2021federated} applied to linear regression with Clustering algorithm instead of global model aggregation in the server. Here we show how our framework captures mixture model in \cite{marfoq2021federated} through this example.

\textbf{Description of Algorithm~\ref{algo:cluster_linear_regression}.} Client $i$ initializes its personalized parameters $\vct{\theta}_i^{(0)}=(X_i^{T}X_i)^{-1}X_i^{T}Y_i$, which is the optimal as a function of the local dataset at the $i$-th client without any prior knowledge. 
In any iteration $t$, for a given prior information $\bbP^{(t)}$, $\lbrace \vct{\mu}_l^{(t)}\rbrace$, the $i$-th client updates the personalized model as $\vct{\theta}_i^{t}= \sum_{l=1}^{k}\alpha_{l}^{(i)}\vct{\mu}_l^{(t)}$, where the weights $\alpha_{l}^{(i)}\propto p_l^{(t)}\exp\left(-\frac{\|X_i\vct{\mu}_l^{(t)}-Y_i\|^2}{2\sigma_x^2}\right)$ and sends its current estimate of the personalized parameter $\vct{\theta}_i^t$ to the server. Upon receiving $\vct{\theta}_1^t,\ldots,\vct{\theta}_m^t$, server will run $\mathsf{Cluster}$ algorithm to update the global parameters $\bbP,\vct{\mu}_1^{(t)},\ldots,\vct{\mu}_k^{(t)}$, and broadcasts them to the clients.

\begin{algorithm}[h]
	\caption{Alternating Minimization for Personalized Learning}
	{\bf Input:} Number of iterations $T$, local datasets $(X_i,Y_i)$ for $i\in[m]$.\\
	\vspace{-0.3cm}
	\begin{algorithmic}[1] 	\label{algo:cluster_linear_regression}
		\STATE \textbf{Initialize} $\vct{\theta}_i^{0}=(X_i^{T}X_i)^{-1}X_i^{T}Y_i$  for $i\in[m]$ (if $X_i^{T}X_i$ is not full-rank, take the pseudo-inverse), $\bbP^{(0)},\vct{\mu}_1^{(0)},\ldots,\vct{\mu}_k^{(0)}$.
		\FOR{$t=1$ \textbf{to} $T$}
		\STATE {\bf On Clients:}
		\FOR {$i=1$ \textbf{to} $m$:}
		\STATE Receive $\bbP^{(t-1)},\vct{\mu}_1^{(t-1)},\ldots,\vct{\mu}_k^{(t)}$ from the server 
		\STATE\label{step:personal_update} Update the personalized parameters and the coefficients:
		\begin{align*}
			\vct{\theta}_i^{t} \gets \sum_{l=1}^{k}\alpha_{l}^{(i)}\vct{\mu}_l^{(t-1)} \qquad \text{ and } \qquad
			\alpha_{l}^{(i)} =\frac{p_l^{(t-1)}\exp\left(-\frac{\|X_{i}\vct{\mu}_l^{(t-1)}-Y_i\|^2}{2\sigma_x^2}\right)}{\sum_{s=1}^{k}p_s^{(t-1)}\exp\left(-\frac{\|X_{i}\vct{\mu}_s^{(t-1)}-Y_i\|^2}{2\sigma_x^2}\right)}
		\end{align*}
		\STATE Send $\vct{\theta}_i^{(t)}$ to the server
		\ENDFOR
		\STATE {\bf At the Server:}
		\STATE Receive $\vct{\theta}_1^{(t)},\ldots,\vct{\theta}_m^{(t)}$ from the clients
		\STATE\label{step:clustering} Update the global parameters:
		$\bbP^{(t)},\vct{\mu}_1^{(t)},\ldots,\vct{\mu}_k^{(t)}\gets \mathsf{Cluster}\left(\vct{\theta}_1^{(t)},\ldots,\vct{\theta}_m^{(t)},k\right)$
		\STATE Broadcast $\bbP^{(t)},\vct{\mu}_1^{(t)},\ldots,\vct{\mu}_k^{(t)}$ to all clients
		\ENDFOR
	\end{algorithmic}
	{\bf Output:} Personalized models $\vct{\theta}_1^{T},\ldots,\vct{\theta}_m^{T}$.
\end{algorithm}